\documentclass{article} % 

\usepackage{silence}
\usepackage{iclr2026_conference,times}

\usepackage{amsmath,amsfonts,bm}

\def\eqref#1{Eq.~\ref{#1}}
\def\1{\bm{1}}

\DeclareMathAlphabet{\mathsfit}{\encodingdefault}{\sfdefault}{m}{sl}
\SetMathAlphabet{\mathsfit}{bold}{\encodingdefault}{\sfdefault}{bx}{n}

\usepackage{hyperref}
\usepackage{url}

\usepackage{booktabs}       % 
\usepackage{amsfonts}       % 
\usepackage{tcolorbox}      % 
\usepackage{nicefrac}       % 
\usepackage{microtype}      % 
\usepackage{xcolor}         % 

\usepackage{subcaption}

\usepackage{float} % 

\usepackage{tcolorbox}
\tcbuselibrary{skins, breakable, listings, theorems} % 

\usepackage{amsfonts}       % 
\usepackage{amsthm}
\newtheorem{theorem}{Theorem}
\usepackage{amsmath}
\newtheorem{definition}{Definition}
\usepackage{mathtools}
\usepackage{enumitem}

\newtheorem*{theorem*}{Theorem}

\usepackage{amsmath}

 {\everymath{\displaystyle\everymath{}}\array}%
 {\endarray}
\everymath{\displaystyle\everymath{}}

\newcommand{\vct}[1]{\boldsymbol{#1}}
\newcommand{\norm}[1]{\left\| #1 \right\|}
\newcommand{\abs}[1]{\left| #1 \right|}

\def\I{\mathbf{I}}

\def\g{\mathbf{g}}

\def\x{\mathbf{x}}

\def\z{\mathbf{z}}

\def\bmu{\boldsymbol{\mu}}

\def\0{\mathbf{0}}
\def\1{\mathbf{1}}

\newcommand{\tomtm}[1]{\textcolor{black}{#1}}

\title{Does the Data Processing Inequality Reflect Practice? On the  Utility of Low-Level Tasks}

\author{Roy Turgeman % 
\\
Faculty of Engineering\\
Bar-Ilan University\\
Ramat Gan, Israel \\
\And
Tom Tirer \\
Faculty of Engineering \\
Bar-Ilan University \\
Ramat Gan, Israel \\
}

\iclrfinalcopy % 
\begin{document}

\maketitle

\begin{abstract}
The data processing inequality is an information-theoretic principle stating that the information content of a signal cannot be increased by processing the observations. In particular, it suggests that there is no benefit in enhancing the signal or encoding it before addressing a classification problem.
This assertion can be proven to be true for the case of the optimal Bayes classifier.
However, in practice, it is common to perform ``low-level'' tasks before ``high-level'' downstream tasks despite the overwhelming capabilities of modern deep neural networks. 
In this paper, we aim to understand when and why low-level processing can be beneficial for classification. 
We present a comprehensive theoretical study of a binary classification setup, where we consider a classifier that is tightly connected to the optimal Bayes classifier and converges to it as the number of training samples increases. We prove that for any finite number of training samples, there exists a pre-classification processing that improves the classification accuracy. 
We also explore the effect of class separation, training set size, and class balance on the relative gain from this procedure.
We support our theory with an empirical investigation of the theoretical setup.
Finally, we conduct an empirical study where we investigate the effect of denoising and encoding on the performance of practical deep classifiers on benchmark datasets.
Specifically, we vary the size and class distribution of the training set, and the noise level, and demonstrate trends that are consistent with our theoretical results.
\end{abstract}

\section{Introduction}
\label{sec:intro}

Deep neural networks (DNNs) have demonstrated remarkable performance across an extensive range of tasks, spanning from image and speech recognition to natural language processing and scientific discovery.
When the end goal is to address ``high-level'' tasks, e.g., classification and detection, a natural approach is to train a DNN to directly solve the task using the raw data/observations as input \citep{yim2017enhancing,hendrycks2018benchmarking,singh2019dual}.
Yet, it is a common practice to begin with addressing a ``low-level'' task in order to improve the quality of the input for the high-level task. 
Such low-level tasks include signal/image restorations 
\citep{tirer2018image,tirer2020back,zhang2021plug}
as considered in \citep{liu2018image,dai2016image,li2023detection,son2020urie,haris2021task,pei2018does}, or encoding to a learned embedding space \citep{lee2022fifo,zhou2017anomaly,wu2023denoising}.

This common pipeline, however, stands in contrast to the data processing inequality, a foundational concept in information theory \citep{cover1999elements}, which states that 
the information content of a signal cannot be increased by processing the observations. 
Concretely, consider the Markov chain of three random variables: $y \to x \to z$, 
{\color{black} which denotes that $z$ is independent of $y$ given $x$, i.e., $p_{z|x,y}(z|x,y) = p_{z|x}(z|x)$ in terms of probability distributions.}
This implies that $p_{x,y,z}(x,y,z)=p_y(y)p_{x|y}(x|y)p_{z|x}(z|x)=p_x(x)p_{y|x}(y|x)p_{z|x}(z|x)$.
The data processing inequality reads as
\begin{align}
\label{eq:dpi}
    I(x,y) \geq I(z,y)
\end{align}
where $I(x,y)$ is the mutual information of the random variables $x$ and $y$.\footnote{The mutual information is defined as $I(x,y)=\iint p_{x,y}(x,y)\log \left( \frac{p_{x,y}(x,y)}{p_x(x)p_y(y)} \right) \mathrm{d}x\mathrm{d}y$.}
In particular, if $y$ is the class of a data sample $x$, this implies that there is no ``benefit'' in low-level processing of the sample (e.g., obtaining $z$ by denoising $x$) before directly considering the classification problem.

Focusing on classification and referring to better (top-1) accuracy as ``benefit'', the previous assertion can be proven to be true for the case of the optimal Bayes classifier (more details in Section \ref{sec:background}).
Clearly, we expect performance gaps between practical classifiers and the optimal Bayes classifier. However, modern DNN-based classifiers reach outstanding classification performance, sometimes even exceeding human capabilities.
This raises the question: What can we say about the margin between this implication of the data processing inequality and practical classifiers? 
To the best of our knowledge, no prior work has attempted to theoretically and systematically investigate this question.

In this paper, we aim to understand when and why low-level processing can be beneficial for classification, even when the classifier is ``strong'' (e.g., converges to the optimal Bayes classifier when the number of training samples grows). 
Our main contributions include:

\begin{itemize}[leftmargin=*]
    \item 
    We present a theoretical study of a binary classification setup, where we consider a classifier that is tightly connected to the optimal Bayes classifier (and converges to it). In the high-dimensional setting, we prove that for any finite number of training samples, there exist a pre-classification processing (specifically, a dimensionality reduction procedure) that improves the classification accuracy.

    \item 
     We establish theoretical results on the effect of various factors, such as the number of training samples, the level of class separation and training set imbalance, on the relative gain from the data processing procedure that we construct. \tomtm{For example, we show that, non-intuitively, the maximal relative gain increases when the class separation improves.}

     \item 
     We present an empirical investigation of the theoretical model that corroborates our theory and sheds more light on the gains from low-level processing. % 

    \item 
    We complement our theoretical work with an empirical study. We investigate the effect of image denoising and self-supervised encoding on the performance of practical deep classifiers on benchmark datasets, where we vary the size of the training set, the class  distribution in the training set, and the noise level in the samples. We demonstrate trends that are consistent with our theoretical results \tomtm{(e.g., the one on the maximal gain)}, highlighting the usefulness of the theoretical setup. % 

\end{itemize}

\section{Background and related work}
\label{sec:background}

Consider the classification task, where the data $(x,y)$ is distributed on $\mathcal{X} \times [C]$, with $[C]:=\{1,\ldots,C\}$ and distribution denoted by $p_{x,y}$.
For the binary $0-1$ criterion, i.e., $\ell(\hat{y},y)=\mathbb{I}(\hat{y} \neq y)$, the expected risk is equivalent to the error probability $\mathbb{E}[\ell(\hat{y}(x),y)]=\mathbb{P}(\hat{y}(x) \neq y)$. It is well-known that this objective is minimized by the (optimal) Bayes classifier: $c_{opt}(x)=\textrm{argmax}_{y \in [C]}\,\,p_{y|x}(y|x)$, where $p_{y|x}$ is the true conditional probability of $y$ given $x$ \citep{bishop2006pattern,fukunaga2013introduction}. 
In practice, of course, the distributions are unknown and a classifier must be learned from data samples. % 

Consider a data processing operation $\mathcal{A}:\mathcal{X} \to \mathcal{Z}$. 
This can be denoising, super-resolution, encoding, etc.  Let $z=\mathcal{A}(x)$.
Notice that $y \to x \to z$ is a Markov chain because $z$ is a function of $x$ and thus $p_{z|x,y}(z|x,y)=p_{z|x}(z|x)$.
Therefore, the data processing inequality in \eqref{eq:dpi} holds.
The optimal Bayes classifier that operates on a processed sample $z=\mathcal{A}(x)$ is given by $\tilde{c}_{opt}(z)=\textrm{argmax}_{y \in [C]}\,\,p_{y|z}(y|z)$, where $p_{y|z}$ is the true conditional probability of $y$ given $z$.

Focusing on the case of binary classification ($C=2$), the following result shows that, similarly to the fact that no $\mathcal{A}$ can increase \textcolor{black}{mutual} information, there is also no hope in improving the accuracy of optimal Bayes classifiers via data processing.

\begin{theorem}
\label{thm:no_gain_Bayes}
    Let $y \to x \to z$ be a Markov chain where $y \in \{1,2\}$ denotes the sample class. We have
    \begin{align}
        \label{eq:no_gain_Bayes}
    \mathbb{P}(c_{opt}(x) \neq y) \leq  \mathbb{P}(\tilde{c}_{opt}(z) \neq y),
    \end{align}
    where $c_{opt}$ and $\tilde{c}_{opt}$ denote optimal Bayes classifiers.
\end{theorem}

A similar statement and proof can be found in an arXiv version of \citep{liu2019classification}. % 
For completeness, we present a clearer proof  in Appendix \ref{app:proofs}. 
\tomtm{Note that \citep{liu2019classification} studies a potential tradeoff between the error of a low-level restoration task and the accuracy of a \emph{fixed} classifier, where only the restoration model is trained using the training data.
In contrast, our work focuses on the high-level end goal---the classification performance---and allows training the classifier after the low-level processing, as is done in practice.
Therefore, \citep{liu2019classification} does not provide any reason why in practice it is common to address a low-level task before high-level ones, which is the central question of our paper.}

Our work is motivated by the contrast between common practice and the information-theoretic concept of the data processing inequality, as well as Theorem \ref{thm:no_gain_Bayes}.
There exist works that use information-theoretic concepts or compute approximate metrics to analyze DNNs, e.g., \citep{tishby2015deep,shwartz2017opening,saxe2019information,gabrie2018entropy,jeon2022information}. Interestingly, since a DNN processes data gradually, layer by layer, the features across the layers form a Markov chain, and thus the data processing inequality applies.
Yet, avoiding the loss of information relevant to the task being learned can be attributed to penalizing failures in predicting the target labels during training, while discarding task-irrelevant information (akin to compression) may be explained by the information bottleneck principle \citep{tishby2015deep,shwartz2017opening,saxe2019information}.
\tomtm{The contrast between representation learning and the data processing inequality has also motivated theoretical works \citep{Xu2020Theory,goldfeld2021sliced} to study variants of the mutual information, incorporating transformations of the signal or line projections.} 
None of the aforementioned works consider a sequence of tasks or explain when and why low-level processing can be beneficial to practical 
classifiers. 
\tomtm{Moreover, here we directly analyze the classifier's probability of error, which is more interpretable than the information-theoretic objectives studied before.}

Finally, we emphasize that in the case of pre-trained classifier under distribution shift, data processing that `reduces the gap' between the test data distribution and the training data distribution is trivially expected to improve the classifier performance. However, we focus in this paper on the non-intuitive case where no distribution shift occurs, and the classifier is strong, in the sense that it converges to the optimal Bayes classifier as the training set increases (with good statistical properties). % 

\section{Theory}
\label{sec:theory}

In this section, we present our theoretical contributions. First, we describe the problem setup, the data distribution, the classifier under study, and a data processing operation. % 
Next, we present our theoretical results demonstrating the benefits of this data processing. Finally, we validate our results through experiments and provide additional insights into the factors that affect the performance gain.

\subsection{Problem setup: data model, classifier, and data processing}
\label{sec:setup}

\textbf{Data model.} 
Similar to a vast body of theoretical work on classifiers \citep{cao2021risk,deng2022model,wang2022binary,kothapalli2025can}, we consider binary classification ($C = 2$), where the data is distributed according to a Gaussian Mixture Model (GMM) of order two in $\mathcal{X}=\mathbb{R}^d$, with one mixture component per class. 
Formally,
\begin{equation}
\label{eq:gmm_setup}
y \in \{1,2\},
\qquad
\vct{x} \ | \ y = j \sim \mathcal{N}(\boldsymbol{\mu}_j, \sigma_j^2\boldsymbol{I}_d),
\qquad
\mathbb{P}(y=j) = \pi_j.
\end{equation}
Similar to previous theoretical works, 
we further assume that
\begin{equation}
\label{eq:model assumptions}
\boldsymbol{\mu}_2=-\boldsymbol{\mu}_1=\boldsymbol{\mu},
\qquad
\sigma_1^2=\sigma_2^2=\sigma^2,
\qquad
\pi_1=\pi_2=1/2,
\end{equation}
{\color{black} where the magnitudes of the entries of $\bmu$ are bounded by some universal constant, and $\sigma$ is independent of $d$.} % 
Let us now define the separation quality factor of the GMM data, which can be understood as the signal-to-noise ratio (SNR): % 
\begin{equation}
    \mathcal{S} := \left(\frac{\norm{\boldsymbol{\mu}_2 - \boldsymbol{\mu}_1}}{\sigma_1+\sigma_2}\right)^2 = \frac{\norm{\boldsymbol{\mu}}^2}{\sigma^2}.
\label{eq:quality factor}
\end{equation}
\tomtm{
Note that the considered setup is standard in theoretical works that aim at rigorous mathematical analysis \citep{cao2021risk,deng2022model,wang2022binary,kothapalli2025can}.
Despite its compactness, the learning problem studied in this paper can be arbitrarily hard because (unlike some of the aforementioned works) our analysis covers SNR arbitrarily close to zero, i.e., nearly indistinguishable classes.}

The training data consists of $N_j$ labeled i.i.d.~samples per class $j$, denoted by
$\mathcal{D} = \{\vct{x}_{i,j}: j\in\{1,2\}, i=1,\ldots,N_j\}$. % 
Without loss of generality, we denote $N_1=N$ and $N_2 = \gamma N$ for some $\gamma \in (0,1]$.

\textbf{The classifier.} In the considered setting, the optimal Bayes classifier reads:
\begin{align*}
    c_{opt}(\vct{x}) &= \underset{j\in\{1,2\}}{\arg\max}  \ \pi_j p_{x| y}(\vct{x}|j) = \underset{j\in\{1,2\}}{\arg\max}  \ \exp\left(-\tfrac{\norm{\vct{x}-\boldsymbol{\mu}_j}^2}{2\sigma^2}\right)
    =\underset{j\in\{1,2\}}{\arg\min} \norm{\vct{x}-\boldsymbol{\mu}_j}.
\end{align*}

In practice, the data distribution is unknown and thus a classifier cannot use the class means, $\{\boldsymbol{\mu}_i\}$, but rather estimate them from the training set.
We therefore study the classifier:
\begin{align}
    \widehat{c}(\vct{x};\mathcal{D}) = \underset{j \in \{1,2\}}{\arg\min} \norm{\vct{x} - \boldsymbol{\widehat{\mu}}_j},
\label{eq:classifier}
\end{align}
where $\widehat{\boldsymbol{\mu}}_j = \tfrac{1}{N_j} \sum \nolimits_{i=1}^{N_j} \vct{x}_{i,j}$ is the maximum likelihood estimate of $\bmu_j$ from the $j$-th class's samples.

We want to explore if data processing can be beneficial even for a ``strong'' classifier.
It is easy to see that $\widehat{\boldsymbol{\mu}}_j \sim \mathcal{N}\left(\boldsymbol{\mu}_j, \tfrac{\sigma^2}{N_j} \boldsymbol{I}_d\right)$. In fact, this is an efficient estimator that attains \tomtm{the Cramér–Rao} lower bound on the variance for \emph{any} $N_j$ \citep{kay1993fundamentals}.
Therefore, in our setting, not only that $\widehat{c}(\cdot)$ is structurally similar to $c_{opt}(\cdot)$ and converges to it for $N_j\to \infty$, but it also has strong statistical properties for finite  $N_j$, making it a natural choice for our study. 
\tomtm{Demonstrating the benefit of low-level processing for such a classifier, which is ``almost optimal'' for the considered setup, underscores the potential advantages for weaker classifiers.}

\textbf{Data processing.} 
As the pre-classification data processing, we are going to study a certain linear dimensionality reduction to $1 \le k < d$. Specifically, we consider
\begin{equation*}
    \vct{z} = \boldsymbol{A} \vct{x}
\end{equation*}
with $\boldsymbol{A} \in \mathbb{R}^{k \times d}$ that obeys
\begin{align}
\label{eq:A_properties}
    \boldsymbol{A} \boldsymbol{A}^\top = \boldsymbol{I}_k, % 
    \qquad    \norm{\boldsymbol{A}\boldsymbol{\mu}} = \norm{\boldsymbol{\mu}}.
\end{align}
\tomtm{Note that, for establishing our main theoretical claim on the practical limitation of \eqref{eq:no_gain_Bayes}, we just need the \emph{existence} of a processing for which we can rigorously show improved classification.
Nevertheless, in the sequel, we provide a constructive proof that also shows how such $\boldsymbol{A}$}
\emph{can be learned from unlabeled data} % 
without 
prior knowledge of $\boldsymbol{\mu}$. Hence, showing that this procedure improves classification performance in our setup underscores the promise of practical low-level procedures learned from unlabeled data.

\textbf{Additional notations.} 
We will analyze and compare the performance of the classifier in \eqref{eq:classifier} before and after the data processing procedure, namely, $\widehat{c}(\vct{x};\mathcal{D})$ versus $\widehat{c}(\vct{z};\mathcal{D}_{\vct{z}})$, where $\mathcal{D}_{\vct{z}} = \{\vct{z}_{i,j}=\boldsymbol{A}\vct{x}_{i,j}: j\in\{1,2\}, i=1,\ldots,N_j\}$. 
We denote the probability of error in these two cases by $p_{\vct{x}}(\mathrm{error}):=\mathbb{P}(\widehat{c}(\vct{x};\mathcal{D})\neq y)$ and $p_{\vct{z}}(\mathrm{error}):=\mathbb{P}(\widehat{c}(\vct{z};\mathcal{D}_{\vct{z}})\neq y)$.
Finally, we define the widely-used $\mathcal{Q}$-function, which will be used to characterize the classification error probability:
\begin{align}
\label{eq:q_func}
    \mathcal{Q}(x) = \mathbb{P}\left(\mathcal{N}(0,1) > x\right) = \frac{1}{\sqrt{2\pi}} \int_{x}^\infty \exp\left(-\tfrac{t^2}{2}\right) \ dt.
\end{align}

\subsection{Theoretical results}
\label{sec:theoretical_results}

In this subsection, we present our theoretical results.
{\color{black}
In Section \ref{seq:theory_gain}, we prove that the error probability decreases due to the data processing. To this end, we establish expressions that accurately approximate the probability of error of the data-driven classifier before and after the processing. We then analyze their relation, where, due to different proof strategies, this is done separately for the balanced and imbalanced training set cases. 
In Section \ref{seq:theory_factors}, we provide a fine-grained analysis of the factors that affect the efficiency of the processing, and, for the balanced training set case, we also establish a connection between the maximal gain and the SNR. 
}

The proofs for all the claims are deferred to Appendix \ref{app:proofs}.

{\color{black}
\subsubsection{Performance gain due to data processing}
\label{seq:theory_gain}

}

We begin with characterizing the probability of error when the classifier is applied without pre-processing. Recall the definitions of $\mathcal{S}$ and $ \mathcal{Q}(x)$ in \eqref{eq:quality factor} and \eqref{eq:q_func}, respectively.

\begin{theorem}[The probability of error before the processing]\label{thm:thm2} 
Consider the setup in Section \ref{sec:setup}. With approximation accuracy $\mathcal{O}(1/\sqrt{d})$ we have $p_{\vct{x}}(\mathrm{error}) \approx \hat{p}_{\vct{x}}(\mathrm{error}) = 
    \hat{p}(\mathcal{S},N,\gamma,d)$, 
where
\begin{equation}
\begin{aligned}
    \hat{p}(\mathcal{S},N,\gamma,d)
    &:= \frac{1}{2}\cdot \mathcal{Q}\left( \frac{\sqrt{\mathcal{S}} + \frac{1}{4N}\cdot \frac{1-\gamma}{\gamma}\cdot \frac{d}{\sqrt{\mathcal{S}}}}{\sqrt{\frac{1}{4N}\cdot \frac{1+\gamma}{\gamma}\cdot \frac{d}{\mathcal{S}} + \frac{1}{8N^2} \cdot \frac{1+\gamma^2}{\gamma^2}\cdot \frac{d}{\mathcal{S}} + \frac{1}{\gamma N} + 1}}\right) \\
    &\quad + \frac{1}{2}\cdot  \mathcal{Q}\left(\frac{\sqrt{\mathcal{S}} - \frac{1}{4N}\cdot \frac{1-\gamma}{\gamma}\cdot \frac{d}{\sqrt{\mathcal{S}}}}{\sqrt{\frac{1}{4N}\cdot \frac{1+\gamma}{\gamma}\cdot \frac{d}{\mathcal{S}} + \frac{1}{8N^2} \cdot \frac{1+\gamma^2}{\gamma^2} \cdot \frac{d}{\mathcal{S}} + \frac{1}{N} + 1}}\right).
\end{aligned}
\label{approx prob x}
\end{equation}
\end{theorem}

\textbf{Remark.} 
{\color{black} The proof is mathematically involved. We express the error event as thresholding a scalar random variable, suitable for an application of a generalized Berry–Esseen theorem. However, this variable depends on the interrelation between the entries of $\widehat{\bmu}_1$, $\widehat{\bmu}_2$, and computing the required moments is a technical challenge.} % 

\textbf{Discussion.} 
Note that:
1) $\hat{p}$ is symmetric in the following sense:
        $\hat{p}\left(\mathcal{S}, N, \gamma, d\right) = \hat{p}\left(\mathcal{S}, \gamma N, \tfrac{1}{\gamma}, d\right)$, which is expected because swapping the amount of samples between the classes does not change the problem; 
{\color{black} 2) As $\mathcal{S} \rightarrow 0^+$ we have: $\lim_{\mathcal{S} \rightarrow 0^+} \hat{p}(\mathcal{S},N,\gamma,d) = 1/2$, % 
aligned with uniform guess;
3) As $\mathcal{S} \rightarrow \infty$ we have: $\lim_{\mathcal{S} \rightarrow \infty} \hat{p}(\mathcal{S},N,\gamma,d) = 0$, aligned with the classes being deterministically separable;}
and
4) As $N \rightarrow \infty$ we have:
        $\lim_{N \rightarrow \infty} \hat{p}(\mathcal{S},N,\gamma,d) = \mathcal{Q}(\sqrt{\mathcal{S}})$,
    which is the probability of error of $c_{opt}$, which knows the exact distribution of the data \citep{fukunaga2013introduction}. 
Let us explore the result of Theorem \ref{thm:thm2} for the case of balanced training data, $\gamma = 1$ ($N_2=N_1=N$), in which the expression simplifies to:
\begin{equation}
\hat{p}_{\vct{x}}(\mathrm{error}) = 
    \mathcal{Q}\left(\frac{\sqrt{\mathcal{S}}}{\sqrt{\left(\frac{d}{2\mathcal{S}} + 1\right)\cdot \frac{1}{N} + \frac{d}{4\mathcal{S}}\cdot \frac{1}{N^2} + 1}}\right). % 
\end{equation} 
{\color{black} Fix $d \gg 1$, which ensures that the approximation is accurate. It is easy to see that when the separation quality factor (SNR), $\mathcal{S}$, decreases (with fixed $d$), the argument of the $\mathcal{Q}$-function decreases, and thus the probability of error increases. In addition, as the number of training samples $N$ increases, the argument increases, and thus the probability of error decreases. These two results are aligned with intuition. Interestingly, the effect of increasing/decreasing $d$ depends on its relation with $\mathcal{S}$. For example, if $d$ increases and $\mathcal{S}$ is fixed, which means that the average entry-wise SNR decreases, then the argument of the $\mathcal{Q}$-function increases.
The contrary holds if $\mathcal{S} \propto d$, which means that the average entry-wise SNR is fixed. In the latter case, high-dimensionality is advantageous in terms of the probability of error.}

Next, let us establish the existence \tomtm{and learnability} of the data processing proposed in Section \ref{sec:setup}.

\begin{theorem}[The existence and learnability of the processing]\label{thm:thm3} For all $1 \le k < d$, there exists a dimension-reducing  matrix $\boldsymbol{A} \in \mathbb{R}^{k \times d}$ with the properties stated in \eqref{eq:A_properties}.
Furthermore, given sufficiently many unlabeled samples, such a matrix can be learned to arbitrary accuracy. 
\end{theorem}

\textbf{Remark.}
{\color{black} The proof of Theorem \ref{thm:thm3} is constructive. It provides an algorithm for computing such $\boldsymbol{A}$ and efficiently estimating the direction of $\boldsymbol{\mu}$ from \textit{unlabeled} data.} 

Note that the semi-orthonormality of $\boldsymbol{A}$ implies that it cannot increase the norm of any vector, 
while the property $\norm{\boldsymbol{A}\boldsymbol{\mu}} = \norm{\boldsymbol{\mu}}$ ensures that the separation quality remains unchanged (equal to \eqref{eq:quality factor}) and is not reduced after the processing.
{\color{black} In more detail, this implies that when applying $\boldsymbol{A}\boldsymbol{x}$, the class-dependent component of $\boldsymbol{x}$ (i.e., the projection of $\boldsymbol{x}$ onto $\pm \boldsymbol{\mu}$) is not attenuated. In contrast, the complementary component of $\boldsymbol{x}$, which corresponds to within-class variability, is attenuated as the overall dimension is reduced and the semi-orthonormality of $\boldsymbol{A}$ prevents amplification. Taken together, this is expected to facilitate classification, as will be rigorously proven below. More details and a graphical illustration of the action of $\boldsymbol{A}$ are presented in Appendix \ref{app:experiments_verification_extended}.}

We now turn to  characterizing the probability of error when applying the classifier on the processed data $\boldsymbol{z} = \boldsymbol{A}\boldsymbol{x}$.

\begin{theorem}[The probability of error on the processed data]\label{thm:thm4} 
Consider the setup in Section \ref{sec:setup}. With approximation accuracy $\mathcal{O}(1/\sqrt{k})$ we have $p_{\vct{z}}(\mathrm{error}) \approx \hat{p}_{\vct{z}}(\mathrm{error}) = \hat{p}\left(\mathcal{S},N, \gamma, k\right)$,
where $\hat{p}$ is defined in \eqref{approx prob x}.
\end{theorem}

\tomtm{
The approximate probability of error of the processed data, $\hat{p}_{\vct{x}}(\mathrm{error})$, admits an expression similar to the one obtained for the raw data, $\hat{p}_{\vct{z}}(\mathrm{error})$, but with a different dimension parameter ($k$ instead of $d$). 
Note that in the high-dimensional case, i.e., $d, k \gg 1$, these estimators are guaranteed to be accurate.
Now, let us present the main outcomes of our theoretical study, which build on these expressions.}

We start with the case where there is no class imbalance in the training set, i.e., $\gamma=1$ so $N_2=N_1=N$.
The next theorem shows that the considered data processing yields a gain for any finite value $N \geq 1$. We assume $\mathcal{S} > 0$, as 
$\mathcal{S} = 0$ is an uninteresting degenerate case.

\begin{theorem}[Performance gain under balanced training data]\label{thm:thm5}
For $\gamma = 1$, and for all
    $\mathcal{S} > 0, \ 1 \le k < d$, and $N\in \mathbb{N}$,
we have
\begin{align}
    \hat{p}_{\vct{x}}(\mathrm{error}) > \hat{p}_{\vct{z}}(\mathrm{error}).
\label{eq:theorem 4}
\end{align}
\end{theorem}

Theorem \ref{thm:thm5} shows that when the training samples are balanced among the classes, the chosen processing always \emph{strictly decreases} the approximated probability of error. 

\textbf{Discussion.} 
As shown in Theorems \ref{thm:thm2} and \ref{thm:thm4}, in the high-dimensional case the true probabilities of error, ${p}_{\vct{x}}(\mathrm{error})$ and ${p}_{\vct{z}}(\mathrm{error})$, are well approximated by $\hat{p}_{\vct{x}}(\mathrm{error})$ and $\hat{p}_{\vct{z}}(\mathrm{error})$.
This makes the result significant.
Moreover, this result---holding for \textit{any} finite 
$N$---is also quite surprising, since in the limit of $N\to\infty$ we have that ${p}_{\vct{x}}(\mathrm{error})$ and ${p}_{\vct{z}}(\mathrm{error})$ converge to $\mathbb{P}(c_{opt}(\vct{x}) \neq y)$ and $\mathbb{P}(\tilde{c}_{opt}(\vct{z}) \neq y)$, respectively, which satisfy the opposite relation ($\leq$) as shown in Theorem \ref{thm:no_gain_Bayes}.

We now consider the case of an imbalanced training set. 
The presence of under-represented classes or groups is of significant interest in the machine learning community, as it raises concerns about generalization and fairness \citep{chawla2002smote,huang2016learning,li2021autobalance}.
Specifically, while the classes have equal probability ($\pi_1=\pi_2=0.5$), the number of training samples from each of the classes is assumed to be $N_1=N$ and $N_2=\gamma N$ with $0<\gamma<1$. 
The following theorem demonstrates the benefit of the considered data processing in this case as well.

\begin{theorem}[Performance gain under imbalanced training data]\label{thm:thm6} Let
    $0 < \gamma < 1, \ 0 < \mathcal{S} \le 1, \ 1 \le k < d$. % 
If $N \geq \tfrac{\gamma^2 - 4\gamma + 1}{2 \gamma(1+\gamma)}$, then
we have
\begin{align}
    \hat{p}_{\vct{x}}(\mathrm{error}) > \hat{p}_{\vct{z}}(\mathrm{error}).
\label{theorem 5}
\end{align}
\end{theorem}

\textbf{Remark.} 
Unlike Theorem \ref{thm:thm5}, which considers $\gamma=1$ and is smoothly obtained from Theorems \ref{thm:thm2} and \ref{thm:thm4}, in this case the complexity of the formulas of $\hat{p}_{\vct{x}}(\mathrm{error})$ and $\hat{p}_{\vct{z}}(\mathrm{error})$ required us to make technical assumptions on $\mathcal{S}$ and $N$ in order to establish a rigorous statement for $\gamma \in (0,1)$.
Nevertheless, these assumptions are reasonable and still encompass the interesting case of low SNR and a % 
\tomtm{reasonable}
number of training samples. {\color{black} Note that for $\gamma \ge 0.162$, the requirement $N \ge \tfrac{\gamma^2 - 4\gamma + 1}{2\gamma\left(1 + \gamma\right)}$ is vacuous (since $N \ge 1$), so it only matters under severe imbalance $(\gamma < 0.162)$.}

{\color{black}
\subsubsection{Factors that affect the performance gain}
\label{seq:theory_factors}

}

So far, we have only considered the relation between $\hat{p}_{\vct{x}}(\mathrm{error})$ and $\hat{p}_{\vct{z}}(\mathrm{error})$.
Let us now discuss the margin between them, which reflects the efficiency of the processing.

\begin{definition}
We define the {\color{black} theoretical efficiency} of the processing as
\begin{align}
\label{eq:theoretical efficiency}
    \eta \coloneqq \left(\frac{\hat{p}_{\vct{x}}(\mathrm{error}) - \hat{p}_{\vct{z}}(\mathrm{error})}{\hat{p}_{\vct{x}}(\mathrm{error})}\right)\cdot 100.
\end{align}
\end{definition}

The following theorem establishes an approximation of $\eta$ for $N \gg 1$, making it easier to gain insights into the different factors that affect the efficiency of the processing in the case of a large number of training samples.

\begin{theorem}[Analysis of the asymptotic efficiency]\label{thm:thm7} Let $\mathcal{S} > 0, \ 1 \le k < d, \ 0 < \gamma \le 1$.
Denote by $N_T = \left(1 + \gamma\right)N$ the total number of training samples. 
With approximation accuracy $\mathcal{O}(1/N_T^2)$, we have
\begin{align}
    \eta \approx \frac{25}{2 \sqrt{2\pi}}\cdot \frac{\exp\left(-\frac{\mathcal{S}}{2}\right)}{\sqrt{\mathcal{S}}\cdot \mathcal{Q}\left(\sqrt{\mathcal{S}}\right)}\cdot \left(3 + 2\gamma + \frac{1}{\gamma}\right)\cdot (d-k)\cdot \frac{1}{N_T}.
\label{efficiency approx}
\end{align}
In particular, for $N_T \gg 1$: The efficiency increases when $d-k$ increases or $\gamma$ decreases within $0 < \gamma \le 1/\sqrt{2}$; The efficiency decreases when $\mathcal{S}$ increases or $N_T$ increases. % 
\end{theorem}

\textbf{Remark.}
The proof of Theorem \ref{thm:thm7} is based on first-order analysis, which differs from the proof technique used for Theorem \ref{thm:thm6}. This allows us to reach the conclusion that there exists $N_0 \in \mathbb{N}$ such that for all $N \geq N_0$ we have $\eta>0$ (since the right-hand side of \eqref{efficiency approx} is positive) which implies $\hat{p}_{\vct{x}}(\mathrm{error}) > \hat{p}_{\vct{z}}(\mathrm{error})$ without a technical assumption on $\mathcal{S}$. On the other hand, Theorem \ref{thm:thm6} can hold even for small values of $N$, depending on $\gamma$.

\textbf{Discussion.} 
\tomtm{
Let us discuss the intuition behind the insights provided in Theorem \ref{thm:thm7}.
First, notice that in the considered regime of $N_T \gg 1$ training samples, the processing efficiency $\eta$ monotonically decreases toward zero as $N_T$ increases. This is consistent with the fact that in the limit $N_T\to \infty$ the classifier approaches the optimal Bayes classifier, which cannot be improved by data processing. 
In this regime, higher class separation $\mathcal{S}$ can be interpreted as 
equivalent to having more effective samples (akin to larger $N_T$), and hence less improvement through the pre-classification processing.
Similarly, larger dimensionality reduction $(d-k)$ can be viewed as greater coverage of the input domain, again, analogous to having more samples.
Lastly, lower $\gamma < 1$
indicates that the classifier's training samples are less balanced between the classes and hence differ more from the data distribution. Intuitively, this leaves more room for improvement through pre-classification processing.} 

{\color{black} In addition, Appendix~\ref{app:approximation_delta} provides an approximation of the difference $\Delta \coloneqq \hat{p}_{\vct{x}}(\mathrm{error}) - \hat{p}_{\vct{z}}(\mathrm{error})$ for \(N \gg 1\). 
The insights we obtain are consistent with those reported in Theorem~\ref{thm:thm7}.}

\tomtm{
So far, our theory shows that the processing efficiency $\eta$ is positive for all $N$ for $\gamma=1$, and under a technical assumption it is positive also for $\gamma\in(0,1)$. Our formulas also show that $\eta=0$ at $N=0$ (where $\hat{p}_{\vct{z}}=\hat{p}_{\vct{x}}=0.5$, i.e., probability of guessing) and that $\eta \to 0$ at $N \to \infty$ (where $\hat{p}_{\vct{z}}=\hat{p}_{\vct{x}}=\mathcal{Q}(\sqrt{\mathcal{S}})$, consistent with the classifier converging to the optimal Bayes classifier for which Theorem~\ref{thm:no_gain_Bayes} applies).
Together, these imply that there is a maximum point of $\eta(N)$. 
Our final theorem provides a surprising insight into this maximum efficiency.} 

\begin{theorem}[Analysis of the maximal efficiency]\label{thm:thm8} Fix $ \gamma = 1$, and let $\mathcal{S} > 0, 1 \le k < d$.
\tomtm{Consider the efficiency $\eta=\eta(N)$ as a function of continuous $N \in \mathbb{R}_+$. We have that 
    the maximal efficiency $\eta_{\max}=\max_{N \ge 0} \ \eta(N)$ increases as a function of $\mathcal{S}$.}
\end{theorem}

\textbf{Discussion.} 
\tomtm{
In the asymptotic regime of $N \rightarrow \infty$, as discussed above, a higher SNR corresponds to lower $\eta$, which aligns with intuition. Interestingly, however, the theorem shows that a higher SNR also leads to a larger $\eta_{\max}$, which is somewhat counterintuitive. One might expect that lower noise would reduce efficiency across all sample sizes, since the raw data is already well-separated. % 
This highlights the subtle relationship between $\eta$ and the SNR. 
An extended version of the theorem can be found in Appendix~\ref{app:ext_max_eta_thm}.} 

\subsection{Empirical verification}
\label{sec:empirical_verification}

In this subsection,
we simulate the theoretical setup in order to further support our theoretical results and also gain more insights on the model, e.g., factors that affect the efficiency of the data processing for small to moderate values of $N$.

{\color{black} We consider data dimension $d=2000$, and fix $\sigma = 1$. The SNR values we work with are % 
$\mathcal{S} \in \{0.75^2, 1.5^2\}$, 
and for each fixed SNR, we use $\gamma \in \{0.25, 0.5, 1\}$. We also consider a wide range of $N_{\text{train}}$, which denotes the total number of given training samples. For each fixed tuple $(\mathcal{S}, \gamma, N_{\text{train}})$, we randomize $\boldsymbol{\mu} \in \mathbb{R}^d$ with $\norm{\boldsymbol{\mu}} = \sigma\sqrt{\mathcal{S}}$, via $\boldsymbol{\mu} = \sigma\sqrt{\mathcal{S}} \tfrac{\boldsymbol{v}}{\norm{\boldsymbol{v}}}$ where $\boldsymbol{v} \sim \mathcal{N}\left(\boldsymbol{0}, \boldsymbol{I}_d\right)$.
We then construct the data processing matrix $\boldsymbol{A} \in \mathbb{R}^{k \times d}$ that reduces the dimension to $k = 1000$, using the algorithm described in Appendix~\ref{app:data_proc_thm}.
Per trial, we sample $N_1 = \mathrm{int}(\tfrac{N_{\text{train}}}{1 + \gamma})$ training points from $\mathcal{N}(-\boldsymbol{\mu}, \sigma^2 \boldsymbol{I}_d)$, and $N_2 = \mathrm{int}(\tfrac{\gamma N_{\text{train}}}{1 + \gamma})$ training points from $\mathcal{N}(\boldsymbol{\mu}, \sigma^2  \boldsymbol{I}_d)$.
Before and after the data processing, the per-class means are estimated using the training points, and the classifier defined in \eqref{eq:classifier} is used on a large amount of fresh test data, sampled with probability $0.5$ from each of the two Gaussians.
With a slight abuse of notation, we denote the empirical probabilities of error before and after the data processing by $p_x(\mathrm{error})$ and $p_z(\mathrm{error})$, respectively. In order to compute $p_z(\mathrm{error})$, we use the training and test samples after processing them by multiplying with $\boldsymbol{A}$. We then compute the empirical efficiency of the processing, defined by $\chi = \left(\tfrac{p_{\vct{x}}(\mathrm{error}) - p_{\vct{z}}(\mathrm{error})}{p_{\vct{x}}(\mathrm{error})}\right)\cdot 100$.
We repeat the computation of $\chi$ over 100 independent trials and report the average.}

Figure \ref{fig:EfficiencyComparisonAll} presents both the theoretical efficiency $\eta$, defined in \eqref{eq:theoretical efficiency}, and the empirical efficiency $\chi$ versus the number of training samples $N_{\text{train}}$, for various values of $\gamma$ and $\mathcal{S}$. % 
Note that the empirical and theoretical efficiencies closely match in all the configurations. % 

\begin{figure}[t]
    \centering
    \begin{subfigure}[b]{0.49\linewidth}
        \centering
        \includegraphics[width=\linewidth]{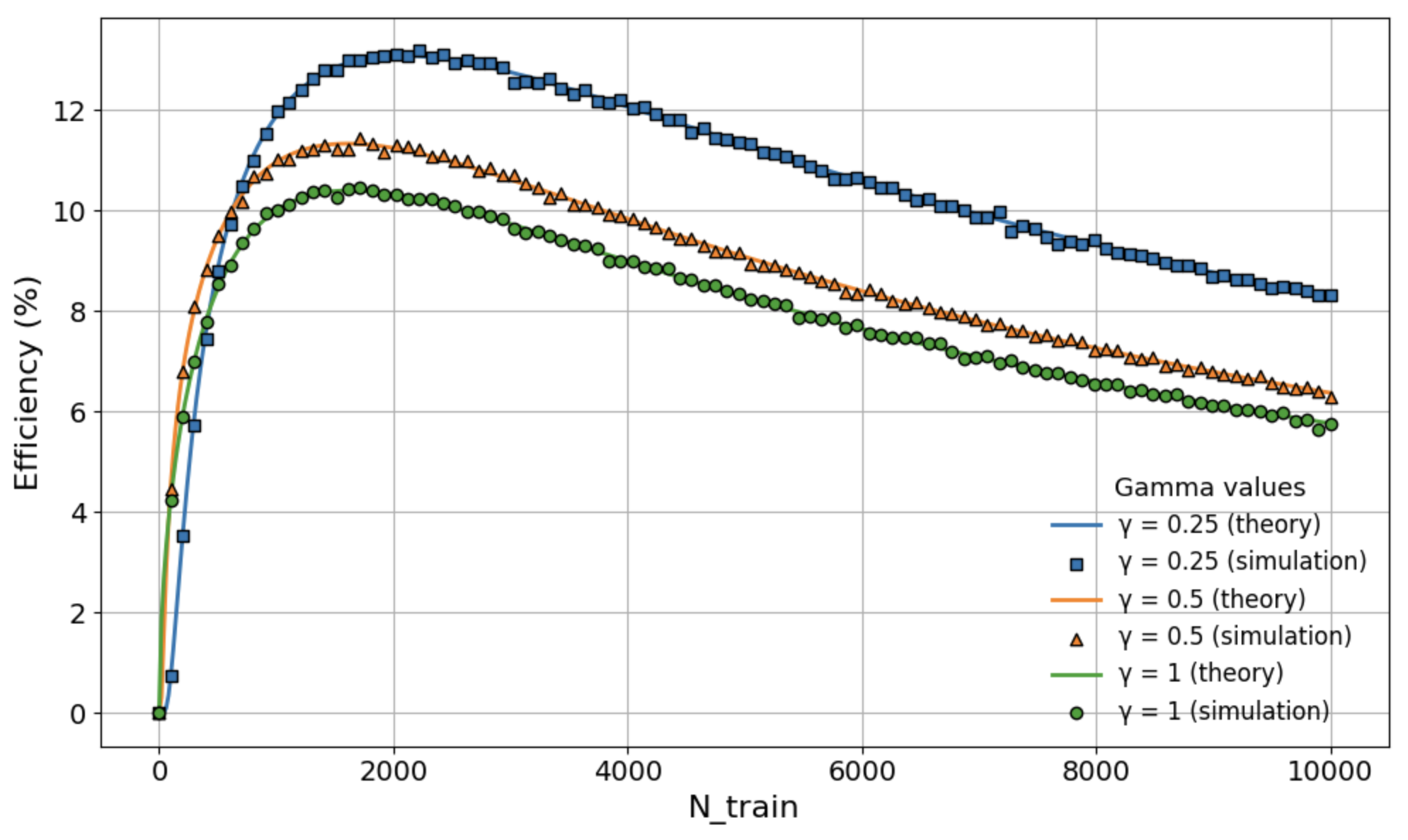}
        \caption{$\mathcal{S}=0.75^2$}
        \label{fig:VerFor_S=0p75_squared}
    \end{subfigure}
    \hfill
    \begin{subfigure}[b]{0.49\linewidth}
        \centering
        \includegraphics[width=\linewidth]{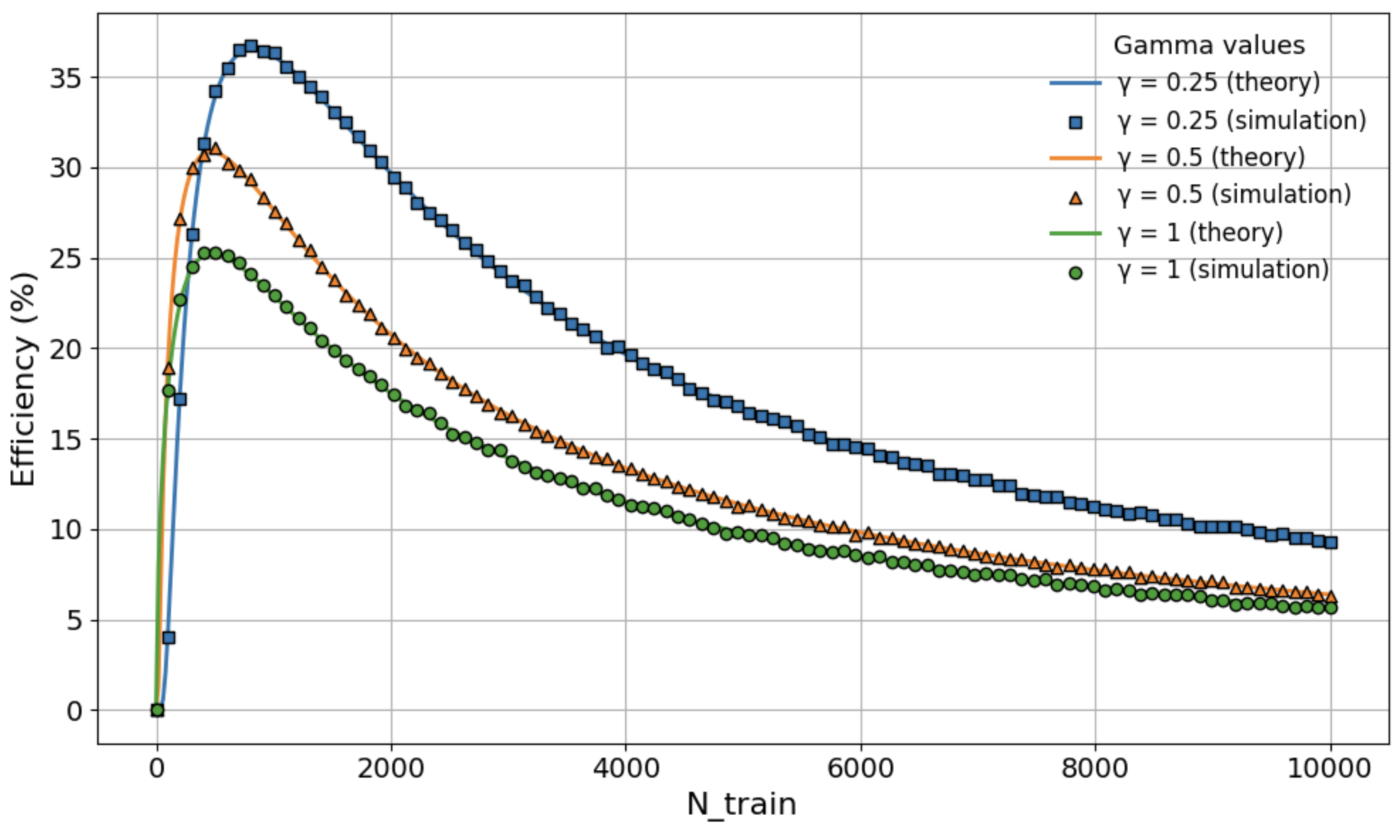}
        \caption{$\mathcal{S}=1.5^2$}
        \label{fig:VerFor_S=1p5_squared}
    \end{subfigure}   

\vspace{-3mm}
    \caption{The theoretical setup. Efficiency of the data processing procedure versus the number of training samples $N_{\text{train}}$, for various values of the training imbalance factor, $\gamma$, and the SNR, $\mathcal{S}$.}
    \label{fig:EfficiencyComparisonAll}
    \vspace{-4mm}
\end{figure}

Let us discuss the trends that are observed in Figure \ref{fig:EfficiencyComparisonAll}.
First, note the non-monotonic curves depicting the efficiency as a function of $N_{\text{train}}$. When $N_{\text{train}}$ approaches zero or grows to infinity the efficiency tends to zero, aligned with our analytical formulas. Indeed, as discussed above, in the absence of training data the classification is based on guess, and thus there is no effect for the data processing.
In the considered setup, as $N_{\text{train}} \to \infty$, the classifiers tend to the optimal Bayes decision rules, which again implies zero efficiency.
A major contribution of our paper is providing rigorous theory for the fact that the efficiency remains positive between these two extreme cases. 

Let us now focus on $N_{\text{train}} \gg 1$ (the right boundary of each sub-figure). We see that increased $\mathcal{S}$ moderately reduces the efficiency. For example, for $(\mathcal{S}, \gamma, N_{\text{train}})=(0.75^2,1,10K)$ the efficiency is around 6, while for $(\mathcal{S}, \gamma, N_{\text{train}})=(1.5^2,1,10K)$ it is around 5.
Moreover, we see that lower values of $\gamma$, corresponding to more imbalanced training data, yield higher efficiency of the data processing.
Note that both are aligned with the insights gained in Theorem \ref{thm:thm7}. % 

Next, note that each of the curves depicts a single maximum point, whose value % 
is aligned with the non-intuitive prediction of Theorem~\ref{thm:thm8}.
Specifically, the maximal efficiency value increases with $\mathcal{S}$. % 

Lastly, note that the empirical investigation of our theoretical setup reveals behaviors at relatively small values of $N_{\text{train}}$, which lie beyond the scope of our theoretical analysis. Specifically, we observe that the relation between decrease in $\gamma$ and increase in efficiency emerges already at quite low $N_{\text{train}}$.  
We also observe dependency between the overall shape of the curves and the value of $\mathcal{S}$. 

Additional verification experiments with $\boldsymbol{A}$ that is learned from unlabeled samples, and different values of $\mathcal{S}, k$ are presented in Appendix~\ref{app:experiments_verification_extended}. All of them are aligned with our theoretical insights.

\vspace{-1mm}
\section{Experiments in Practical Settings}
\label{sec:empirical}
\vspace{-2mm}

While our paper focuses on theoretical contributions,
in this section, we empirically examine the correlation between the behaviors observed in {\color{black} four} practical deep learning settings and the theoretical results. Note that such a study, which examines the effects of sample size, SNR, and class balance, requires exhaustive training efforts of both the data-processing module and the classifier.

\begin{figure}[t]
    \centering
    \begin{subfigure}[b]{0.4\linewidth}
        \centering
        \includegraphics[width=\linewidth]{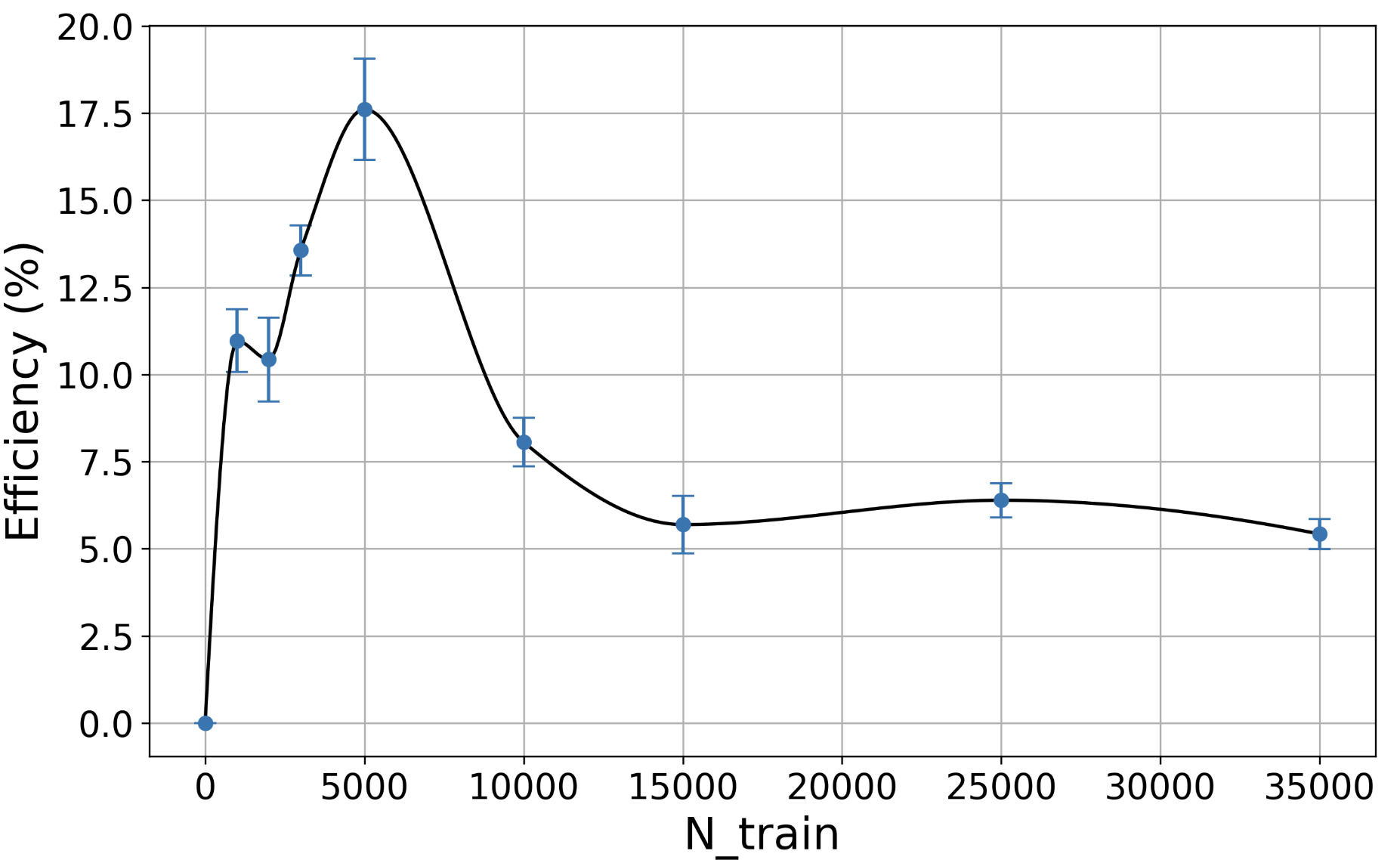}
        \caption{$\sigma = 0.25, \gamma = 1$}
        \label{fig:Sigma_0p25_gamma_1_dncnn}
    \end{subfigure}
    \hspace{3mm}
    \begin{subfigure}[b]{0.4\linewidth}
        \centering
        \includegraphics[width=\linewidth]{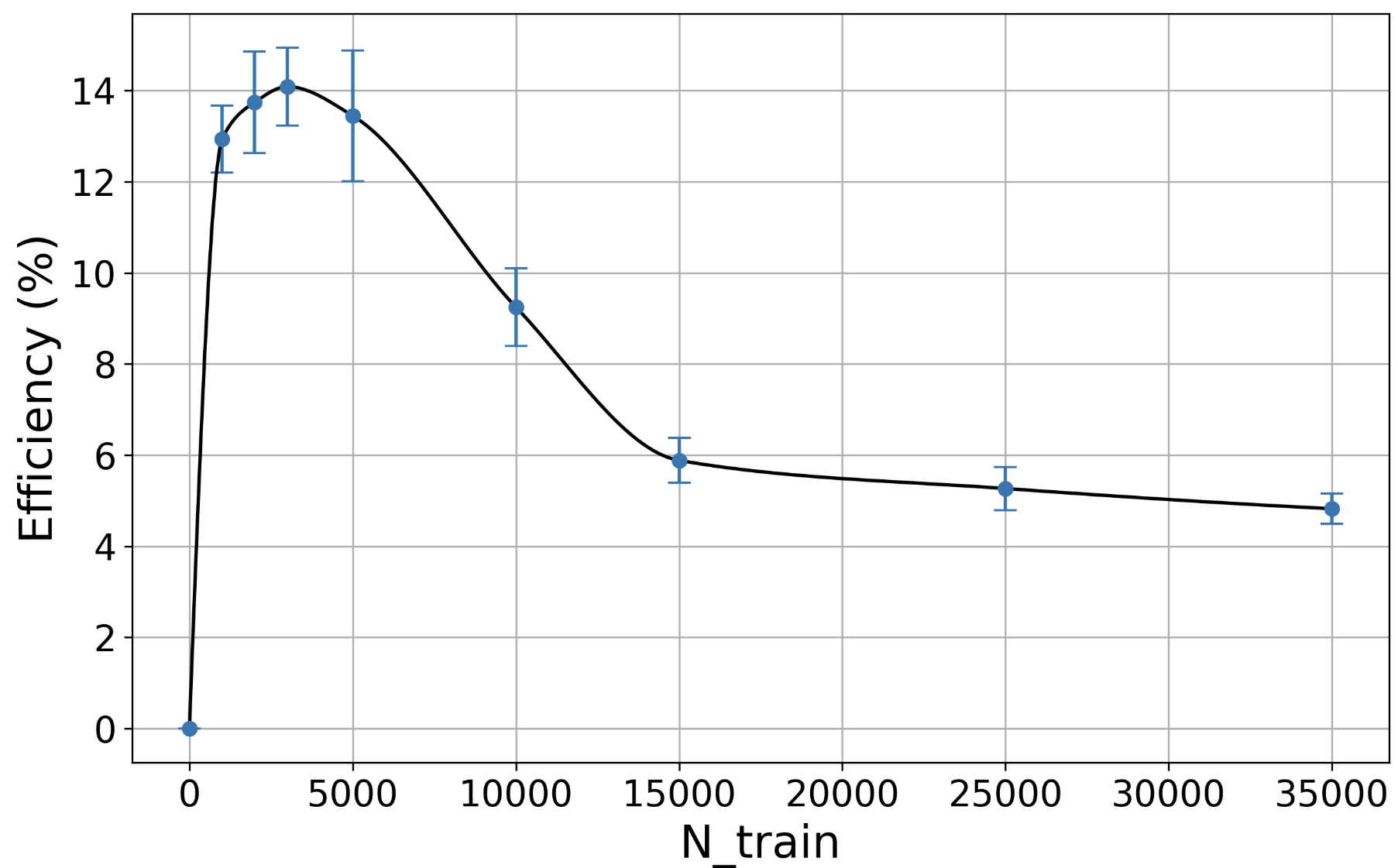}
        \caption{$\sigma = 0.4, \gamma = 1$}
        \label{fig:Sigma_0p4_gamma_1_dncnn}
    \end{subfigure}

\vspace{-2mm}
    \caption{Noisy CIFAR-10 and pre-classification denoising. Efficiency versus $N_{\text{train}}$.}
    \label{fig:Practical_Networks_Efficiency}
    \vspace{-5mm}
\end{figure}

\subsection{Noisy CIFAR-10 and pre-classification denoising}
\label{sec:exp_cifar10_denoising}

We consider the CIFAR-10 dataset \citep{krizhevsky2009learning} and the ResNet18 model \citep{he2016deep}. 
The train and test sets both experience additive Gaussian noise of the same level (i.e., no distribution shift) with standard deviation $\sigma \in \{0.25, 0.4\}$.
A detailed description of the training procedure of the classifier is given in Appendix \ref{app:experiments}.
We also note that we verify that the classifier performs well when trained on clean CIFAR-10 data, achieving 90\% accuracy.

The data processing step examined here is image denoising, applied to the noisy data, using the DnCNN model \citep{zhang2017beyond}. The denoiser is trained with the MSE loss on 15,000 clean unlabeled images, which are not part of the classifier's training set.
More details on the training procedure of the denoiser are given in Appendix \ref{app:experiments}.
Note that, given such a pretrained denoiser, the Markov chain: ``label''---``noisy image''---``denoised image'' still holds.
Thus, the data processing inequality, as well as Theorem \ref{thm:no_gain_Bayes}, suggest that the denoiser will not improve the results.

In Appendix \ref{app:experiments}, 
we also investigate 
{\color{black} another setting}, where we
train the denoiser with SURE loss (i.e., without clean ground truth images) \citep{stein1981estimation,soltanayev2018training}, and observe similar results. 

We consider various values of $N_{\text{train}}$, the total number of given training samples (across all 10 classes), and examine different training imbalance factors, $\gamma=1$ here, and  $\gamma \in \{0.5, 0.75\}$ in Appendix~\ref{app:experiments}. For both the denoised and the noisy case, and for each fixed tuple $(\sigma, \gamma, N_{\text{train}})$, we divide $\tfrac{N_{\text{train}}}{1 + \gamma}$ equally among the first 5 classes, and $\tfrac{\gamma N_{\text{train}}}{1 + \gamma}$ equally among the other 5 classes. 
We train the classifier 6 times, each time with a different seed, and report the average and standard deviation of the probabilities of error, to obtain a more reliable result. After we have the mean and standard deviation of the probability of error before and after the data processing, we compute the empirical efficiency, % 
i.e., the relative percentage change in the probability of error induced by the denoising step.

Figure \ref{fig:Practical_Networks_Efficiency} presents the efficiency versus $N_{\text{train}}$. % 
We see two main similarities to the theory. First, the non-monotonic behavior (increasing for small $N_{\text{train}}$ and decreasing for large $N_{\text{train}}$) is expected from the same argument in Section \ref{sec:empirical_verification}: the efficiency tends to zero as $N_{\text{train}}$ tends to either 0 or $\infty$, while, importantly, it remains positive between these two extreme cases, aligned with our theory.
Second, we see that the maximal efficiency value decreases with $\sigma$: % 
its value for $\sigma = 0.25$ is larger than its value for $\sigma= 0.4$. That is, the maximal efficiency increases with the SNR.

\vspace{-2mm}
\subsection{Noisy Mini-ImageNet and pre-classification encoding}
\vspace{-2mm}

We turn to investigate a more complex data processing pipeline using the Mini-ImageNet dataset \citep{vinyals2016matching} and the ResNet50 model. Both the training and test sets are subjected to additive Gaussian noise with standard deviations $\sigma \in \{50/255, 100/255\}$. The data processing step examined here is an encoding step, which maps the images from $224 \times 224$ pixels to 256-dimensional embeddings. This encoder model follows \citep{lu2025ditch} and is trained from scratch with self-supervision on all noisy unlabeled images for each noise level. Then, for each combination of $(\sigma, \gamma, N_{\text{train}})$, we divide $\tfrac{N_{\text{train}}}{1 + \gamma}$ equally among the first 50 classes, and $\tfrac{\gamma N_{\text{train}}}{1 + \gamma}$ equally among the other 50 classes. Then, across three seeds, we train a ResNet50 model on the noisy images and, in parallel, a small MLP on the corresponding embeddings. After we have the mean and standard deviation of the probability of error before and after the data processing, we compute the empirical efficiency, i.e., the relative percentage change in the probability of error induced by the encoding step. Details of the training procedures for both the ResNet50 and the MLP are provided in Appendix~\ref{app:experiments_encoding}. % 
 
Figure \ref{fig:Practical_Networks_Efficiency_encoding} presents the efficiency versus $N_{\text{train}}$, for $\gamma = 1$.
Experiments for $\gamma \in \{0.5, 0.75\}$ appear in Appendix~\ref{app:experiments_encoding}.
We see the same trends that are aligned with our theory as before: 1) similar non-monotonicity of the curve while remaining positive, and 2) the maximal efficiency increases with the SNR.
A message to practitioners is that when labeled samples are scarce, data processing can be especially advantageous for `high quality' data.

\begin{figure}[t]
    \centering
    \begin{subfigure}[b]{0.4\linewidth}
        \centering
        \includegraphics[width=\linewidth]{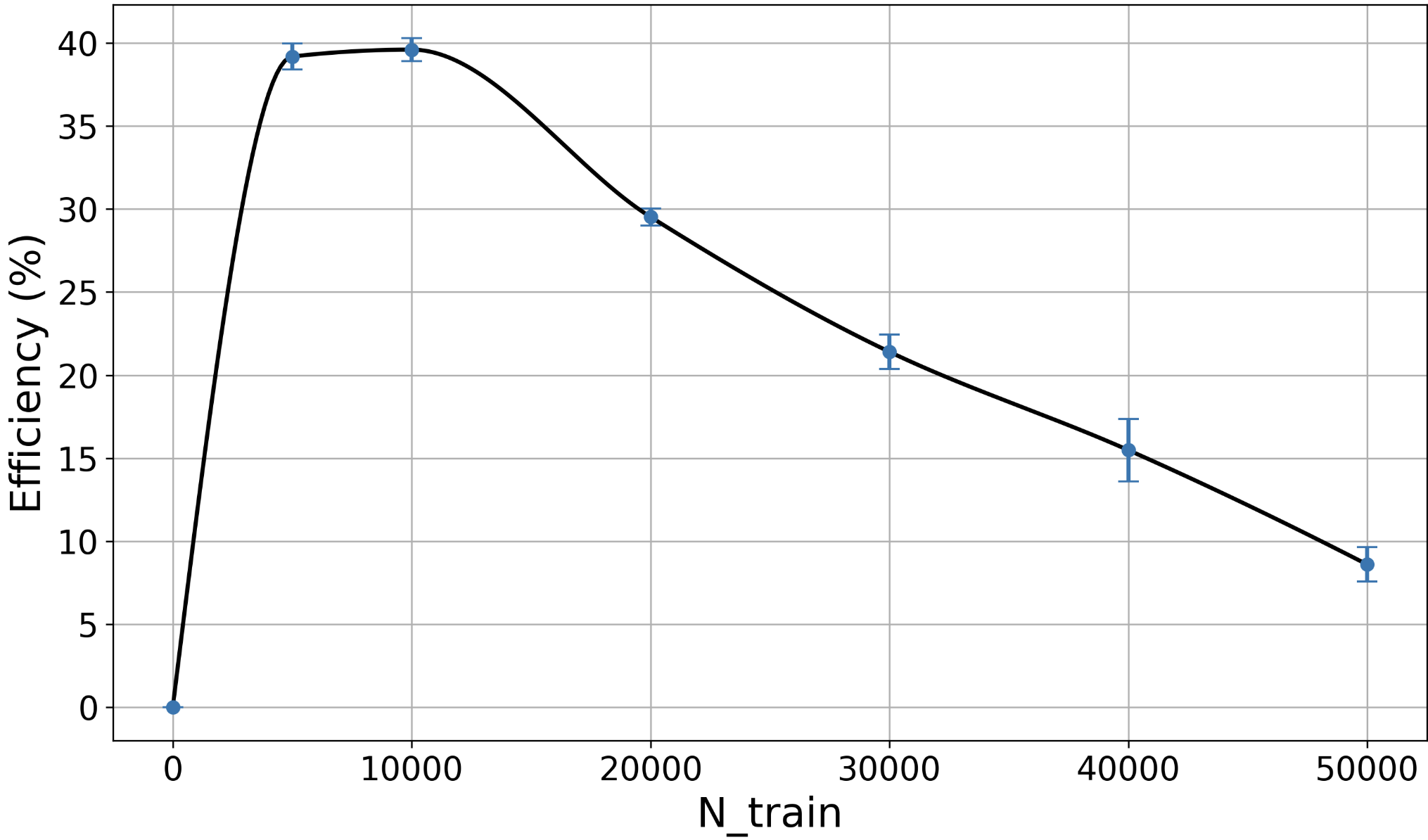}
        \caption{$\sigma = \tfrac{50}{255}, \gamma = 1$}
        \label{fig:Sigma_50_gamma_1_embedding}
    \end{subfigure}
    \hspace{3mm}
    \begin{subfigure}[b]{0.4\linewidth}
        \centering
        \includegraphics[width=\linewidth]{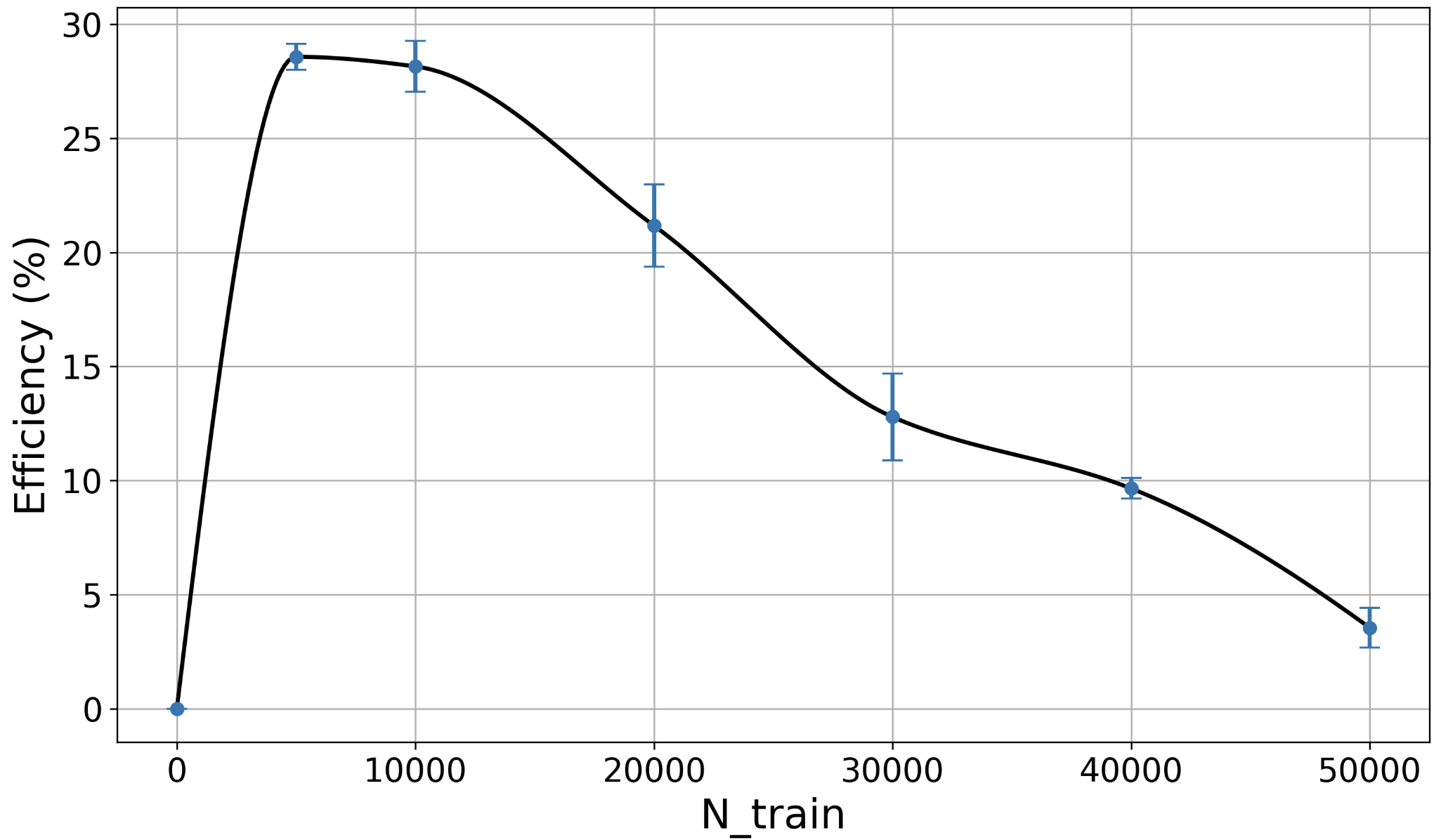}
        \caption{$\sigma = \tfrac{100}{255}, \gamma = 1$}
        \label{fig:Sigma_100_gamma_1_embedding}
    \end{subfigure}
\vspace{-2mm}
    \caption{Noisy Mini-ImageNet and pre-classification encoding. Efficiency versus $N_{\text{train}}$.}
    \label{fig:Practical_Networks_Efficiency_encoding}
    \vspace{-4mm}
\end{figure}

{\color{black}
\vspace{-2mm}
\subsection{Noisy CIFAR-10 and pre-classification encoding}
\vspace{-2mm}

For the noisy CIFAR-10 setup considered in Section \ref{sec:exp_cifar10_denoising}, we also examine the performance of data processing based on encoding instead of denoising. Due to space limitations, the details are deferred to Appendix \ref{app:experiments_encoding_cifar10}, and the results are presented there in Figure~\ref{fig:Practical_Networks_Efficiency_encoding_cifar10}. The trends stated above are observed there as well.

These results further demonstrate higher efficiency values compared to those obtained for the denoising procedure in Section \ref{sec:exp_cifar10_denoising}, indicating that, for the classification task, encoding may be a more effective low-level processing method than denoising. However, we believe that this may not be the case for other high-level tasks, which may require preserving spatial information in the image (e.g., object detection). % 

}

\vspace{-1mm}
\section{Conclusion}
\label{sec:conclusion}
\vspace{-1mm}

In this paper, we addressed the question: How can we explain the common practice of performing a ``low-level'' task before a ``high-level'' downstream task, such as classification, despite theoretical principles like the data processing inequality and the overwhelming capabilities of modern deep neural networks?
We presented a theoretical study of a binary classification setup, where we considered a ``strong'' classifier that is tightly connected to the optimal Bayes classifier (and converges to it), and yet, we constructed a pre-classification processing step that for any finite number of training samples provably improves the classification accuracy. We also provided both theoretical and empirical insights into various factors that affect the gains from such low-level processing.
Finally, we demonstrated that the trends observed in {\color{black} four} practical deep learning settings, where image denoising or encoding is applied before image classification, are consistent with those established by our theoretical study.

Our work motivates ongoing research on signal and image restoration and enhancement \citep{garber2024image,garber2025zero,zhang2025improving,hen2025robust} and other low-level tasks. Since it 
shows the benefit of low-level tasks even when the classifier's training and test data share the same distribution, it naturally suggests an \emph{even greater advantage} in out-of-distribution scenarios.
{\color{black} % 
As directions for future research, it would be valuable to extend the theoretical analysis to high-level tasks beyond classification or to investigate non-linear low-level processing. Another interesting direction is to study the optimal low-level processing corresponding to a given high-level task.}

% {\textit{Remark}.}
% In this paper, we used LLMs only to polish the writing.

% 
% 
% 
% 
% 

% 
% 
% 

% 
% 
% 
%\vspace{-1mm}
\section*{Acknowledgments}
%\vspace{-1mm}
The work was supported by the Israel Science Foundation (No. 1940/23) and MOST (No. 0007091) grants.
LLMs were used only to polish the writing.

\bibliography{refs}
\bibliographystyle{iclr2026_conference}

\newpage
\appendix

\section{Proofs}
\label{app:proofs}

\subsection{Existing results}

Let us present a proof for Theorem \ref{thm:no_gain_Bayes}, which is similar to a proof that can be found in an arXiv version of \citep{liu2019classification}, but better clarifies how the Markovianity is used.

\textbf{Theorem} $\mathbf{1.}$ 
\emph{
    Let $y \to x \to z$ be a Markov chain where $y \in \{1,2\}$ denotes the sample class. We have
    \begin{align*}
    \mathbb{P}(c_{opt}(x) \neq y) \leq  \mathbb{P}(\tilde{c}_{opt}(z) \neq y),
    \end{align*}
    where $c_{opt}$ and $\tilde{c}_{opt}$ denote optimal Bayes classifiers.
}

\begin{proof}
    Let us denote by $\mathcal{X}, \mathcal{Z}$ the supports of $x, z$, respectively, and by
    \begin{equation}
        P_1 \coloneqq \mathbb{P}(y=1), \ P_2 \coloneqq \mathbb{P}(y=2),
    \label{prior of y}
    \end{equation}
    the prior probability for the binary label $y \in \{1,2\}$.
    Let us also define:
    \begin{equation}
        p_{x_1}(\xi) \coloneqq \mathbb{P}\left(x = \xi \ | \ y=1\right), \ p_{x_2}(\xi) \coloneqq \mathbb{P}\left(x = \xi \ | \ y=2\right).
    \label{likelihood of x}
    \end{equation}
    Now, from \eqref{prior of y} and \eqref{likelihood of x}, the probability of error of the optimal Bayes classifier on $x$ reads:
    \begin{equation}
    \begin{aligned}
        \mathbb{P}\left(c_{opt}(x) \ne y\right) &= \sum_{\xi \in \mathcal{X}} \min\left(P_1 p_{x_1}(\xi), P_2 p_{x_2}(\xi)\right) 
        \\
        &= \frac{1}{2} - \frac{1}{2}\sum_{\xi \in \mathcal{X}} \abs{P_1 p_{x_1}(\xi) - P_2 p_{x_2}(\xi)}.
    \end{aligned}
    \label{expression for p(c_opx(x) ne y)}
    \end{equation}
    Similarly to \eqref{likelihood of x}, we define:
    \begin{equation}
        p_{z_1}(\zeta) := \mathbb{P}\left(z = \zeta \ | \ y=1\right), \ p_{z_2}(\zeta) := \mathbb{P}\left(z = \zeta \ | \ y = 2 \right).
    \label{likelihood of z}
    \end{equation}
    From \eqref{prior of y} and \eqref{likelihood of z}, the probability of error of the optimal Bayes classifier on $z$ reads:
    \begin{equation}
    \begin{aligned}
        \mathbb{P}\left(\tilde{c}_{opt}(z) \ne y\right) &= \sum_{\zeta \in \mathcal{Z}} \min\left(P_1 p_{z_1}(\zeta), P_2 p_{z_2}(\zeta)\right) 
        \\
        &= \frac{1}{2} - \frac{1}{2}\sum_{\zeta \in \mathcal{Z}} \abs{P_1 p_{z_1}(\zeta) - P_2 p_{z_2}(\zeta)}.
    \end{aligned}
    \label{expression for p(c tilde opt(z) ne y)}
    \end{equation}
    From \eqref{likelihood of z} and the Markov assumption, we expand:
    \begin{equation}
    \begin{aligned}
        p_{z_i}(\zeta) &= \mathbb{P}\left(z = \zeta \ | \ y= i\right) = \sum_{\xi \in \mathcal{X}} \mathbb{P}\left(z = \zeta, x = \xi \ | \ y = i\right)
        \\
        &= \sum_{\xi \in \mathcal{X}} \mathbb{P}\left(z = \zeta \ | \ x = \xi , \ y=i\right) \mathbb{P}\left(x = \xi \ | \ y=i\right)
        \\
        &= \sum_{\xi \in \mathcal{X}} \mathbb{P}\left(z = \zeta \ | \ x = \xi\right) \mathbb{P}\left(x = \xi \ | \ y=i\right)
        \\
        &= \sum_{\xi \in \mathcal{X}} p_{z|x}(\zeta \ | \ \xi) p_{x_i}(\xi).
    \end{aligned}
    \label{p_(z_i) as a function of p_(x_i)}
    \end{equation}
    The key step is the fourth equality, which eliminates the dependence on $y$ in the first factor of the summand. We also denote
    \begin{equation*}
        p_{z | x}(\zeta \ | \ \xi) \coloneqq \mathbb{P}\left(z = \zeta \ | \ x = \xi\right).
    \end{equation*}
    By substituting \eqref{p_(z_i) as a function of p_(x_i)} into \eqref{expression for p(c tilde opt(z) ne y)}, we get:
    \begin{equation}
    \begin{aligned}
        \mathbb{P}\left(\tilde{c}_{opt}(z) \ne y\right) &= \frac{1}{2} - \frac{1}{2}\sum_{\zeta \in \mathcal{Z}} \abs{P_1 p_{z_1}(\zeta) - P_2 p_{z_2}(\zeta)}
        \\
        &= \frac{1}{2} - \frac{1}{2}\sum_{\zeta \in \mathcal{Z}} \abs{\sum_{\xi \in \mathcal{X}} P_1p_{z|x}(\zeta \ | \ \xi)p_{x_1}(\xi) - P_2p_{z|x}(\zeta \ | \ \xi)p_{x_2}(\xi)}
        \\
        &\ge \frac{1}{2} - \frac{1}{2}\sum_{\zeta \in \mathcal{Z}}\sum_{\xi \in \mathcal{X}} p_{z|x}(\zeta \ | \ \xi)\cdot \abs{P_1p_{x_1}(\xi) - P_2p_{x_2}(\xi)}
        \\
        &= \frac{1}{2} - \frac{1}{2}\sum_{\xi \in \mathcal{X}}\sum_{\zeta \in \mathcal{Z}} p_{z|x}(\zeta \ | \ \xi)\cdot \abs{P_1p_{x_1}(\xi) - P_2p_{x_2}(\xi)}
        \\
        &= \frac{1}{2} - \frac{1}{2}\sum_{\xi \in \mathcal{X}} \left(\abs{P_1p_{x_1}(\xi) - P_2p_{x_2}(\xi)}\cdot \sum_{\zeta \in \mathcal{Z}} p_{z|x}(\zeta \ | \ \xi)\right)
        \\
        &= \frac{1}{2} - \frac{1}{2}\sum_{\xi \in \mathcal{X}} \abs{P_1p_{x_1}(\xi) - P_2p_{x_2}(\xi)}
        \\
        &= \mathbb{P}\left(c_{opt}(x) \ne y\right).
    \end{aligned}
    \end{equation}
\end{proof}

Let us now present a theorem that will be utilized in the proof of Theorem \ref{thm:thm2}. 

\begin{theorem}[Generalized Berry-Esseen Theorem, \citep{feller1991introduction}] \label{thm: generalized_barry_essen} Let $X_1,X_2,\dots,X_d$ be independent random variables with:
\begin{itemize}
    \item Means $\eta_i=\mathbb{E}[X_i]$.
    \item Variances $\xi_i^2=\operatorname{Var}(X_i)$.
    \item Third absolute moments $\rho_i = \mathbb{E}\left[\abs{X_i - \eta_i}^3\right]$.
\end{itemize}
Define the normalized sum:
\begin{equation*}
    S_d = \frac{1}{\sqrt{\sum_{i=1}^d \xi_i^2}} \sum_{i=1}^d (X_i - \eta_i).
\end{equation*}
Then, there exists an absolute constant $C_0 > 0$ independent of $d$ such that:
\begin{equation*}
    \sup_{x \in \mathbb{R}} \abs{\mathbb{P}(S_d > x) - \mathcal{Q}(x)} \le \frac{C_0\sum_{i=1}^d \rho_i}{\left(\sum_{i=1}^d \xi_i^2\right)^{\frac{3}{2}}}.
\end{equation*}
\end{theorem}

In the following subsections, we present the proofs of Theorems \ref{thm:thm2}, \ref{thm:thm3}, \ref{thm:thm4}, \ref{thm:thm5}, \ref{thm:thm6}, \ref{thm:thm7}, \ref{thm:thm8}. % 
\\
\\
\subsection{Proof of Theorem \ref{thm:thm2}}

\begin{proof}
    The probability of error is
    \begin{equation}
    \begin{aligned}
        p_{\vct{x}}(\mathrm{error}) &= \mathbb{P}\left(\widehat{c}(\vct{x}) \ne y\right) = \pi_1\cdot \mathbb{P}\left(\widehat{c}(\vct{x}) = 2 \ | \ y = 1\right) + \pi_2 \cdot \mathbb{P}\left(\widehat{c}(\vct{x}) = 1 \ | \ y=2\right)
        \\
        &= \frac{1}{2}\cdot \mathbb{P}\left(\widehat{c}(\vct{x}) = 2 \ | \ y = 1\right) + \frac{1}{2} \cdot \mathbb{P}\left(\widehat{c}(\vct{x}) = 1 \ | \ y=2\right)
        \\
        &= \frac{1}{2}\cdot q(1,2) + \frac{1}{2}\cdot q(2,1)
    \end{aligned}
    \label{full probability of error expansion}
    \end{equation}

    where we used the assumption of a uniform prior and defined
    \begin{equation}
        q(i,j) = \mathbb{P}\left(\widehat{c}(\vct{x}) = j \ | \ y=i\right).
    \end{equation}
    Following \eqref{eq:classifier}, the first conditional probability of error $q(1,2)$ reads:
    \begin{equation}
    \begin{aligned}
        q(1,2) &= \mathbb{P}\left(\norm{\vct{x} - \widehat{\boldsymbol{\mu}}_2}^2 < \norm{\vct{x} - \widehat{\boldsymbol{\mu}}_1}^2 \ | \ y=1\right)
        \\
        &= \mathbb{P}\left((\vct{x}-\widehat{\boldsymbol{\mu}}_2)^\top(\vct{x}-\widehat{\boldsymbol{\mu}}_2) < (\vct{x}-\widehat{\boldsymbol{\mu}}_1)^\top(\vct{x}-\widehat{\boldsymbol{\mu}}_1) \ | \ y=1 \right)
        \\
        &= \mathbb{P}
        \left(\vct{x}^\top\vct{x} - \vct{x}^\top\widehat{\boldsymbol{\mu}}_2 - \widehat{\boldsymbol{\mu}}_2^\top\vct{x} + \widehat{\boldsymbol{\mu}}_2^\top\widehat{\boldsymbol{\mu}}_2 < \vct{x}^\top\vct{x} - \vct{x}^\top\widehat{\boldsymbol{\mu}}_1 - \widehat{\boldsymbol{\mu}}_1^\top\vct{x} + \widehat{\boldsymbol{\mu}}_1^\top\widehat{\boldsymbol{\mu}}_1 \ | \ y= 1\right)
        \\
        &= \mathbb{P}\left(-2\widehat{\boldsymbol{\mu}}_2^\top\vct{x} + 
        \norm{\widehat{\boldsymbol{\mu}}_2}^2 < -2\widehat{\boldsymbol{\mu}}_1^\top \vct{x} + \norm{\widehat{\boldsymbol{\mu}}_1}^2 \ | \ y=1\right)
        \\
        &=
        \mathbb{P}\left(2 (\widehat{\boldsymbol{\mu}}_2-\widehat{\boldsymbol{\mu}}_1)^\top\vct{x} > \norm{\widehat{\boldsymbol{\mu}}_2}^2 - \norm{\widehat{\boldsymbol{\mu}}_1}^2 \ | \ y=1 \right)
        \\
        &=
        \mathbb{P}\left((\widehat{\boldsymbol{\mu}}_2-\widehat{\boldsymbol{\mu}}_1)^\top \vct{x} > \frac{\norm{\widehat{\boldsymbol{\mu}}_2}^2-\norm{\widehat{\boldsymbol{\mu}}_1}^2}{2} \ | \ y = 1\right)
        \\
        &=
        \mathbb{P}\left((\widehat{\boldsymbol{\mu}}_2-\widehat{\boldsymbol{\mu}}_1)^\top \vct{x} - \frac{\norm{\widehat{\boldsymbol{\mu}}_2}^2-\norm{\widehat{\boldsymbol{\mu}}_1}^2}{2} > 0 \ | \ y = 1\right)
        \\
        &= 
        \mathbb{P}\left(w > 0 \ | \ y = 1 \right)
    \end{aligned}
    \label{eq: probability of error expansion-general case}
    \end{equation}
    where we defined
    \begin{equation}
        w = \left(\widehat{\boldsymbol{\mu}}_2 - \widehat{\boldsymbol{\mu}}_1\right)^\top\vct{x} - \frac{1}{2} \left(\norm{\widehat{\boldsymbol{\mu}}_2}^2 - \norm{\widehat{\boldsymbol{\mu}}_1}^2\right).
    \label{w definition - general case}
    \end{equation}
    Let us define
    
    \begin{equation}
    y_i = \left(\widehat{\mu}_2\right)_i\cdot x_i - \left(\widehat{\mu}_1\right)_i\cdot x_i - \frac{1}{2}\cdot \left(\widehat{\mu}_2\right)_i^2 + \frac{1}{2} \cdot \left(\widehat{\mu}_1\right)_i^2.
    \label{y_i formula - general case}
    \end{equation}
    Thus,
    \begin{equation}
        w = \sum_{i=1}^d y_i.
    \label{w as sum of y_i's - general case}
    \end{equation}
    In total, from \eqref{eq: probability of error expansion-general case}, \eqref{w as sum of y_i's - general case}, it follows that:
    \begin{equation}
        q(1,2) = \mathbb{P}\left(\sum_{i=1}^d y _i > 0 \ | \ y=1\right).
    \label{eq: probability of error as sum - general case}
    \end{equation}
    The setup of our theoretical investigation clearly implies that the random variables $\{y_i\}_{i=1}^d$, defined in \eqref{y_i formula - general case}, are independent, and thus we can apply Theorem \ref{thm: generalized_barry_essen}. Let us now compute the following expressions, that will be crucial when applying Theorem \ref{thm: generalized_barry_essen}:

    \begin{enumerate}
        \item \begin{equation} \eta_i \coloneqq \mathbb{E}[y_i] \label{E(y_i) def general case} \end{equation}
        \item \begin{equation}  \xi_i^2 \coloneqq \operatorname{Var}(y_i) \label{Var(y_i) def general case} \end{equation}
        \item \begin{equation} \rho_i \coloneqq \mathbb{E}\left[\abs{y_i - \eta_i}^3\right] \label{rho_i def general case} \end{equation}
    \end{enumerate}

    Note that:
    \begin{equation*}
        \widehat{\boldsymbol{\mu}}_j \sim \mathcal{N}\left(\boldsymbol{\mu}_j, \frac{\sigma^2}{N_j}\boldsymbol{I}_d\right).
    \end{equation*}

    Thus, given equations \eqref{eq:gmm_setup}, \eqref{eq:model assumptions}, for each $1 \le i \le d$, we have:
    \begin{equation}
        p_{j,i} \coloneqq (\widehat{\mu}_j)_i \sim \mathcal{N}\left(\left(\mu_j\right)_i, \frac{\sigma^2}{N_j}\right), \ x_i \ | \ y=1 \sim \mathcal{N}\left(-\mu_i, \sigma^2\right)
    \label{eq: basic distributions-general case}
    \end{equation}
    and from \eqref{y_i formula - general case}, it follows that:
    \begin{equation}
        y_i = p_{2,i} x_i - p_{1,i} x_i - \frac{1}{2}p_{2,i}^2 + \frac{1}{2}p_{1,i}^2
    \label{eq:y_i final definition - general case}
    \end{equation}
    where from \eqref{eq:model assumptions}, \eqref{eq: basic distributions-general case}, we have:
    \begin{equation}
        p_{2,i} \sim \mathcal{N}\left(\mu_i, \frac{\sigma^2}{N_2}\right), \ p_{1,i} \sim \mathcal{N}\left(-\mu_i, \frac{\sigma^2}{N_1}\right).
    \label{eq:p final distribution - general case}
    \end{equation}

    For every $1 \le i \le d$, let us define the following random variables:
    \begin{equation}
        a_i \coloneqq \left(p_{2,i} - p_{1,i}\right)\cdot x_i, \ b_i \coloneqq p_{2,i}^2 - p_{1,i}^2.
    \label{a_i, b_i formula}
    \end{equation}
    Thus, from \eqref{eq:y_i final definition - general case}, it follows that:
    \begin{equation}
        y_i = a_i - \frac{1}{2}b_i.
    \label{eq:y_i in terms of a_i,b_i - general case}
    \end{equation}

    \underline{We first compute $\eta_i=\mathbb{E}[y_i]:$}
    From \eqref{eq:y_i in terms of a_i,b_i - general case}, it follows that:
    \begin{equation}
        \eta_i = \mathbb{E}\left[y_i\right] = \mathbb{E}\left[a_i\right] - \frac{1}{2}\mathbb{E}\left[b_i\right].
    \label{eq:eta_i in terms of a_i, b_i - general case}
    \end{equation}
    We now compute each expectation separately. From \eqref{eq: basic distributions-general case}, \eqref{eq:p final distribution - general case}, and the assumption of independence, it follows that:
    \begin{equation}
    \begin{aligned}
        \mathbb{E}[a_i] &= \mathbb{E}\left[p_{2,i}x_i\right] -\mathbb{E}\left[p_{2,i}x_i\right]
        \\
        &= \mathbb{E}\left[p_{2,i}\right]\cdot \mathbb{E}\left[x_i\right] -\mathbb{E}\left[p_{1,i}\right]\cdot \mathbb{E}\left[x_i\right]
        \\
        &= -\mu_i^2 - \mu_i^2
        \\
        &= -2\mu_i^2
    \end{aligned}
    \label{E(a_i) formula - general case}
    \end{equation}
    and
    \begin{equation}
    \begin{aligned}
        \mathbb{E}[b_i] &= \mathbb{E}[p_{2,i}^2] - \mathbb{E}\left[p_{1,i}^2\right]
        \\
        &= \left(\frac{\sigma^2}{N_2} + \mu_i^2\right) - \left(\frac{\sigma^2}{N_1} + \mu_i^2\right)
        \\
        &= \sigma^2\left(\frac{1}{\gamma N} - \frac{1}{N}\right)
        \\
        &= \frac{1 - \gamma}{\gamma}\cdot \frac{\sigma^2}{N}.
    \end{aligned}
    \label{E(b_i) formula - general case}
    \end{equation}
    Thus, from \eqref{eq:eta_i in terms of a_i, b_i - general case}, \eqref{E(a_i) formula - general case}, \eqref{E(b_i) formula - general case}, we have:
    \begin{equation}
        \eta_i = -2\mu_i^2 - \frac{1}{2}\cdot \frac{1 - \gamma}{\gamma}\cdot \frac{\sigma^2}{N} = -\left(2\mu_i^2 + \frac{1 - \gamma}{\gamma}\cdot \frac{\sigma^2}{2N}\right).
    \label{eq:eta_i final formula - general case}
    \end{equation}

    \underline{We now compute $\xi_i^2=\operatorname{Var}(y_i):$} From \eqref{eq:y_i in terms of a_i,b_i - general case}, we have:
    \begin{equation}
        \xi_i^2 = \operatorname{Var}(y_i) = \operatorname{Var}(a_i) + \frac{1}{4}\cdot \operatorname{Var}(b_i) - \operatorname{Cov}\left(a_i, b_i\right).
    \label{eq:Var(y_i) first formula - general case}
    \end{equation}
    We now compute each piece separately, starting from $\operatorname{Var}(a_i)$. From equations \eqref{eq: basic distributions-general case}, \eqref{eq:p final distribution - general case}, \eqref{a_i, b_i formula}, \eqref{E(a_i) formula - general case}, it follows that:
    \begin{equation}
    \begin{aligned}
        \operatorname{Var}(a_i) &= \mathbb{E}[a_i^2] - \mathbb{E}[a_i]^2
        \\
        &= \mathbb{E}\left[\left(p_{2,i} - p_{1,i}\right)^2\right]\cdot \mathbb{E}\left[x_i^2\right] - 4\mu_i^4
        \\
        &= \left(\frac{1 + \gamma}{\gamma}\cdot \frac{\sigma^2}{N} + 4\mu_i^2\right)\cdot \left(\sigma^2 + \mu_i^2\right) - 4\mu_i^4
        \\
        &= \frac{1 + \gamma}{\gamma}\cdot \frac{\sigma^4}{N} + \mu_i^2\cdot \left(\frac{1 + \gamma}{\gamma}\cdot \frac{\sigma^2}{N} + 4\sigma^2\right)
    \end{aligned}
    \label{eq: Var(a_i) formula - general case}
    \end{equation}
    where we used the statistical independence between $p_{2,i}, p_{1,i}$ and \eqref{eq:p final distribution - general case}, to conclude that:
    \begin{equation}
        p_{2,i} - p_{1,i} \sim \mathcal{N}\left(2\mu_i, \frac{\sigma^2}{N_1} + \frac{\sigma^2}{N_2}\right)\Rightarrow p_{2,i} - p_{1,i} \sim \mathcal{N}\left(2\mu_i, \frac{1 + \gamma}{\gamma}\cdot \frac{\sigma^2}{N}\right).
    \label{eq : p_2,i - p_1,i distribution - general case}
    \end{equation}
    We now compute $\operatorname{Var}(b_i)$. From equations \eqref{eq:p final distribution - general case}, \eqref{a_i, b_i formula}, and the statistical independence between $p_{2,i}, p_{1,i}$, it follows that:
    \begin{equation}
    \begin{aligned}
        \operatorname{Var}(b_i) &= \operatorname{Var}\left(p_{2,i}^2 - p_{1,i}^2\right)
        \\
        &= \operatorname{Var}\left(p_{2,i}^2\right) + \operatorname{Var}\left(p_{1,i}^2\right)
        \\
        &= \frac{2\sigma^2}{N_2}\left(\frac{\sigma^2}{N_2}+2\mu_i^2\right) + \frac{2\sigma^2}{N_1}\left(\frac{\sigma^2}{N_1}+2\mu_i^2\right)
        \\
        &= \frac{2\sigma^4}{\gamma^2N^2} + \frac{4\sigma^2\mu_i^2}{\gamma N} + \frac{2\sigma^4}{N^2} + \frac{4\sigma^2\mu_i^2}{N}
        \\
        &= \frac{1 + \gamma^2}{\gamma^2}\cdot \frac{2\sigma^4}{N^2} + \frac{1 + \gamma}{\gamma}\cdot \frac{4\sigma^2\mu_i^2}{N}
    \end{aligned}
    \label{eq: Var(b_i) formula - general case}
    \end{equation}
    where we used the fact that if $x\sim \mathcal{N}\left(\mu_x,\sigma_x^2\right)$, then
    \begin{align*}
        \operatorname{Var}(x^2) &= \mathbb{E}[x^4] - \mathbb{E}[x^2]^2
        \\
        &= \left(3\sigma_x^4 + 6\sigma_x^2 \mu_x^2 + \mu_x^4\right) - \left(\sigma_x^2 + \mu_x^2\right)^2
        \\
        &= \left(3\sigma_x^4 + 6\sigma_x^2 \mu_x^2 + \mu_x^4\right) - \left(\sigma_x^4 + 2\sigma_x^2\mu_x^2 + \mu_x^4\right)
        \\
        &= 2\sigma_x^4+4\sigma_x^2\mu_x^2
        \\
        &= 2\sigma_x^2\cdot \left(\sigma_x^2 + 2\mu_x^2\right).
    \end{align*}
    Finally, we compute $\operatorname{Cov}(a_i,b_i)$. From equations \eqref{eq: basic distributions-general case}, \eqref{eq:p final distribution - general case}, \eqref{a_i, b_i formula}, \eqref{E(a_i) formula - general case}, \eqref{E(b_i) formula - general case}, it follows that:
    \begin{equation}
    \begin{aligned}
        \operatorname{Cov}(a_i,b_i) &= \mathbb{E}\left[a_ib_i\right] - \mathbb{E}[a_i]\cdot \mathbb{E}[b_i]
        \\
        &= \mathbb{E}\left[\left(p_{2,i} - p_{1,i}\right)\left(p_{2,i}^2 - p_{1,i}^2\right)\cdot x_i\right] + \frac{1-\gamma}{\gamma}\cdot \frac{2\sigma^2\mu_i^2}{N}
        \\
        &= \left(\mathbb{E}\left[p_{2,i}^3\right] - \mathbb{E}\left[p_{2,i}\right]\cdot \mathbb{E}\left[p_{1,i}^2\right] - \mathbb{E}\left[p_{1,i}\right]\cdot \mathbb{E}\left[p_{2,i}^2\right] + \mathbb{E}\left[p_{1,i}^3\right]\right)\cdot \mathbb{E}\left[x_i\right] + \frac{1-\gamma}{\gamma}\cdot \frac{2\sigma^2\mu_i^2}{N}
        \\
        &= -\mu_i\cdot \left(\left(\mu_i^3 + 3\mu_i\cdot \frac{\sigma^2}{N_2}\right) - \mu_i\cdot \left(\frac{\sigma^2}{N_1} + \mu_i^2\right) + \mu_i\cdot \left(\frac{\sigma^2}{N_2} + \mu_i^2\right) - \left(\mu_i^3 + 3\mu_i\cdot \frac{\sigma^2}{N_1}\right)\right)
        \\
        &+ \frac{1-\gamma}{\gamma}\cdot \frac{2\sigma^2\mu_i^2}{N}
        \\
        &= -\mu_i\cdot \left(4\mu_i\cdot \frac{\sigma^2}{\gamma N} - 4\mu_i\cdot \frac{\sigma^2}{N}\right) + \frac{1-\gamma}{\gamma}\cdot \frac{2\sigma^2\mu_i^2}{N}
        \\
        &= \frac{4\sigma^2\mu_i^2}{N} - \frac{4\sigma^2\mu_i^2}{\gamma N} + \frac{1-\gamma}{\gamma}\cdot \frac{2\sigma^2\mu_i^2}{N}
        \\
        &= -\frac{1 - \gamma}{\gamma}\cdot \frac{4\sigma^2\mu_i^2}{N} + \frac{1-\gamma}{\gamma}\cdot \frac{2\sigma^2\mu_i^2}{N}
        \\
        &= - \frac{1 - \gamma}{\gamma}\cdot \frac{2\sigma^2\mu_i^2}{N}
    \end{aligned}
    \label{eq: Cov(a_i, b_i) final formula - general case}
    \end{equation}
    where we used the fact that if $x \sim \mathcal{N}\left(\mu_x, \sigma_x^2\right)$, then
    \begin{equation}
        \mathbb{E}[x^3] = \mu_x^3 + 3\mu_x\sigma_x^2.
    \label{eq : e(x^3) formula for gaussian x}
    \end{equation}
    Thus, from equations \eqref{eq:Var(y_i) first formula - general case}, \eqref{eq: Var(a_i) formula - general case}, \eqref{eq: Var(b_i) formula - general case}, \eqref{eq: Cov(a_i, b_i) final formula - general case}, it follows that:
    \begin{equation}
    \begin{aligned}
        \xi_i^2 &= \frac{1 + \gamma}{\gamma}\cdot \frac{\sigma^4}{N} + \mu_i^2\cdot \left(\frac{1 + \gamma}{\gamma}\cdot \frac{\sigma^2}{N} + 4\sigma^2\right)
        \\
        &+ \frac{1}{4}\cdot \left(\frac{1 + \gamma^2}{\gamma^2}\cdot \frac{2\sigma^4}{N^2} + \frac{1 + \gamma}{\gamma}\cdot \frac{4\sigma^2\mu_i^2}{N}\right) + \frac{1 - \gamma}{\gamma}\cdot \frac{2\sigma^2\mu_i^2}{N}
        \\
        &= \frac{1 + \gamma}{\gamma}\cdot \frac{\sigma^4}{N} + \frac{1 + \gamma^2}{\gamma^2}\cdot \frac{\sigma^4}{2N^2} + \mu_i^2\left(\frac{1 + \gamma}{\gamma}\cdot \frac{2\sigma^2}{N} + \frac{1 - \gamma}{\gamma}\cdot \frac{2\sigma^2}{N} + 4\sigma^2\right)
        \\
        &= \sigma^2\left(\frac{1 + \gamma}{\gamma}\cdot \frac{\sigma^2}{N} + \frac{1 + \gamma^2}{2\gamma^2}\cdot  \frac{\sigma^2}{N^2} + 4\mu_i^2\left(1 + \frac{1 + \gamma}{2\gamma}\cdot \frac{1}{N} + \frac{1 - \gamma}{2\gamma}\cdot \frac{1}{N}\right)\right)
        \\
        &= \sigma^2\left(\frac{1 + \gamma}{\gamma}\cdot \frac{\sigma^2}{N} + \frac{1 + \gamma^2}{2\gamma^2}\cdot  \frac{\sigma^2}{N^2} + 4\mu_i^2\left(1 + \frac{1}{\gamma N}\right)\right).
    \end{aligned}
    \label{eq: Var(y_i) final formula - general case}
    \end{equation}
    We get the following lower bound:
    \begin{equation}
        \xi_i^2 \ge D \coloneqq \frac{\sigma^4}{N}\cdot \left(\frac{1 + \gamma}{\gamma} + \frac{1 + \gamma^2}{2\gamma^2}\cdot\frac{1}{N}\right).
    \label{eq:xi_i^2 lower bound - general case}
    \end{equation}

    \underline{Finally, we compute $\rho_i = \mathbb{E}\left[\abs{y_i - \eta_i}^3\right]$:} We will show that $\rho_i$ is globally bounded. We first note the following inequality, which holds for any real-valued random variable $x$ with $\mathbb{E}[x^4] < \infty$:
    \begin{equation*}
        \mathbb{E}\left[\abs{x}^3\right] \le \left(\mathbb{E}\left[x^4\right]\right)^{\frac{3}{4}}.
    \end{equation*}
    This is a consequence of Lyapunov's inequality. Setting $x = y_i - \mu_i$ yields the following upper bound:
    \begin{equation}
        \rho_i = \mathbb{E}\left[\abs{y_i - \eta_i}^3\right] \le \left(\mathbb{E}\left[\left(y_i - \eta_i\right)^4\right]\right)^{\frac{3}{4}}.
    \label{eq: first bound on rho_i - general case}
    \end{equation}
    We now expand:
    \begin{equation*}
        (y_i - \eta_i)^4 = y_i^4 - 4y_i^3\eta_i +6y_i^2\eta_i^2 - 4y_i\eta_i^3 + \eta_i^4
    \end{equation*}
    which, from \eqref{eq: first bound on rho_i - general case}, implies that:
    \begin{equation}
    \begin{aligned}
        \rho_i &\le \left(\mathbb{E}[y_i^4] - 4\eta_i\cdot \mathbb{E}[y_i^3] + 6\eta_i^2\cdot \mathbb{E}[y_i^2] - 4\eta_i^3\cdot \mathbb{E}[y_i] + \eta_i^4 \right)^{\frac{3}{4}}
        \\
        &= \left(\mathbb{E}[y_i^4] - 4\eta_i\cdot \mathbb{E}[y_i^3] + 6\eta_i^2\cdot \left(\operatorname{Var}(y_i) + \mathbb{E}[y_i]^2\right) - 3\eta_i^4 \right)^{\frac{3}{4}}
        \\
        &= \left(\mathbb{E}[y_i^4] - 4\eta_i\cdot \mathbb{E}[y_i^3] + 6\eta_i^2\cdot \left(\xi_i^2 + \eta_i^2\right) - 3\eta_i^4 \right)^{\frac{3}{4}}
        \\
        &= \left(\mathbb{E}[y_i^4] - 4\eta_i\cdot \mathbb{E}[y_i^3] + 6\eta_i^2\xi_i^2 + 3\eta_i^4 \right)^{\frac{3}{4}}
    \end{aligned}
    \label{eq: rho_i second inequality - general case}
    \end{equation}
    where we used the definitions $\xi_i^2 = \operatorname{Var}(y_i), \ \eta_i = \mathbb{E}[y_i]$. It is now left to compute
    \begin{equation}
        \chi_i \coloneqq \mathbb{E}[y_i^4], \ \delta_i \coloneqq \mathbb{E}[y_i^3]
    \label{eq: chi, delta definitions - general case}
    \end{equation}
    which implies, from \eqref{eq: rho_i second inequality - general case}, that
    \begin{equation}
        \rho_i \le \left(\chi_i - 4\delta_i\eta_i + 6\xi_i^2\eta_i^2 + 3\eta_i^4\right)^{\frac{3}{4}}.
    \label{eq: rho_i third inequality - general case}
    \end{equation}
    Let $f \in \{\eta, \xi^2, \delta, \chi\}$. We argue that for all $1 \le i \le d$:
    \begin{equation}
        f_i = \sum_{k=0}^{q(i,f)} c_{k}(i,f)\cdot \mu_i^k
    \label{eq: moments form f_i - general case}
    \end{equation}
    where $q(i,f) \in \mathbb{N}$ and the constants $\{c_k(i,f)\}_{k=0}^{q(i,f)} $ don't depend on $d$. We already saw that $\eta_i, \xi_i$ follows that structure in equations \eqref{eq:eta_i final formula - general case}, \eqref{eq:Var(y_i) first formula - general case}.
    
    We now compute $\delta_i$. From \eqref{eq:y_i in terms of a_i,b_i - general case}, it follows that:
    \begin{equation}
    \begin{aligned}
        \delta_i &= \mathbb{E}\left[\left(a_i - \frac{1}{2}b_i\right)^3\right]
        \\
        &= \mathbb{E}\left[a_i^3\right] - \frac{3}{2}\mathbb{E}\left[a_i^2b_i\right] + \frac{3}{4}\mathbb{E}\left[a_ib_i^2\right] - \frac{1}{8}\mathbb{E}\left[b_i^3\right]
    \end{aligned}
    \label{eq : delta_i first formula - general case}
    \end{equation}
    we now compute each part separately, starting with $\mathbb{E}[a_i^3]$. From equations \eqref{eq: basic distributions-general case}, \eqref{a_i, b_i formula}, and the assumption of independence, it follows that:
    \begin{equation}
    \begin{aligned}
        \mathbb{E}\left[a_i^3\right] &= \mathbb{E}\left[\left(p_{2,i} - p_{1,i}\right)^3\cdot x_i^3\right]
        \\
        &= \mathbb{E}\left[\left(p_{2,i} - p_{1,i}\right)^3\right]\cdot \mathbb{E}\left[x_i^3\right]
        \\
        &= \left(8\mu_i^3 + 6\mu_i\cdot \frac{1 + \gamma}{\gamma}\cdot \frac{\sigma^2}{N}\right)\left(-\mu_i^3 - 3\mu_i\sigma^2\right)
        \\
        &= -\mu_i^2\left(8\mu_i^2 + \frac{6(1+\gamma)}{\gamma}\cdot \frac{\sigma^2}{N}\right)\left(\mu_i^2 + 3\sigma^2\right)
    \end{aligned}
    \end{equation}
    which is a polynomial in $\mu_i$ with real coefficients. The other expressions are computed similarly, and they all have the form as in \eqref{eq: moments form f_i - general case}. We now turn to the assumption that
    \begin{equation}
        \exists_{M \ge 0}  \ \forall_{d \in \mathbb{N}} \ \forall_{1 \le i \le d} \ \abs{\mu_i} \le M
    \end{equation}
    and thus for each $f \in \mathcal{F} \coloneqq \{\eta,\xi^2, \delta, \chi\}$ and for all $1 \le i \le d$, from the triangle inequality, it follows that:
    \begin{equation}
    \begin{aligned}
        \abs{f_i} &\le \sum_{k=0}^{q(i,f)} \abs{c_k(i,f)}\cdot \abs{\mu_i}^k \le \sum_{k=0}^{q(i,f)} \abs{c_k(i,f)}\cdot M^k
        \\
        &\le \sum_{k=0}^{\max_{f \in \mathcal{F}} q(i,f)} \max_{f \in \mathcal{F}} \abs{c_k(i,f)}\cdot M^k
        \\
        &\le \max_{1 \le i \le d} \sum_{k=0}^{\max_{f \in \mathcal{F}} q(i,f)} \max_{f \in \mathcal{F}} \abs{c_k(i,f)}\cdot M^k
    \end{aligned}
    \label{eq: upper bounded on |f_i| - general case}
    \end{equation}
    where we define $c_k(i,f) = 0$ for all $ k > q(i,f)$.
    Let us denote
    \begin{equation}
        L \coloneqq \max_{1 \le i \le d} \sum_{k=0}^{\max_{f \in \mathcal{F}} q(i,f)} \max_{f \in \mathcal{F}} \abs{c_k(i,f)}\cdot M^k.
    \label{eq: L def - general case}
    \end{equation}
    Thus, from \eqref{eq: upper bounded on |f_i| - general case}, we have:
    \begin{equation}
        \forall_{f\in \mathcal{F}} \forall_{1 \le i \le d} \ \abs{f_i} \le L.
    \label{eq: upper bounded globally L on |f_i| - general case}
    \end{equation}
    Now, $L$ is independent of $i$ (because we took the maximum over all possible $1 \le i \le d$) and $d$ (because the degree $q$ and the coefficients $c$ will never depend directly on $d$, because $\sigma$ doesn't depend on $d$). We thus showed that the absolute value of each relevant moment is upper bounded by a global value $L \ge 0$ that is independent of $i$ and $d$. Thus, from \eqref{eq: rho_i third inequality - general case}, \eqref{eq: upper bounded globally L on |f_i| - general case}, it follows that:
    \begin{align*}
        \rho_i &\le \left(\abs{\chi_i - 4\delta_i\eta_i + 6\xi_i^2\eta_i^2 + 3\eta_i^4}\right)^{\frac{3}{4}}
        \\
        &\le \left(\abs{\chi_i} + 4\abs{\delta_i}\abs{\eta_i} + 6\xi_i^2\eta_i^2 + 3\eta_i^4\right)^{\frac{3}{4}}
        \\
        &\le \left(L + 4L^2 + 6L^4 + 3L^4\right)^{\frac{3}{4}}
        \\
        &= \left(9L^4 + 4L^2 + L\right)^{\frac{3}{4}}.
    \end{align*}
    Let us now denote $C = \left(9L^4 + 4L^2 + L\right)^{\frac{3}{4}}$, where $L \ge 0$ is defined in \eqref{eq: L def - general case}. Thus, $C \ge 0$ is independent of both $i$ and $d$, and
    \begin{equation*}
        \forall_{1 \le i \le d} \ \rho_i \le C.
    \end{equation*}
    When combining this result with \eqref{eq:xi_i^2 lower bound - general case}, we get that there exists some $C\ge 0$, $D > 0$ that doesn't depend on $i$ or $d$ such that
    \begin{equation*}
        \rho_i \le C, \ \xi_i^2 \ge D.
    \end{equation*}
    Thus,
    \begin{equation}
    \begin{aligned}
        \frac{\sum_{i=1}^d \rho_i}{\left(\sum_{i=1}^d \xi_i^2\right)^{\frac{3}{2}}} \le \frac{\sum_{i=1}^d C}{\left(\sum_{i=1}^d D\right)^{\frac{3}{2}}} \le \frac{C}{D^{\frac{3}{2}}\sqrt{d}}.
    \end{aligned}
    \label{eq: upper bound in barry essen g(d) - general case}
    \end{equation}
    We have verified that the conditions of Theorem \ref{thm: generalized_barry_essen} are satisfied, and thus, there exists some $C_0 > 0$ independent of $d$ such that for all $x \in \mathbb{R}$
    \begin{equation*}
        \abs{\mathbb{P}\left(\frac{1}{\sqrt{\sum_{i=1}^d \xi_i^2}}\sum_{i=1}^d (y_i-\eta_i) > x \ | \ y = 1\right) - \mathcal{Q}(x)} \le \frac{C_0\sum_{i=1}^d \rho_i}{\left(\sum_{i=1}^d \xi_i^2\right)^{\frac{3}{2}}} \le \frac{A}{\sqrt{d}}
    \end{equation*}
    where we used \eqref{eq: upper bound in barry essen g(d) - general case}, and denoted $A = \frac{C_0C}{D^{\frac{3}{2}}} \ge 0$. Now, $q(1,2)$, which is defined in \eqref{eq: probability of error as sum - general case}, reads:
    \begin{equation}
    \begin{aligned}
        q(1,2) &= \mathbb{P}\left(\sum_{i=1}^d y_i > 0 \ | \ y=1\right)
        \\
        &= \mathbb{P}\left(\sum_{i=1}^d (y_i - \eta_i) > - \sum_{i=1}^d \eta_i \ | \ y = 1 \right)
        \\
        &= \mathbb{P}\left(\frac{1}{\sqrt{\sum_{i=1}^d \xi_i^2}}\sum_{i=1}^d (y_i - \eta_i) > - \frac{\sum_{i=1}^d \eta_i}{\sqrt{\sum_{i=1}^d \xi_i^2}} \ | \ y = 1 \right)
        \\
        &= \mathcal{Q}\left(- \frac{\sum_{i=1}^d \eta_i}{\sqrt{\sum_{i=1}^d \xi_i^2}}\right) + \mathcal{O}\left(\frac{1}{\sqrt{d}}\right)
        \\
        &= \mathcal{Q}\left(\frac{\sum_{i=1}^d \left(2\mu_i^2 + \frac{1 - \gamma}{\gamma}\cdot \frac{\sigma^2}{2N}\right)}{\sqrt{\sum_{i=1}^d \sigma^2\left(\frac{1 + \gamma}{\gamma}\cdot \frac{\sigma^2}{N} + \frac{1 + \gamma^2}{2\gamma^2}\cdot  \frac{\sigma^2}{N^2} + 4\mu_i^2\left(1 + \frac{1}{\gamma N}\right)\right)}}\right) + \mathcal{O}\left(\frac{1}{\sqrt{d}}\right)
        \\
        &= \mathcal{Q}\left(\frac{2\norm{\boldsymbol{\mu}}^2 + \frac{d}{2N}\cdot \frac{1 - \gamma}{\gamma} \cdot \sigma^2}{\sqrt{\sigma^2\cdot \left(\left(\frac{1 + \gamma}{\gamma}\cdot \frac{\sigma^2}{N} + \frac{1 + \gamma^2}{2\gamma^2}\cdot \frac{\sigma^2}{N^2}\right)d + 4\left(1 + \frac{1}{\gamma N}\right)\norm{\boldsymbol{\mu}}^2\right)}}\right) + \mathcal{O}\left(\frac{1}{\sqrt{d}}\right)
        \\
        &= \mathcal{Q}\left(\frac{\norm{\boldsymbol{\mu}} + \frac{d}{4N} \cdot \frac{1 - \gamma}{\gamma}\cdot \frac{\sigma^2}{\norm{\boldsymbol{\mu}}}}{\sigma\cdot \sqrt{\left(\frac{1}{4N}\cdot \frac{1 + \gamma}{\gamma}\cdot \left(\frac{\sigma}{\norm{\boldsymbol{\mu}}}\right)^2 + \frac{1}{8N^2}\cdot \frac{1 + \gamma^2}{\gamma^2}\cdot  \left(\frac{\sigma}{\norm{\boldsymbol{\mu}}}\right)^2\right)\cdot d + \left(1 + \frac{1}{\gamma N}\right)}}\right) + \mathcal{O}\left(\frac{1}{\sqrt{d}}\right)
        \\
        &= \mathcal{Q}\left(\frac{\frac{\norm{\boldsymbol{\mu}}}{\sigma} + \frac{d}{4N} \cdot \frac{1 - \gamma}{\gamma}\cdot \frac{\sigma}{\norm{\boldsymbol{\mu}}}}{\sqrt{\left(\frac{1}{4N}\cdot \frac{1 + \gamma}{\gamma}\cdot \left(\frac{\sigma}{\norm{\boldsymbol{\mu}}}\right)^2 + \frac{1}{8N^2}\cdot \frac{1 + \gamma^2}{\gamma^2}\cdot  \left(\frac{\sigma}{\norm{\boldsymbol{\mu}}}\right)^2\right)\cdot d + \left(1 + \frac{1}{\gamma N}\right)}}\right) + \mathcal{O}\left(\frac{1}{\sqrt{d}}\right).
    \end{aligned}
    \label{first full prob of error - general case}
    \end{equation}
    We now revisit \eqref{eq:quality factor}:
    \begin{equation*}
        \mathcal{S} = \left(\frac{\norm{\boldsymbol{\mu}}}{\sigma}\right)^2.
    \end{equation*}
    Thus, from \eqref{first full prob of error - general case}, $q(1,2)$ reads:
    \begin{equation}
        q(1,2) = \mathcal{Q}\left(\frac{\sqrt{\mathcal{S}} + \frac{1}{4N}\cdot \frac{1 -\gamma}{\gamma}\cdot \frac{d}{\sqrt{\mathcal{S}}}}{\sqrt{\frac{1}{4N}\cdot \frac{1 + \gamma}{\gamma}\cdot \frac{d}{\mathcal{S}}+ \frac{1}{8N^2}\cdot \frac{1 + \gamma^2}{\gamma^2}\cdot \frac{d}{\mathcal{S}}+ \frac{1}{\gamma N} + 1}}\right) + \mathcal{O}\left(\frac{1}{\sqrt{d}}\right).
    \label{eq: final q(1,2) formula - general case}
    \end{equation}  

    We now compute $q(2,1)$. Similarly to the computation of $q(1,2)$, we have:
    \begin{equation}
    \begin{aligned}
       q(2,1) &= \mathbb{P}(\widehat{c}(\vct{x}) = 1 \ | \ y=2)
       \\
       &= \mathbb{P}\left((\widehat{\boldsymbol{\mu}}_2-\widehat{\boldsymbol{\mu}}_1)^\top \vct{x} < \frac{\norm{\widehat{\boldsymbol{\mu}}_2}^2-\norm{\widehat{\boldsymbol{\mu}}_1}^2}{2} \ | \ y = 2\right)
       \\
       &= \mathbb{P}(w<0 \ | \ y=2)
       \\
       &= \mathbb{P}\left(\sum_{i=1}^d y_i < 0 \ | \ y=2\right)
    \end{aligned}
    \label{q(2,1) first formula}
    \end{equation}

    where the random variables $\{y_i\}_{i=1}^d$ are defined in \eqref{eq:y_i final definition - general case}. The new conditional distribution of $x_i$ is:
    \begin{equation}
        \forall_{1 \le i \le d} \ x_i \ | \ y=2 \sim \mathcal{N}(\mu_i, \sigma^2).
    \label{eq:dist of x given y = 2 - general case}
    \end{equation}
    
    We first compute $\eta_i = \mathbb{E}[y_i]$. It is easy to verify from \eqref{eq:eta_i in terms of a_i, b_i - general case}, \eqref{E(a_i) formula - general case}, \eqref{E(b_i) formula - general case}, \eqref{eq:dist of x given y = 2 - general case}, that $\eta_i$ is given by:
    \begin{equation}
        \eta_i = 2\mu_i^2 - \frac{1 - \gamma}{\gamma}\cdot \frac{\sigma^2}{2N}.
    \label{eq:eta_i final formula q(2,1) - general case}
    \end{equation}
    We now compute $\xi_i^2 = \operatorname{Var}\left(y_i\right)$. It still has three components, as in \eqref{eq:Var(y_i) first formula - general case}. It is easy to see from \eqref{eq: Var(b_i) formula - general case} that $\operatorname{Var}(b_i)$ remains unchanged because it doesn't depend on the conditional distribution of $\vct{x}_i$. Thus,
    \begin{equation}
        \operatorname{Var}(b_i) = \frac{1 + \gamma^2}{\gamma^2}\cdot \frac{2\sigma^4}{N^2} + \frac{1 + \gamma}{\gamma}\cdot \frac{4\sigma^2\mu_i^2}{N}
    \label{eq: Var(b_i) formula q(2,1) - general case}.
    \end{equation}
    
    We observe from \eqref{eq: Var(a_i) formula - general case}, \eqref{eq:dist of x given y = 2 - general case} that $\operatorname{Var}(a_i)$ remains unchanged since it depends on $\mathbb{E}[\vct{x}_i^2] = \sigma^2 + \mu_i^2$, which is unaffected.
    \begin{equation}
        \operatorname{Var}(a_i) = \frac{1 + \gamma}{\gamma}\cdot \frac{\sigma^4}{N} + \mu_i^2\cdot \left(\frac{1 + \gamma}{\gamma}\cdot \frac{\sigma^2}{N} + 4\sigma^2\right)
    \label{eq: Var(a_i) formula q(2,1) - general case}.
    \end{equation}
    
    It remains to compute $\operatorname{Cov}(a_i, b_i) = \mathbb{E}[a_i b_i] - \mathbb{E}[a_i]\mathbb{E}[b_i]$. From \eqref{E(a_i) formula - general case} and \eqref{eq:dist of x given y = 2 - general case}, we have $\mathbb{E}[a_i] = 2\mu_i^2$. According to \eqref{E(b_i) formula - general case}, $\mathbb{E}[b_i]$ is unchanged, as it does not depend on the conditional distribution of $\vct{x}_i$. Similarly, from \eqref{eq: Cov(a_i, b_i) final formula - general case}, $\mathbb{E}[a_i b_i]$ picks up a minus sign, so overall, $\operatorname{Cov}(a_i, b_i)$ changes sign. Therefore, \eqref{eq: Cov(a_i, b_i) final formula - general case} implies:
    \begin{equation}
        \operatorname{Cov}(a_i,b_i) = \frac{1 - \gamma}{\gamma}\cdot \frac{2\sigma^2\mu_i^2}{N}.
    \label{eq: Cov(a_i, b_i) final formula q(2,1) - general case}
    \end{equation}

    Thus, from equations \eqref{eq:Var(y_i) first formula - general case}, \eqref{eq: Var(b_i) formula q(2,1) - general case}, \eqref{eq: Var(a_i) formula q(2,1) - general case}, \eqref{eq: Cov(a_i, b_i) final formula q(2,1) - general case}, it follows that:
    \begin{equation}
    \begin{aligned}
        \xi_i^2 &= \frac{1 + \gamma}{\gamma}\cdot \frac{\sigma^4}{N} + \mu_i^2\cdot \left(\frac{1 + \gamma}{\gamma}\cdot \frac{\sigma^2}{N} + 4\sigma^2\right)
        \\
        &+ \frac{1}{4}\cdot \left(\frac{1 + \gamma^2}{\gamma^2}\cdot \frac{2\sigma^4}{N^2} + \frac{1 + \gamma}{\gamma}\cdot \frac{4\sigma^2\mu_i^2}{N}\right) - \frac{1 - \gamma}{\gamma}\cdot \frac{2\sigma^2\mu_i^2}{N}
        \\
        &= \frac{1 + \gamma}{\gamma}\cdot \frac{\sigma^4}{N} + \frac{1 + \gamma^2}{\gamma^2}\cdot \frac{\sigma^4}{2N^2} + \mu_i^2\left(\frac{1 + \gamma}{\gamma}\cdot \frac{2\sigma^2}{N} - \frac{1 - \gamma}{\gamma}\cdot \frac{2\sigma^2}{N} + 4\sigma^2\right)
        \\
        &= \sigma^2\left(\frac{1 + \gamma}{\gamma}\cdot \frac{\sigma^2}{N} + \frac{1 + \gamma^2}{2\gamma^2}\cdot  \frac{\sigma^2}{N^2} + 4\mu_i^2\left(1 + \frac{1 + \gamma}{2\gamma}\cdot \frac{1}{N} - \frac{1 - \gamma}{2\gamma}\cdot \frac{1}{N}\right)\right)
        \\
        &= \sigma^2\left(\frac{1 + \gamma}{\gamma}\cdot \frac{\sigma^2}{N} + \frac{1 + \gamma^2}{2\gamma^2}\cdot  \frac{\sigma^2}{N^2} + 4\mu_i^2\left(1 + \frac{1}{N}\right)\right)
        \\
        &\ge \frac{\sigma^4}{N}\cdot \left(\frac{1 + \gamma}{\gamma} + \frac{1 + \gamma^2}{2\gamma^2}\cdot\frac{1}{N}\right) = D.
    \end{aligned}
    \label{eq: Var(y_i) final formula q(2,1) - general case}
    \end{equation}

    Thus, $\xi_i^2 \ge D$ where $D > 0$ is the same constant defined in \eqref{eq:xi_i^2 lower bound - general case}. A similar argument for the case $y = 1$ shows that $\rho_i = \mathbb{E}[\lvert y_i - \eta_i \rvert^3] \le C$, where $C \ge 0$ and $D > 0$ are constants independent of $i$ and $d$. Since the variables $\{y_i\}_{i=1}^d$ are independent, we may apply Theorem \ref{thm: generalized_barry_essen}, which guarantees the existence of a constant $C_0 > 0$ independent of $d$ such that for all $x \in \mathbb{R}$:

    \begin{equation*}
        \abs{\mathbb{P}\left(\frac{1}{\sqrt{\sum_{i=1}^d \xi_i^2}}\sum_{i=1}^d (y_i-\eta_i) > x \ | \ y=2 \right) - \mathcal{Q}(x)} \le \frac{C_0\sum_{i=1}^d \rho_i}{\left(\sum_{i=1}^d \xi_i^2\right)^{\frac{3}{2}}} \le \frac{A}{\sqrt{d}}.
    \end{equation*}
    Where we denoted $A = \frac{C_0C}{D^{\frac{3}{2}}} \ge 0$. Now, $q(2,1)$, which is defined in \eqref{q(2,1) first formula}, reads:
    \begin{equation}
    \begin{aligned}
        q(2,1) &= \mathbb{P}\left(\sum_{i=1}^d y_i < 0 \ | \ y=2\right)
        \\
        &= \mathbb{P}\left(\sum_{i=1}^d (y_i - \eta_i) < - \sum_{i=1}^d \eta_i \ | \ y=2\right)
        \\
        &= \mathbb{P}\left(\frac{1}{\sqrt{\sum_{i=1}^d \xi_i^2}}\sum_{i=1}^d (y_i - \eta_i) < - \frac{\sum_{i=1}^d \eta_i}{\sqrt{\sum_{i=1}^d \xi_i^2}} \ | \ y=2\right)
        \\
        &= 1 - \mathbb{P}\left(\frac{1}{\sqrt{\sum_{i=1}^d \xi_i^2}}\sum_{i=1}^d (y_i - \eta_i) \ge - \frac{\sum_{i=1}^d \eta_i}{\sqrt{\sum_{i=1}^d \xi_i^2}} \ | \ y=2 \right)
        \\
        &= 1 - \left(\mathcal{Q}\left(- \frac{\sum_{i=1}^d \eta_i}{\sqrt{\sum_{i=1}^d \xi_i^2}}\right) + \mathcal{O}\left(\frac{1}{\sqrt{d}}\right)\right)
        \\
        &= \mathcal{Q}\left( \frac{\sum_{i=1}^d \eta_i}{\sqrt{\sum_{i=1}^d \xi_i^2}}\right) + \mathcal{O}\left(\frac{1}{\sqrt{d}}\right)
        \\
        &= \mathcal{Q}\left(\frac{\sum_{i=1}^d \left(2\mu_i^2 + \frac{1 - \gamma}{\gamma}\cdot \frac{\sigma^2}{2N}\right)}{\sqrt{\sum_{i=1}^d \sigma^2\left(\frac{1 + \gamma}{\gamma}\cdot \frac{\sigma^2}{N} + \frac{1 + \gamma^2}{2\gamma^2}\cdot  \frac{\sigma^2}{N^2} + 4\mu_i^2\left(1 + \frac{1}{N}\right)\right)}}\right) + \mathcal{O}\left(\frac{1}{\sqrt{d}}\right)
        \\
        &= \mathcal{Q}\left(\frac{2\norm{\boldsymbol{\mu}}^2 + \frac{d}{2N}\cdot \frac{1 - \gamma}{\gamma} \cdot \sigma^2}{\sqrt{\sigma^2\cdot \left(\left(\frac{1 + \gamma}{\gamma}\cdot \frac{\sigma^2}{N} + \frac{1 + \gamma^2}{2\gamma^2}\cdot \frac{\sigma^2}{N^2}\right)d + 4\left(1 + \frac{1}{N}\right)\norm{\boldsymbol{\mu}}^2\right)}}\right) + \mathcal{O}\left(\frac{1}{\sqrt{d}}\right)
        \\
        &= \mathcal{Q}\left(\frac{\norm{\boldsymbol{\mu}} + \frac{d}{4N} \cdot \frac{1 - \gamma}{\gamma}\cdot \frac{\sigma^2}{\norm{\boldsymbol{\mu}}}}{\sigma\cdot \sqrt{\left(\frac{1}{4N}\cdot \frac{1 + \gamma}{\gamma}\cdot \left(\frac{\sigma}{\norm{\boldsymbol{\mu}}}\right)^2 + \frac{1}{8N^2}\cdot \frac{1 + \gamma^2}{\gamma^2}\cdot  \left(\frac{\sigma}{\norm{\boldsymbol{\mu}}}\right)^2\right)\cdot d + \left(1 + \frac{1}{N}\right)}}\right) + \mathcal{O}\left(\frac{1}{\sqrt{d}}\right)
        \\
        &= \mathcal{Q}\left(\frac{\frac{\norm{\boldsymbol{\mu}}}{\sigma} + \frac{d}{4N} \cdot \frac{1 - \gamma}{\gamma}\cdot \frac{\sigma}{\norm{\boldsymbol{\mu}}}}{\sqrt{\left(\frac{1}{4N}\cdot \frac{1 + \gamma}{\gamma}\cdot \left(\frac{\sigma}{\norm{\boldsymbol{\mu}}}\right)^2 + \frac{1}{8N^2}\cdot \frac{1 + \gamma^2}{\gamma^2}\cdot  \left(\frac{\sigma}{\norm{\boldsymbol{\mu}}}\right)^2\right)\cdot d + \left(1 + \frac{1}{N}\right)}}\right) + \mathcal{O}\left(\frac{1}{\sqrt{d}}\right)
    \end{aligned}
    \label{q(2,1) first formula - general case}
    \end{equation}
    where we used the identity $\mathcal{Q}(-x) = 1 - \mathcal{Q}(x)$. We now revisit \eqref{eq:quality factor}:
    \begin{equation*}
        \mathcal{S} = \left(\frac{\norm{\boldsymbol{\mu}}}{\sigma}\right)^2.
    \end{equation*}
    Thus, from \eqref{q(2,1) first formula - general case}, $q(2,1)$ reads:
    \begin{equation}
        q(2,1) = \mathcal{Q}\left(\frac{\sqrt{\mathcal{S}} + \frac{1}{4N}\cdot \frac{1 -\gamma}{\gamma}\cdot \frac{d}{\sqrt{\mathcal{S}}}}{\sqrt{\frac{1}{4N}\cdot \frac{1 + \gamma}{\gamma}\cdot \frac{d}{\mathcal{S}}+ \frac{1}{8N^2}\cdot \frac{1 + \gamma^2}{\gamma^2}\cdot \frac{d}{\mathcal{S}}+ \frac{1}{N} + 1}}\right) + \mathcal{O}\left(\frac{1}{\sqrt{d}}\right).
    \label{eq: final q(2,1) formula - general case}
    \end{equation}  
    
    To finish the proof, from 
    \eqref{full probability of error expansion}, \eqref{eq: final q(1,2) formula - general case}, \eqref{eq: final q(2,1) formula - general case}, the probability of error reads:
    \begin{align*}
        p_{\vct{x}}(\mathrm{error}) &= \frac{1}{2}\cdot q(1,2) + \frac{1}{2}\cdot q(2,1) 
        \\
        &= \frac{1}{2}\cdot \left(\mathcal{Q}\left(\frac{\sqrt{\mathcal{S}} + \frac{1}{4N}\cdot \frac{1 -\gamma}{\gamma}\cdot \frac{d}{\sqrt{\mathcal{S}}}}{\sqrt{\frac{1}{4N}\cdot \frac{1 + \gamma}{\gamma}\cdot \frac{d}{\mathcal{S}}+ \frac{1}{8N^2}\cdot \frac{1 + \gamma^2}{\gamma^2}\cdot \frac{d}{\mathcal{S}}+ \frac{1}{\gamma N} + 1}}\right) + \mathcal{O}\left(\frac{1}{\sqrt{d}}\right)\right)
        \\
        & + \frac{1}{2}\cdot \left(\mathcal{Q}\left(\frac{\sqrt{\mathcal{S}} - \frac{1}{4N}\cdot \frac{1-\gamma}{\gamma}\cdot \frac{d}{\sqrt{\mathcal{S}}}}{\sqrt{\frac{1}{4N}\cdot \frac{1+\gamma}{\gamma}\cdot \frac{d}{\mathcal{S}} + \frac{1}{8N^2} \cdot \frac{1+\gamma^2}{\gamma^2} \cdot \frac{d}{\mathcal{S}} + \frac{1}{N} + 1}}\right) + \mathcal{O}\left(\frac{1}{\sqrt{d}}\right)\right)
        \\
        &= \frac{1}{2}\cdot \mathcal{Q}\left(\frac{\sqrt{\mathcal{S}} + \frac{1}{4N}\cdot \frac{1 -\gamma}{\gamma}\cdot \frac{d}{\sqrt{\mathcal{S}}}}{\sqrt{\frac{1}{4N}\cdot \frac{1 + \gamma}{\gamma}\cdot \frac{d}{\mathcal{S}}+ \frac{1}{8N^2}\cdot \frac{1 + \gamma^2}{\gamma^2}\cdot \frac{d}{\mathcal{S}}+ \frac{1}{\gamma N} + 1}}\right) 
        \\
        &+ \frac{1}{2}\cdot \mathcal{Q}\left(\frac{\sqrt{\mathcal{S}} - \frac{1}{4N}\cdot \frac{1-\gamma}{\gamma}\cdot \frac{d}{\sqrt{\mathcal{S}}}}{\sqrt{\frac{1}{4N}\cdot \frac{1+\gamma}{\gamma}\cdot \frac{d}{\mathcal{S}} + \frac{1}{8N^2} \cdot \frac{1+\gamma^2}{\gamma^2} \cdot \frac{d}{\mathcal{S}} + \frac{1}{N} + 1}}\right) + \mathcal{O}\left(\frac{1}{\sqrt{d}}\right)
        \\
        &= \hat{p}(\mathcal{S},N,\gamma,d)  + \mathcal{O}\left(\frac{1}{\sqrt{d}}\right)
    \end{align*}
    where $\hat{p}$ is given in \eqref{approx prob x}.
\end{proof}

\subsection{Proof of Theorem \ref{thm:thm3}}
\label{app:data_proc_thm}

\begin{proof}
    We provide an algorithm to construct $\boldsymbol{A} \in \mathbb{R}^{k\times d}$ given $\frac{\boldsymbol{\mu}}{\norm{\boldsymbol{\mu}}}$ and prove that it satisfies \eqref{eq:A_properties}. Later, we will show how to estimate it from unlabeled data.
\begin{enumerate}
    \item Define $\boldsymbol{u} \coloneqq \boldsymbol{a}_1 = \frac{\boldsymbol{\mu}}{\norm{\boldsymbol{\mu}}}$. If $k = 1$, define $\boldsymbol{A} = \boldsymbol{u}^\top$. Else, continue.
    
    \item Find $\boldsymbol{a}_2, \dots, \boldsymbol{a}_k \in \mathbb{R}^d$ such that $\langle \boldsymbol{a}_i, \boldsymbol{u}\rangle = 0$ and $\langle \boldsymbol{a}_i, \boldsymbol{a}_j\rangle = \delta_{ij}$.
        
    \item Define the matrix $\boldsymbol{A} \in \mathbb{R}^{k \times d}$
    where the $i$-th row is given by $\boldsymbol{a}_i^\top$.
\end{enumerate}

The proof that the algorithm works is given below.
\begin{itemize}
    \item \underline{Step 1:} If $k=1$, we define $\boldsymbol{A} = \frac{\boldsymbol{\mu}^\top}{\norm{\boldsymbol{\mu}}}$. It is easy to ensure that it satisfies \eqref{eq:A_properties}.
    \item \underline{Step 2:} If $\boldsymbol{\mu} = \boldsymbol{0}$, then the result is trivial, because we can construct on orthonormal set
    \begin{equation*}
        \{\boldsymbol{a}_2, \dots,\boldsymbol{a_k}\}\subset \mathbb{R}^d.
    \end{equation*}
    Otherwise, $\boldsymbol{\mu}\ne \boldsymbol{0}$ and let us define the following subset of $\mathbb{R}^d$:
    \begin{align*}
        V \coloneqq \{\vct{x} \in \mathbb{R}^d : \langle \vct{x}, \boldsymbol{\mu}\rangle = 0 \} \subset \mathbb{R}^d.
    \end{align*}
    We see that $V = \left(\operatorname{span}\{\boldsymbol{\mu}\}\right)^\perp$ is a linear subspace of $\mathbb{R}^d$ of dimension $d-1$. 
    Thus, there exists a basis $\{\boldsymbol{v}_1,\dots, \boldsymbol{v}_{d-1}\}\subseteq V$. We know that $k - 1 \le d  - 1$ and thus $\{\boldsymbol{v}_1 , \dots, \boldsymbol{v}_{k-1}\} \subseteq V$ is a linearly independent set. That is, we can apply the Gram-Schmidt procedure on this set, to get an orthonormal set $\{\vct{a}_2, \dots, \vct{a}_{k}\} \subseteq V$. This is a subset of $V$ because Gram–Schmidt outputs vectors that are linear combinations of the input, which lie in $V$.

    \item \underline{Step 3:} The rows of $\boldsymbol{A}$ are orthonormal, so $\boldsymbol{A} \boldsymbol{A}^\top = \boldsymbol{I}_k$. From step 2, it follows easily that
    \begin{align*}
            \boldsymbol{A}\boldsymbol{\mu} = \begin{bmatrix} \norm{\boldsymbol{\mu}} \\ 0 \\ \vdots \\ 0 \end{bmatrix} \Rightarrow \norm{\boldsymbol{A}\boldsymbol{\mu}} = \norm{\boldsymbol{\mu}}
    \end{align*}
\end{itemize}
Thus, $\boldsymbol{A}$ meets the needed requirements, and thus we have proved the existence of such a matrix $\boldsymbol{A}$. Assuming $\boldsymbol{\mu} \ne \boldsymbol{0}$, it is now left to prove that one can learn $\boldsymbol{A}$ from infinite unlabeled data $\{\vct{x}_i\}_{i=1}^{\infty}$. This data is taken from the distribution 
\begin{equation}
    \vct{x} \sim \frac{1}{2}\mathcal{N}\left(-\boldsymbol{\mu}, \sigma^2\boldsymbol{I}_d\right) + \frac{1}{2}\mathcal{N}\left(\boldsymbol{\mu}, \sigma^2\boldsymbol{I}_d\right)
\label{proof thm 3 - dist x}
\end{equation}
where the label is called $y \in \{1,2\}$. Let us assume that there is $m$ unlabeled data. We first compute
\begin{equation*}
    \boldsymbol{\Sigma}_m = \frac{1}{m}\sum_{i=1}^m \vct{x}_i\vct{x}_i^\top.
\end{equation*}
As $m \rightarrow \infty$, we have $\boldsymbol{\Sigma}_m \xrightarrow{\text{a.s.}} \boldsymbol{\Sigma}$, where
\begin{align*}
    \boldsymbol{\Sigma} = \mathbb{E}\left[\vct{x}\vct{x}^\top\right] &= \mathbb{E}\left[\vct{x}\vct{x}^\top \ | \ y = 1\right]\cdot \mathbb{P}\left(y = 1\right) + \mathbb{E}\left[\vct{x}\vct{x}^\top \ | \ y = 2\right]\cdot \mathbb{P}\left(y = 2\right)
    \\
    &= \frac{1}{2}\left(\sigma^2\boldsymbol{I}_d + \boldsymbol{\mu}\boldsymbol{\mu}^\top\right) + \frac{1}{2}\left(\sigma^2\boldsymbol{I}_d + \boldsymbol{\mu}\boldsymbol{\mu}^\top\right)
    \\
    &= \sigma^2\boldsymbol{I}_d + \boldsymbol{\mu}\boldsymbol{\mu}^\top
\end{align*}
where we used \eqref{proof thm 3 - dist x}. That is, we can learn the matrix
\begin{equation}
    \boldsymbol{\Sigma} = \sigma^2\boldsymbol{I}_d + \boldsymbol{\mu}\boldsymbol{\mu}^\top.
\label{proof thm 3 - Sigma matrix}
\end{equation}
We now argue that the maximal eigenvalue of $\boldsymbol{\Sigma}$ is $\lambda_{\max} = \sigma^2 + \norm{\boldsymbol{\mu}}^2$, with eigen-space $V_{\lambda_{\max}} = \text{span}\{\boldsymbol{\mu}\}$. Indeed, from \eqref{proof thm 3 - Sigma matrix}, it follows that:
\begin{equation*}
    \boldsymbol{\Sigma}\boldsymbol{\mu} = \left(\sigma^2 + \norm{\boldsymbol{\mu}}^2\right)\boldsymbol{\mu}
\end{equation*}
and for all $\boldsymbol{v} \ \bot \  \boldsymbol{\mu}$ we have
\begin{equation*}
    \boldsymbol{\Sigma}\boldsymbol{v} = \sigma^2\boldsymbol{v}.
\end{equation*}
Thus, the eigenvalues of $\boldsymbol{\Sigma}$ are
\begin{equation*}
    \sigma^2 = \lambda_{\text{min}} < \lambda_{\text{max}} = \sigma^2 + \norm{\boldsymbol{\mu}}^2.
\end{equation*}
The eigen-space of $\lambda_{\text{min}}$ satisfies:
\begin{equation*}
    V_{\lambda_{\text{min}}} = \left(\text{span}\{\boldsymbol{\mu}\}\right)^\bot \Rightarrow \text{dim}\left(V_{\lambda_{\text{min}}}\right) = d-1
\end{equation*}
Thus, $\text{dim}\left(V_{\lambda_{\text{max}}}\right) = 1$, which implies that
\begin{equation}
    V_{\lambda_{\max}} = \text{span}\{\boldsymbol{\mu}\}.
\label{eigan space of lambda_max - proof 3}
\end{equation}
We now apply the power iteration method on the matrix $\boldsymbol{\Sigma}_m$. 
For large enough number of iterations and sufficiently large
$m \gg 1$, it returns a vector that is \textit{arbitrarily close} to the eigenvector of $\boldsymbol{\Sigma}$ that corresponds to the maximal eigenvalue $\lambda_{\text{max}}$ (ensured by the spectral gap of $\|\boldsymbol{\mu}\|^2>0$ between the two largest eigenvalues of $\boldsymbol{\Sigma}$), which from \eqref{eigan space of lambda_max - proof 3}, is characterized as $\alpha\boldsymbol{\mu}$ where $\alpha \ne 0$ is a constant. Normalizing this vector leads to 
    $\pm \frac{\boldsymbol{\mu}}{\norm{\boldsymbol{\mu}}}$.
Now, we apply the algorithm we presented above to compute $\boldsymbol{A}$. As a side note, using the vector $\boldsymbol{a}_1 = - \frac{\boldsymbol{\mu}^\top}{\norm{\boldsymbol{\mu}}}$ as the first row of $\boldsymbol{A}$ has no effect on the resulting properties of $\boldsymbol{A}$.
\end{proof}

\subsection{Proof of Theorem \ref{thm:thm4}}
\begin{proof}
    We know that
    \begin{equation*}
        \vct{z} = \boldsymbol{A} \vct{x},
    \end{equation*}
    where $\boldsymbol{A} \in \mathbb{R}^{k \times d}$ is a deterministic matrix satisfying:
    \begin{itemize}
        \item $\boldsymbol{A} \boldsymbol{A}^\top = \boldsymbol{I}_k$.
        \item $\norm{\boldsymbol{A}\boldsymbol{\mu}} = \norm{\boldsymbol{\mu}}$.
    \end{itemize}
It is a standard result that a linear transformation of a Gaussian vector is also a Gaussian vector, thus:
\begin{equation*}
    \forall_{j \in \{1,2\}} \ \vct{z} \ | \ y=j \sim \mathcal{N}\left(\boldsymbol{A} \boldsymbol{\mu}_j, \boldsymbol{A} \sigma^2\boldsymbol{I}_d \boldsymbol{A}^\top\right)
\end{equation*}
that is, for all $j \in \{1,2\}$ we have:
\begin{equation*}
    \vct{z} \ | \ y = j \sim \mathcal{N}\left(\boldsymbol{\eta}_j, \sigma^2\boldsymbol{I}_k\right)
\end{equation*}
where 
\begin{equation*}
    \boldsymbol{\eta}_j = \boldsymbol{A} \boldsymbol{\mu}_j.
\end{equation*}
We know that $\boldsymbol{\mu}_2 = -\boldsymbol{\mu}_1 = \boldsymbol{\mu}$, and thus $\boldsymbol{\eta}_2 = -\boldsymbol{\eta}_1 = \boldsymbol{\eta} = \boldsymbol{A}\boldsymbol{\mu}$. That is, our model assumptions still hold, with the following modifications:
\begin{itemize}
    \item $d \mapsto k$.
    \item $\boldsymbol{\mu} \mapsto \boldsymbol{\eta} = \boldsymbol{A}\boldsymbol{\mu}$.
\end{itemize}
The new separation quality factor $\mathcal{S}_{\vct{z}}$ of the new GMM (computed similarly to \eqref{eq:quality factor}) is given by:
\begin{equation*}
    \mathcal{S}_{\vct{z}} = \left(\frac{\norm{\boldsymbol{\eta}_2 - \boldsymbol{\eta}_1}}{2\sigma}\right)^2=\left(\frac{\norm{\boldsymbol{\eta}}}{\sigma}\right)^2 = \left(\frac{\norm{\boldsymbol{A}\boldsymbol{\mu}}}{\sigma}\right)^2 = \left(\frac{\norm{\boldsymbol{\mu}}}{\sigma}\right)^2 = \mathcal{S}.
\end{equation*}
That is, the separation quality factor remains the same after the processing. The result is now immediate from Theorem \ref{thm:thm2} and changing $d \mapsto k$. % 
\end{proof}

\subsection{Proof of Theorem \ref{thm:thm5}}

\begin{proof}
Let us fix $\gamma = 1$ and take some
\begin{equation*}
    \mathcal{S} > 0, \ 1 \le k < d, \ N \in \mathbb{N}.
\end{equation*}
From Theorems \ref{thm:thm2} and \ref{thm:thm4}, it follows that we need to show the following:
\begin{equation}
    \hat{p}(\mathcal{S}, N, 1, k) < \hat{p}(\mathcal{S},N,1,d)
\label{need to show theorem 4}
\end{equation}
where $\hat{p}$ is given in \eqref{approx prob x}. It is easy to prove that:
\begin{equation*}
    \forall_{q \in \mathbb{N}} \ \hat{p}\left(\mathcal{S}, N, 1, q\right) = \mathcal{Q}\left(\frac{\sqrt{\mathcal{S}}}{\sqrt{\left(\frac{q}{2\mathcal{S}} + 1\right)\cdot \frac{1}{N} + \frac{q}{4\mathcal{S}}\cdot \frac{1}{N^2} + 1}}\right).
\end{equation*}
Following \eqref{need to show theorem 4}, we need to show that:
\begin{equation*}
    \mathcal{Q}\left(\frac{\sqrt{\mathcal{S}}}{\sqrt{\left(\frac{k}{2\mathcal{S}} + 1\right)\cdot \frac{1}{N} + \frac{k}{4\mathcal{S}}\cdot \frac{1}{N^2} + 1}}\right) < \mathcal{Q}\left(\frac{\sqrt{\mathcal{S}}}{\sqrt{\left(\frac{d}{2\mathcal{S}} + 1\right)\cdot \frac{1}{N} + \frac{d}{4\mathcal{S}}\cdot \frac{1}{N^2} + 1}}\right)
\end{equation*}
which is immediate because the argument in the $\mathcal{Q}$ is strictly higher in the LHS, and the $\mathcal{Q}$ function is strictly decreasing.
\end{proof}

\subsection{Proof of Theorem \ref{thm:thm6}}

\begin{proof}
Let us take some
\begin{equation*}
    0 < \gamma < 1, \ 0 < \mathcal{S} \le 1, 1 \le k < d, N \ge \frac{\gamma^2 - 4\gamma + 1}{2\gamma\cdot (1+\gamma)}
\end{equation*}
we need to show that 
\begin{equation*}
    \hat{p}(\mathcal{S},N,\gamma,k) < \hat{p}(\mathcal{S},N,\gamma,d).
\end{equation*}
That is, it is sufficient to show that the function
\begin{equation*}
    f(x) = 2\hat{p}(\mathcal{S},N,\gamma,x)
\end{equation*}
is strictly increasing for all $x \ge 1$, where $\hat{p}$ is defined in \eqref{approx prob x}. It is easy to verify that:
\begin{equation}
    f(x) = \mathcal{Q}\left(\frac{\sqrt{\mathcal{S}}+ \frac{1 - \gamma}{4\gamma N \sqrt{\mathcal{S}}}\cdot x}{\sqrt{\left(\frac{1 + \gamma}{4\gamma N \mathcal{S}} + \frac{1 + \gamma^2}{8\gamma^2 N^2 \mathcal{S}}\right)\cdot x + \frac{1}{\gamma N} + 1}}\right)
    +
    \mathcal{Q}\left(\frac{\sqrt{\mathcal{S}}- \frac{1 - \gamma}{4\gamma N \sqrt{\mathcal{S}}}\cdot x}{\sqrt{\left(\frac{1 + \gamma}{4\gamma N \mathcal{S}} + \frac{1 + \gamma^2}{8\gamma^2 N^2 \mathcal{S}}\right)\cdot x + \frac{1}{N} + 1}}\right).
\label{f(x) def}
\end{equation}
Let us define the following functions:
\begin{equation}
    g_1(x) = \frac{\sqrt{\mathcal{S}} + \frac{1 - \gamma}{4\gamma N \sqrt{\mathcal{S}}}\cdot x}{\sqrt{\left(\frac{1 + \gamma}{4\gamma N \mathcal{S}} + \frac{1 + \gamma^2}{8\gamma^2 N^2 \mathcal{S}}\right)\cdot x + \frac{1}{\gamma N} + 1}}
\label{g_1(x) def}
\end{equation}
and
\begin{equation}
    g_2(x) = \frac{\sqrt{\mathcal{S}} - \frac{1 - \gamma}{4\gamma N \sqrt{\mathcal{S}}}\cdot x}{\sqrt{\left(\frac{1 + \gamma}{4\gamma N \mathcal{S}} + \frac{1 + \gamma^2}{8\gamma^2 N^2 \mathcal{S}}\right)\cdot x + \frac{1}{N} + 1}}.
\label{g_2(x) def}
\end{equation}

Thus, \eqref{f(x) def} reads:
\begin{equation}
    f(x) = \mathcal{Q}\left(g_1(x)\right) + \mathcal{Q}\left(g_2(x)\right).
\label{simplified f(x)}
\end{equation}
From the chain rule, the derivative reads:
\begin{equation}
\begin{aligned}
    f'(x) &=\mathcal{Q}'\left(g_1(x)\right)\cdot g_1'(x) + \mathcal{Q}'\left(g_2(x)\right)\cdot g_2'(x)
    \\
    &= -\frac{1}{\sqrt{2\pi}}\cdot \left(\exp\left(-\frac{1}{2}\cdot g_1^2(x)\right)\cdot g_1'(x) + \exp\left(-\frac{1}{2}\cdot g_2^2(x)\right)\cdot g_2'(x)\right)
    \\
    &= -\frac{1}{\sqrt{2\pi}}\cdot \left(w_1(x)\cdot g_1'(x) + w_2(x)\cdot g_2'(x)\right).
\end{aligned}
\label{f'(x) formula}
\end{equation}
We used the following property of the $\mathcal{Q}$ function:
\begin{equation*}
    \frac{d}{dx} \mathcal{Q}(x) = -\frac{1}{\sqrt{2\pi}}\cdot \exp\left(-\frac{x^2}{2}\right)
\end{equation*}
and the following notation:
\begin{equation}
    w_i(x) = \exp\left(-\frac{1}{2}\cdot g_i^2(x)\right).
\label{w_i def}
\end{equation}
Thus, showing that $f$ is strictly increasing for all $x \ge 1$ is equivalent to proving that for all $x \ge 1$
\begin{equation}
    f'(x) > 0 \Leftrightarrow w_1(x)\cdot g_1'(x) + w_2(x)\cdot g_2'(x) < 0.
\label{w_1*g_1'+w_2*g_2'<0}
\end{equation}
We argue now that for all $x \ge 1$:
\begin{enumerate}
    \item \begin{equation} w_1(x) < w_2(x) \label{w_1<w_2} \end{equation}
    \item \begin{equation} g_2'(x) < 0 \label{g_2'<0} \end{equation}
    \item \begin{equation} g_1'(x) + g_2'(x) \le 0 \label{g_1'+g_2'<=0} \end{equation}
\end{enumerate}
\underline{Let us first prove \eqref{w_1<w_2}:} From \eqref{w_i def} it follows that it is sufficient to prove:
\begin{equation}
    \forall_{x \ge 1 } \ \abs{g_1(x)} > \abs{g_2(x)}.
\label{|g_1(x)| > |g_2(x)|}
\end{equation}
Let us take some $x \ge 1$. It is easy to see from \eqref{g_1(x) def} that $g_1(x) \ge 0$ and thus $\abs{g_1(x)} = g_1(x)$. That is, it is sufficient to prove that
\begin{equation}
    g_1(x) > g_2(x)
\label{g_1(x) > g_2(x)}
\end{equation}
and
\begin{equation}
    g_2(x) > -g_1(x) \Leftrightarrow g_1(x) + g_2(x) > 0.
\label{g_1(x) + g_2(x) > 0}
\end{equation}
Let us now define the following parameters:
\begin{equation}
\begin{cases}
    B = \frac{1-\gamma}{4\gamma N \sqrt{\mathcal{S}}}
    \\
    C = \frac{1+\gamma}{4\gamma N \mathcal{S}} + \frac{1+\gamma^2}{8\gamma^2 N^2 \mathcal{S}}
    \\
    c_1 = \frac{1}{\gamma N} + 1
    \\
    c_2 = \frac{1}{N} + 1 < c_1
\end{cases}
\label{B, C, c_1, c_2 def}
\end{equation}
We also define the following functions:
\begin{equation}
\begin{cases}
    D_1(x) = \sqrt{Cx + c_1}
    \\
    D_2(x) = \sqrt{Cx + c_2} < D_1(x)
\end{cases}
\label{D_1(x), D_2(x) def}
\end{equation}
From \eqref{g_1(x) def}, \eqref{g_2(x) def}, \eqref{B, C, c_1, c_2 def}, \eqref{D_1(x), D_2(x) def}, it follows that:
\begin{equation}
\begin{cases}
    g_1(x) = \frac{\sqrt{\mathcal{S}} + B x}{D_1(x)}
    \\
    g_2(x) = \frac{\sqrt{\mathcal{S}}-B x}{D_2(x)}
\end{cases}
\label{g_1(x), g_2(x) definition - thm6}
\end{equation}
We first prove \eqref{g_1(x) > g_2(x)}. Their difference $g_1(x) - g_2(x)$ reads:
\begin{equation}
\begin{aligned}
    g_1(x) - g_2(x) &= \frac{\sqrt{\mathcal{S}} + B x}{D_1(x)} - \frac{\sqrt{\mathcal{S}} - B x}{D_2(x)} = \frac{(\sqrt{\mathcal{S}} + B x)\cdot D_2(x) - (\sqrt{\mathcal{S}}-B x)\cdot D_1(x)}{D_1(x)\cdot D_2(x)} 
    \\
    &= \frac{\sqrt{\mathcal{S}}\cdot \left(D_2(x)-D_1(x)\right) + Bx\cdot \left(D_2(x)+D_1(x)\right)}{D_1(x)\cdot D_2(x)}.
\end{aligned}
\label{g_1 - g_2 formula}
\end{equation}
Now, from \eqref{B, C, c_1, c_2 def}, \eqref{D_1(x), D_2(x) def}, we have:
\begin{equation}
\begin{aligned}
    D_2(x) - D_1(x) &= \frac{D_2^2(x) - D_1^2(x)}{D_2(x) +D_1(x)} = \frac{c_2 - c_1}{D_2(x)+D_1(x)} 
    \\
    &=\frac{\frac{1}{N} - \frac{1}{\gamma\cdot N}}{D_1(x)+D_2(x)}
    \\
    &= -\frac{1}{N}\cdot \frac{1-\gamma}{\gamma}\cdot \frac{1}{D_2(x)+D_1(x)}.
\end{aligned}
\label{D_2-D_1 formula}
\end{equation}
In order to show \eqref{g_1(x) > g_2(x)}, it is sufficient to show that the expression in \eqref{g_1 - g_2 formula} is strictly positive. Substituting \eqref{D_2-D_1 formula}, we get:
\begin{gather*}
    \underbrace{
    -\frac{\sqrt{\mathcal{S}}}{N} \cdot \frac{1 - \gamma}{\gamma} \cdot \frac{1}{D_2(x) + D_1(x)}
    }_{\text{\( \sqrt{\mathcal{S}}\cdot \left(D_2(x) - D_1(x)\right) \)}} + \frac{1-\gamma}{4\gamma\cdot N\cdot \sqrt{\mathcal{S}}}\cdot x\cdot \left(D_2(x)+D_1(x)\right) > 0
    \\
    \frac{1}{4\sqrt{\mathcal{S}}}\cdot x\cdot \left(D_2(x)+D_1(x)\right) > \frac{\sqrt{\mathcal{S}}}{D_2(x)+D_1(x)}
    \\
    x\cdot \left(D_2(x) + D_1(x)\right)^2 > 4\mathcal{S}
\end{gather*}
That is, in order to show \eqref{w_1<w_2}, it is sufficient to show that:
\begin{equation}
    h(x) \coloneqq x\cdot \left(D_2(x) + D_1(x)\right)^2 > 4\mathcal{S}.
\label{h(x) def}
\end{equation}
We now show that $h(x)$ is strictly increasing:
\begin{equation*}
    h'(x) =\left(D_2(x) + D_1(x)\right)^2 + 2x\cdot \left(D_2(x)+D_1(x)\right) > 0.
\end{equation*}
Thus, it follows that:
\begin{equation}
    x \ge 1 \Rightarrow h(x) > h(1).
\label{h(x) > h(1)}
\end{equation}
Now, from \eqref{h(x) def}, it follows that:
\begin{align*}
    h(1) = \left(D_2(1) + D_1(1)\right)^2 > 4\mathcal{S} \Leftrightarrow D_2(1) + D_1(1) > 2\sqrt{\mathcal{S}}.
\end{align*}
Indeed, from \eqref{B, C, c_1, c_2 def}, \eqref{D_1(x), D_2(x) def}, we have:
\begin{align*}
    D_2(1) + D_1(1) > 2\cdot D_2(1) &= 2\cdot \sqrt{C + c_2}
    \\
    &= 2\cdot \sqrt{C+ \frac{1}{N} + 1}
    \\
    &> 2
    \\
    & \ge 2\sqrt{\mathcal{S}}
\end{align*}
where we used the assumption that $\mathcal{S} \le 1$, and $C > 0$. That is, we proved \eqref{g_1(x) > g_2(x)}. We will now prove \eqref{g_1(x) + g_2(x) > 0}. From \eqref{g_1(x), g_2(x) definition - thm6}, the sum $g_1(x) + g_2(x)$ reads:
\begin{equation}
\begin{aligned}
    g_1(x) + g_2(x) &= \frac{\sqrt{\mathcal{S}} + B\cdot x}{D_1(x)} + \frac{\sqrt{\mathcal{S}} - B x}{D_2(x)} = \frac{\left(\sqrt{\mathcal{S}} + B x\right)\cdot D_2(x) + \left(\sqrt{\mathcal{S}} - B x\right)\cdot D_1(x)}{D_1(x) \cdot D_2(x)}
    \\
    &= \frac{\sqrt{\mathcal{S}}\cdot \left(D_2(x) + D_1(x)\right) + B x\cdot \left(D_2(x) - D_1(x)\right)}{D_1(x)\cdot D_2(x)}.
\end{aligned}
\label{g_1 + g_2 formula}
\end{equation}
In order to show \eqref{g_1(x) + g_2(x) > 0}, it is sufficient to show that the expression in \eqref{g_1 + g_2 formula} is strictly positive. Substituting \eqref{D_2-D_1 formula}, we get:
\begin{gather*}
    \sqrt{\mathcal{S}}\cdot \left(D_2(x) + D_1(x)\right) + \underbrace{
    \left(-\frac{1 - \gamma}{4\gamma\cdot N\cdot \sqrt{\mathcal{S}}}\cdot \frac{1}{N}\cdot \frac{1 -\gamma}{\gamma}\cdot \frac{x}{D_2(x) + D_1(x)}\right)
    }_{\text{\( B x\cdot  \left(D_2(x) - D_1(x)\right) \)}} > 0
    \\
    \sqrt{\mathcal{S}}\cdot \left(D_2(x) + D_1(x)\right) > \frac{(1 - \gamma)^2}{4\gamma^2N^2\sqrt{\mathcal{S}}}\cdot \frac{x}{D_2(x) + D_1(x)}
    \\
    \frac{\left(D_2(x) + D_1(x)\right)^2}{x} > \frac{(1 - \gamma)^2}{4\gamma^2N^2\mathcal{S}}
\end{gather*}
That is, in order to show \eqref{g_1(x) + g_2(x) > 0}, it is sufficient to show that:
\begin{equation}
    p(x) \coloneqq \frac{\left(D_2(x) + D_1(x)\right)^2}{x} > \frac{(1 - \gamma)^2}{4\gamma^2N^2\mathcal{S}}.
\label{p(x) def}
\end{equation}
Indeed,
\begin{equation}
\begin{aligned}
    p(x) = \frac{D_2^2(x) + 2D_2(x)D_1(x) + D_1^2(x)}{x} &\ge \frac{D_2^2(x) + D_1^2(x)}{x} 
    \\
    &= \frac{Cx + c_1 + Cx + c_2}{x}
    \\
    &= 2C + \frac{c_1 + c_2}{x}
    \\
    &> 2C
    \\
    &= \frac{1 + \gamma}{2\gamma N \mathcal{S}} + \frac{1 + \gamma^2}{4\gamma^2 N^2 \mathcal{S}}
    \\
    &> \frac{1 + \gamma^2}{4\gamma^2 N^2 \mathcal{S}}
    \\
    &> \frac{(1 - \gamma)^2}{4\gamma^2N^2\mathcal{S}}
\end{aligned}
\end{equation}
where we used $c_1, c_2 > 0$ and $(1 - \gamma)^2 < 1 + \gamma^2$ for all $\gamma > 0$. That is, we proved \eqref{g_1(x) + g_2(x) > 0} and thus we showed that \eqref{w_1<w_2} is satisfied. \underline{We will now prove \eqref{g_2'<0}:} Let us first define the following parametric function
\begin{equation}
    T_{B,C,D}(x)= \frac{\sqrt{\mathcal{S}} + Bx}{\sqrt{C x + D}}.
\label{T_B,C,D def}
\end{equation}
Its derivative reads:
\begin{equation}
\begin{aligned}
    T'_{B,C,D}(x) &= \frac{B\cdot \sqrt{C x + D} - \frac{C}{2\cdot \sqrt{C x+D}}\cdot \left(\sqrt{\mathcal{S}} + B x\right)}{C x+D} 
    \\
    &= \frac{2B\cdot (C x+D) - C\cdot (\sqrt{\mathcal{S}}+B x)}{2\cdot (C x+D)^{1.5}}
    \\
    &= \frac{BC\cdot x + 2\cdot BD - \sqrt{\mathcal{S}}C}{2\cdot (C x+D)^{1.5}}.
\end{aligned}
\label{T' formula}
\end{equation}

Now, from \eqref{g_2(x) def}, \eqref{T_B,C,D def}, it follows that:
\begin{align*}
    g_2(x) = T_{-B,C, c_2}(x).
\end{align*}
That is, from \eqref{T' formula}, we have:
\begin{equation}
    g_2'(x) = \frac{-BC\cdot x -2B\cdot c_2 - \sqrt{\mathcal{S}}\cdot C}{2\cdot \left(C x + c_2\right)^{1.5}} = - \frac{BC\cdot x + 2B\cdot c_2 + \sqrt{\mathcal{S}}\cdot C}{2\cdot (C x+c_2)^{1.5}} < 0
\label{g_2' formula}
\end{equation}
where we used $B, C, c_2 > 0$, which follows from \eqref{B, C, c_1, c_2 def}, and $\mathcal{S} > 0$.
\underline{Finally, we will prove \eqref{g_1'+g_2'<=0}:} From \eqref{g_1(x) def}, \eqref{T_B,C,D def}, \eqref{T' formula}, it follows that:
\begin{equation}
    g_1(x) = T_{B,C,c_1}(x) \Rightarrow g_1'(x) = \frac{BC\cdot x + 2B\cdot c_1 - \sqrt{\mathcal{S}}\cdot C}{2\cdot \left(C x + c_1\right)^{1.5}}.
\label{g_1' formula}
\end{equation}
Thus, from \eqref{g_2' formula},\eqref{g_1' formula}, proving that $g_1'(x) + g_2'(x) \le 0$ is equivalent to proving that:
\begin{gather*}
    \frac{BC\cdot x + 2B\cdot c_1 - \sqrt{\mathcal{S}}\cdot C}{2\cdot \left(C x + c_1\right)^{1.5}} \le \frac{BC\cdot x + 2B\cdot c_2 + \sqrt{\mathcal{S}}\cdot C}{2\cdot (C x+c_2)^{1.5}}.
\end{gather*}
From \eqref{B, C, c_1, c_2 def}, and the assumption of $0 < \gamma < 1$, we know that $c_1 > c_2$. Thus, if the numerator in the LHS is negative, then the inequality holds trivially. Otherwise, it is sufficient to prove that:
\begin{gather*}
    2B\cdot c_1-\sqrt{\mathcal{S}}\cdot C \le 2B\cdot c_2 + \sqrt{\mathcal{S}}\cdot C
    \\
    2B\cdot (c_1-c_2) \le 2\sqrt{\mathcal{S}}\cdot C
    \\
    B\cdot \left(\frac{1}{\gamma\cdot N} - \frac{1}{N}\right) \le \sqrt{\mathcal{S}}\cdot C
\end{gather*}
Thus, we need to prove that:
\begin{equation}
    \frac{1}{N}\cdot \frac{1-\gamma}{\gamma} \le \frac{\sqrt{\mathcal{S}}\cdot C}{B}.
\label{final stuff to show theorem 5}
\end{equation}
From \eqref{B, C, c_1, c_2 def}, the RHS in \eqref{final stuff to show theorem 5} reads:
\begin{align*}
    \frac{\sqrt{\mathcal{S}}\cdot C}{B} &= \frac{\left(\frac{1 + \gamma}{4\cdot \gamma\cdot N\cdot \sqrt{\mathcal{S}}} + \frac{1 + \gamma^2}{8\cdot \gamma^2\cdot N^2\cdot \sqrt{\mathcal{S}}}\right)}{\left(\frac{1 - \gamma}{4\cdot \gamma\cdot N\cdot \sqrt{\mathcal{S}}}\right)}
    = \frac{\left(\frac{1 + \gamma}{4} + \frac{1 + \gamma^2}{8\cdot \gamma\cdot N}\right)}{\left(\frac{1-\gamma}{4}\right)}
    \\
    &= \left(\frac{1+\gamma}{4} + \frac{1+\gamma^2}{8\cdot \gamma\cdot N}\right)\cdot \frac{4}{1-\gamma}
    \\
    &= \frac{1 + \gamma}{1 - \gamma} + \frac{1 + \gamma^2}{\gamma\cdot (1-\gamma)}\cdot \frac{1}{2\cdot N}
    \\
    &= \frac{2N\cdot \gamma(1+\gamma) + 1 + \gamma^2}{2N\cdot \gamma(1-\gamma)}
    \\
    &= \frac{(2N+1)\cdot \gamma^2 + 2N\cdot \gamma + 1}{2N\cdot \gamma(1-\gamma)}.
\end{align*}
Thus, \eqref{final stuff to show theorem 5} reads:
\begin{gather*}
    \frac{1 - \gamma}{\gamma\cdot N} \le \frac{(2N+1)\cdot \gamma^2 + 2N\cdot \gamma + 1}{2N\cdot \gamma(1-\gamma)}
    \\
    2\cdot (1-\gamma)^2 \le (2N+1)\cdot \gamma^2+2N\cdot \gamma+1
    \\
    2\cdot (\gamma^2-2\gamma+1) \le(2N+1)\cdot\gamma^2+2N\cdot\gamma+1
    \\
    (2N-1)\cdot \gamma^2+2\cdot (N+2)\cdot \gamma-1 \ge 0
    \\
    2\cdot \gamma\cdot (\gamma + 1)\cdot N - (\gamma^2 - 4\gamma + 1) \ge 0
    \\
    N \ge \frac{\gamma^2 - 4\gamma + 1}{2 \gamma(1+\gamma)}
\end{gather*}
Which holds from the theorem assumptions. Finally, let us take some $x \ge 1$. We need to prove that:
\begin{align*}
    w_1(x)\cdot g_1'(x) + w_2(x)\cdot g_2'(x) < 0 &\Leftrightarrow w_1(x)\cdot g_1'(x) < -w_2(x)\cdot g_2'(x)
    \\
    &\Leftrightarrow g_1'(x) < -\frac{w_2(x)}{w_1(x)}\cdot g_2'(x).
\end{align*}
Indeed, from \eqref{w_1<w_2}, \eqref{g_2' formula}, \eqref{g_1'+g_2'<=0}, it follows that:
\begin{align*}
    g_1'(x) \le -g_2'(x) < \frac{w_2(x)}{w_1(x)}\cdot \left(-g_2'(x)\right) = -\frac{w_2(x)}{w_1(x)}\cdot g_2'(x).
\end{align*}
Which finishes the proof. {\color{black} Note that for $\gamma \ge 0.162$, the requirement $N \ge \tfrac{\gamma^2 - 4\gamma + 1}{2\gamma\left(1 + \gamma\right)}$ is vacuous (since $N \ge 1$), so it only matters under severe imbalance $(\gamma < 0.162)$.}
\end{proof}

\subsection{Proof of Theorem \ref{thm:thm7}}

In order to have a fair comparison between cases with different values of $\gamma$, we fix the total number of samples to be $N_{T}$. Thus, we take $N_1=xN_{T}$ samples from the first class and $N_2=\gamma\cdot xN_{T}$ from the second class such that:
\begin{align*}
    N_1+N_2 = xN_{T} + \gamma\cdot xN_{T} = N_{T}\Rightarrow x= \frac{1}{1+ \gamma}
\end{align*}
Meaning, the number of samples in the first class is
\begin{equation}
    N = \frac{N_{T}}{1+\gamma}.
\label{N tilde formula}
\end{equation}

\begin{proof}
Let us take some $N_T \in \mathbb{N}$ and
\begin{equation*}
    \mathcal{S} > 0, 1 \le k < d, 0 < \gamma \le 1.
\end{equation*}

Let us define the following parametric function:
\begin{equation}
    f_{s,a,q}(x) \coloneqq \frac{\sqrt{\mathcal{S}} +s\cdot  \frac{(1 - \gamma)\cdot q}{4\gamma\cdot \sqrt{\mathcal{S}}}\cdot \frac{1}{x}}{\sqrt{\left(\frac{(1+\gamma)\cdot q}{4\gamma\cdot \mathcal{S}} + \frac{1}{a}\right)\cdot \frac{1}{x} + \frac{(1 + \gamma^2)\cdot q}{8\gamma^2\cdot \mathcal{S}}\cdot \frac{1}{x^2} + 1}} = \frac{\sqrt{\mathcal{S}} + \frac{B}{x}}{\sqrt{\frac{C}{x} + \frac{D}{x^2} + 1}}
\label{f_s,a,q}
\end{equation}
where the parameters $B,C,D$ are:
\begin{equation}
\begin{cases}
    B = s\cdot \frac{(1-\gamma)\cdot q}{4\gamma\cdot \sqrt{\mathcal{S}}}
    \\
    C = \frac{(1+\gamma)\cdot q}{4\gamma\cdot \mathcal{S}} + \frac{1}{a}
    \\
    D = \frac{(1+\gamma^2)\cdot q}{8\gamma^2\cdot \mathcal{S}}
\end{cases}
\label{B,C,D parameters}
\end{equation}
From Definition \ref{eq:theoretical efficiency} and \eqref{approx prob x}, it follows that:
\begin{equation}
\begin{aligned}
    \eta &= 100\cdot \left(1 - \frac{\hat{p}_{\vct{z}}(\mathrm{error})}{\hat{p}_{\vct{x}}(\mathrm{error})}\right) = 100\cdot \left(1 - \frac{\hat{p}(\mathcal{S},{N},\gamma,k)}{\hat{p}(\mathcal{S},{N},\gamma,d)}\right)
    \\
    &= 100\cdot \left(1 - \frac{\mathcal{Q}\left(f_{1, \gamma, k}({N})\right) + \mathcal{Q}\left(f_{-1, 1, k}({N})\right)}{\mathcal{Q}\left(f_{1,\gamma,d}({N})\right) + \mathcal{Q}\left(f_{-1,1,d}({N})\right)}\right)
    \\
    &= 100\cdot h({N})
\end{aligned}
\label{eta first expansion as 100*h(N)}
\end{equation}
where we defined the following function:
\begin{equation}
\label{eq:h(x)_def_b}
    h(x) \coloneqq 1 - \frac{\mathcal{Q}\left(f_{1, \gamma, k}(x)\right) + \mathcal{Q}\left(f_{-1, 1, k}(x)\right)}{\mathcal{Q}\left(f_{1,\gamma,d}(x)\right) + \mathcal{Q}\left(f_{-1,1,d}(x)\right)}.
\end{equation}
Now, let us define the following parametric function: 
\begin{equation}
    g_{s,a,q}(x) \coloneqq f_{s,a,q}\left(\frac{1}{x}\right) = \frac{\sqrt{\mathcal{S}} + B\cdot x}{\sqrt{D\cdot x^2+ C\cdot x+1}}
\label{g_s,a,q formula}
\end{equation}
where we used \eqref{f_s,a,q}, \eqref{B,C,D parameters}. Let us also define:
\begin{equation}
    \ell(x) \coloneqq h\left(\frac{1}{x}\right) = 1 - \frac{\mathcal{Q}\left(g_{1, \gamma, k}(x)\right) + \mathcal{Q}\left(g_{-1, 1, k}(x)\right)}{\mathcal{Q}\left(g_{1,\gamma,d}(x)\right) + \mathcal{Q}\left(g_{-1,1,d}(x)\right)}
\label{l(x) def}
\end{equation}
where we used \eqref{g_s,a,q formula}, \eqref{eq:h(x)_def_b}. Thus, Taylor expansion to first order of $\ell$ yields:
\begin{equation*}
    \ell(x) = \ell(0) + \ell'(0)\cdot x + \mathcal{O}\left(x^2\right)
\end{equation*}
Where the approximation is exact for $x \ll 1$. Thus, the following is exact for $x \gg 1$:
\begin{equation}
    x \gg 1 \Rightarrow h(x) = \ell\left(\frac{1}{x}\right) =  \ell(0) + \frac{\ell'(0)}{x} + \mathcal{O}\left(\frac{1}{x^2}\right).
\label{h(x) approximation}
\end{equation}
Assuming:
\begin{equation}
    {N} = \frac{N_{T}}{1+\gamma} \gg 1\Leftrightarrow N_{T} \gg 1 + \gamma
\label{condition on N}
\end{equation}
means that following first-order approximation is exact:
\begin{equation}
    h\left({N}\right) = \ell(0) + \frac{\ell'(0)}{{N}} + \mathcal{O}\left(\frac{1}{N^2}\right).
\label{h(N tilde) approximation}
\end{equation}
\underline{Let us first compute $\ell(0)$:}
\begin{equation}
    \ell(0) = 1 - \frac{\mathcal{Q}\left(g_{1,\gamma,k}(0)\right) +\mathcal{Q}\left(g_{-1,1,k}(0)\right)}{\mathcal{Q}\left(g_{1,\gamma,d}(0)\right)+ \mathcal{Q}\left(g_{-1,1,d}(0)\right)} = 1- \frac{2\cdot \mathcal{Q}(\sqrt{\mathcal{S}})}{2\cdot \mathcal{Q}(\sqrt{\mathcal{S}})} = 0.
\label{ell(0)}
\end{equation}
\underline{Finally, we will compute $\ell'(0)$:} We first compute $g_{s,a,q}'(0)$. From \eqref{g_s,a,q formula} it follows that:
\begin{align*}
    g'_{s,a,q}(x) &= \frac{B\cdot \sqrt{D\cdot x^2 + C\cdot x + 1} - \frac{2D\cdot x + C}{2\cdot \sqrt{D\cdot x^2+C\cdot x + 1}}\cdot \left(\sqrt{\mathcal{S}} + B\cdot x\right)}{D\cdot x^2+C\cdot x+1}
    \\
    &= \frac{2B\cdot \left(D\cdot x^2+C\cdot x+ 1\right) - \left(2D\cdot x + C\right)\cdot \left(\sqrt{\mathcal{S}} + B\cdot x\right)}{2\cdot \left(D\cdot x^2+C\cdot x + 1\right)^{1.5}}
    \\
    &= \frac{\left(BC - 2SD\right)\cdot x + (2B - C\sqrt{\mathcal{S}})}{2\cdot \left(D\cdot x^2+C\cdot x + 1\right)^{1.5}}.
\end{align*}
Thus, the derivative at 0 is:
\begin{equation}
\begin{aligned}
    g_{s,a,q}'(0) &= \frac{2B - C\sqrt{\mathcal{S}}}{2} = B - \frac{1}{2}\cdot C\sqrt{\mathcal{S}}
    \\
    &= \frac{s\cdot (1 - \gamma)\cdot q}{4\gamma\cdot S} - \frac{1}{2}\cdot \left(\frac{(1+ \gamma)\cdot q}{4\gamma\cdot \mathcal{S}} + \frac{1}{a}\right)\cdot \sqrt{\mathcal{S}}
    \\
    &= \frac{s\cdot (1-\gamma)\cdot q}{4\gamma\cdot S} - \frac{(1+\gamma)\cdot q}{8\gamma\cdot \sqrt{\mathcal{S}}} - \frac{\sqrt{\mathcal{S}}}{2a} 
    \\
    &= \frac{2s\cdot (1 - \gamma)\cdot q -  (1+\gamma)\cdot q}{8\gamma\cdot S} - \frac{\sqrt{\mathcal{S}}}{2a}
    \\
    &= \frac{\left(\left(2s-1\right) -  \left(2s + 1\right)\cdot \gamma\right)\cdot q}{8\gamma\cdot \sqrt{\mathcal{S}}} - \frac{\sqrt{\mathcal{S}}}{2a}.
\end{aligned}
\label{g_s,a,q'(0)}
\end{equation}

Now, from \eqref{l(x) def}, the derivative $\ell'(x)$ reads
\begin{equation}
\begin{aligned}
    \ell'(x) &= -\frac{d}{dx}\left(\frac{\mathcal{Q}\left(g_{1,\gamma,k}(x)\right) + \mathcal{Q}\left(g_{-1,1,k}(x)\right)}{\mathcal{Q}\left(g_{1,\gamma,d}(x)\right) + \mathcal{Q}\left(g_{-1,1,d}(x)\right)}\right)
    \\
    &= - \frac{d}{dx}\left(\frac{u(x)}{v(x)}\right)
    \\
    &= - \frac{u'(x)\cdot v(x) - v'(x)\cdot u(x)}{v^2(x)}
    \\
    &= \frac{v'(x)\cdot u(x) - u'(x)\cdot v(x)}{v^2(x)}
\end{aligned}
\label{ell'(x) first formula}
\end{equation}
where we defined the following auxiliary functions:
\begin{equation}
\begin{cases}
    u(x) = \mathcal{Q}\left(g_{1,\gamma,k}(x)\right) + \mathcal{Q}\left(g_{-1,1,k}(x)\right) 
    \\
    v(x) = \mathcal{Q}\left(g_{1,\gamma,d}(x)\right) + \mathcal{Q}\left(g_{-1,1,d}(x)\right) 
\end{cases}
\label{u(x),v(x) def to l}
\end{equation}
From the chain rule, their derivatives are:
\begin{equation}
\begin{cases}
    u'(x)=g'_{1,\gamma,k}(x)\cdot \mathcal{Q}'\left(g_{1,\gamma,k}(x)\right) + g'_{-1,1,k}(x)\cdot \mathcal{Q}'\left(g_{-1,1,k}(x)\right)
    \\
    v'(x)=g'_{1,\gamma,d}(x)\cdot \mathcal{Q}'\left(g_{1,\gamma,d}(x)\right) + g'_{-1,1,d}(x)\cdot \mathcal{Q}'\left(g_{-1,1,d}(x)\right)
\end{cases}
\label{u'(x), v'(x) formulas for ell}
\end{equation}
Now, from \eqref{ell'(x) first formula}, \eqref{u'(x), v'(x) formulas for ell} it follows that:
\begin{equation}
    \ell'(0) = \frac{v'(0)\cdot u(0) - u'(0)\cdot v(0)}{v(0)^2}.
\label{l'(0) formula}
\end{equation}
It is easy to verify from \eqref{g_s,a,q formula} that $g_{s,a,q}(0) = \sqrt{\mathcal{S}}$. Thus, from \eqref{u(x),v(x) def to l}, \eqref{u'(x), v'(x) formulas for ell}, \eqref{g_s,a,q'(0)}, we have the following formulas:
\begin{equation}
\begin{cases}
    u(0) = 2\cdot \mathcal{Q}(\sqrt{\mathcal{S}})
    \\
    v(0) = 2\cdot \mathcal{Q}(\sqrt{\mathcal{S}})
    \\
    u'(0) = \left(\frac{(1 - 3\gamma)\cdot k}{8\gamma\cdot \sqrt{\mathcal{S}}} - \frac{\sqrt{\mathcal{S}}}{2\gamma}\right)\cdot \mathcal{Q}'(\sqrt{\mathcal{S}}) + \left(\frac{-(3 + \gamma)\cdot k}{8\gamma\cdot \sqrt{\mathcal{S}}} - \frac{\sqrt{\mathcal{S}}}{2}\right)\cdot \mathcal{Q}'(\sqrt{\mathcal{S}})
    \\
    v'(0) = \left(\frac{(1 - 3\gamma)\cdot d}{8\gamma\cdot \sqrt{\mathcal{S}}} - \frac{\sqrt{\mathcal{S}}}{2\gamma}\right)\cdot \mathcal{Q}'(\sqrt{\mathcal{S}}) + \left(\frac{-(3 + \gamma)\cdot d}{8\gamma\cdot \sqrt{\mathcal{S}}} - \frac{\sqrt{\mathcal{S}}}{2}\right)\cdot \mathcal{Q}'(\sqrt{\mathcal{S}})
\end{cases}
\label{u(0),v(0),u'(0),v'(0) formulas}
\end{equation}
Now, we substitute \eqref{u(0),v(0),u'(0),v'(0) formulas} in \eqref{l'(0) formula}, to get the following formula for $\ell'(0)$:
\begin{equation}
\begin{aligned}
    \ell'(0) &= \frac{2\cdot \mathcal{Q}(\sqrt{\mathcal{S}})\cdot \left(\left(\frac{(1 - 3\gamma)\cdot d}{8\gamma\cdot \sqrt{\mathcal{S}}} - \frac{\sqrt{\mathcal{S}}}{2\gamma}\right)\cdot \mathcal{Q}'(\sqrt{\mathcal{S}}) + \left(\frac{-(3 + \gamma)\cdot d}{8\gamma\cdot \sqrt{\mathcal{S}}} - \frac{\sqrt{\mathcal{S}}}{2}\right)\cdot \mathcal{Q}'(\sqrt{\mathcal{S}})\right)}{4\cdot \mathcal{Q}^2(\sqrt{\mathcal{S}})}
    \\
    &- \frac{2 \mathcal{Q}(\sqrt{\mathcal{S}})\cdot \left( \left(\frac{(1 - 3\gamma)\cdot k}{8\gamma\cdot \sqrt{\mathcal{S}}} - \frac{\sqrt{\mathcal{S}}}{2\gamma}\right)\cdot \mathcal{Q}'(\sqrt{\mathcal{S}}) + \left(\frac{-(3 + \gamma)\cdot k}{8\gamma\cdot \sqrt{\mathcal{S}}} - \frac{\sqrt{\mathcal{S}}}{2}\right)\cdot \mathcal{Q}'(\sqrt{\mathcal{S}}) \right )}{4\cdot \mathcal{Q}^2(\sqrt{\mathcal{S}})}
    \\
    &= \frac{\mathcal{Q}'(\sqrt{\mathcal{S}})}{2\cdot \mathcal{Q}(\sqrt{\mathcal{S}})}\cdot \left(\frac{(1 - 3\gamma)\cdot (d-k)}{8\gamma\cdot \sqrt{\mathcal{S}}} - \frac{(3 + \gamma)\cdot (d-k)}{8\gamma\cdot \sqrt{\mathcal{S}}}\right)
    \\
    &= \frac{\mathcal{Q}'(\sqrt{\mathcal{S}})}{2\cdot \mathcal{Q}(\sqrt{\mathcal{S}})}\cdot (d-k)\cdot \left(\frac{-2 -4 \gamma}{8\gamma\cdot \sqrt{\mathcal{S}}}\right)
    \\
    &= -\frac{\mathcal{Q}'(\sqrt{\mathcal{S}})}{\mathcal{Q}(S)}\cdot (d-k)\cdot \left(\frac{1 + 2 \gamma}{8\gamma\cdot \sqrt{\mathcal{S}}}\right)
    \\
    &= -\frac{\mathcal{Q}'(\sqrt{\mathcal{S}})}{\sqrt{\mathcal{S}}\cdot \mathcal{Q}(\sqrt{\mathcal{S}})}\cdot (d-k)\cdot \left(\frac{1}{8\gamma} + \frac{1}{4}\right).
\end{aligned}
\label{ell'(0) final formula-theorem 6}
\end{equation}
Finally, from \eqref{eta first expansion as 100*h(N)}, \eqref{condition on N}, \eqref{h(N tilde) approximation}, \eqref{ell(0)}, \eqref{ell'(0) final formula-theorem 6}, it follows that:
\begin{equation}
\begin{aligned}
    \eta&=100\cdot h({N})
    \\
    &= 100\cdot \left(\ell(0) + \frac{\ell'(0)}{{N}} + \mathcal{O}\left(\frac{1}{N^2}\right)\right) 
    \\
    &= 100\cdot \frac{\ell'(0)}{{N}}  + \mathcal{O}\left(\frac{1}{N^2}\right)
    \\
    &= -\frac{100\cdot \mathcal{Q}'(\sqrt{\mathcal{S}})}{\sqrt{\mathcal{S}}\cdot \mathcal{Q}(\sqrt{\mathcal{S}})}\cdot (d-k)\cdot \left(\frac{1}{8\gamma} + \frac{1}{4}\right)\cdot \frac{1}{{N}} + \mathcal{O}\left(\frac{1}{N^2}\right)
\end{aligned}
\label{last - first approx for eta}
\end{equation}
We proceed with \eqref{last - first approx for eta} and substitute \eqref{condition on N}:
\begin{equation*}
    {N} = \frac{N_{T}}{1+\gamma}
\end{equation*}
to get the following approximation for $\eta$:
\begin{equation}
\begin{aligned}
    \eta &= -100\cdot \frac{\mathcal{Q}'(\sqrt{\mathcal{S}})}{\sqrt{\mathcal{S}}\cdot \mathcal{Q}(\sqrt{\mathcal{S}})}\cdot (d-k)\cdot \left(\frac{1}{8\gamma} + \frac{1}{4}\right)\cdot \left(1 + \gamma\right)\cdot \frac{1}{N_{T}} + \mathcal{O}\left(\frac{1}{N_T^2}\right)
    \\
    &= -50\cdot \frac{\mathcal{Q}'(\sqrt{\mathcal{S}})}{\sqrt{\mathcal{S}}\cdot \mathcal{Q}(\sqrt{\mathcal{S}})}\cdot (d-k)\cdot \left(\frac{1}{4\gamma} + \frac{1}{2}\right)\cdot (1+\gamma)\cdot \frac{1}{N_{T}} + \mathcal{O}\left(\frac{1}{N_T^2}\right)
    \\
    &= -50\cdot \frac{\mathcal{Q}'(\sqrt{\mathcal{S}})}{\sqrt{\mathcal{S}}\cdot \mathcal{Q}(\sqrt{\mathcal{S}})}\cdot (d-k)\cdot \left(\frac{1 + 2\gamma}{4\gamma}\right)\cdot (1+\gamma)\cdot \frac{1}{N_{T}} + \mathcal{O}\left(\frac{1}{N_T^2}\right)
    \\
    &= - \frac{50\cdot \mathcal{Q}'(\sqrt{\mathcal{S}})}{4\cdot \sqrt{\mathcal{S}} \mathcal{Q}(\sqrt{\mathcal{S}})}\cdot (d-k)\cdot (1 + 2\gamma)\cdot \left(1 + \frac{1}{\gamma}\right)\cdot \frac{1}{N_{T}} + \mathcal{O}\left(\frac{1}{N_T^2}\right)
    \\
    &= - \frac{25\cdot \mathcal{Q}'(\sqrt{\mathcal{S}})}{2\cdot \sqrt{\mathcal{S}} \mathcal{Q}(\sqrt{\mathcal{S}})}\cdot (d-k)\cdot \left(3 + 2\gamma + \frac{1}{\gamma}\right)\cdot \frac{1}{N_{T}} + \mathcal{O}\left(\frac{1}{N_T^2}\right)
    \\
    &= \frac{25}{2 \sqrt{2\pi}}\cdot \frac{\exp\left(-\frac{\mathcal{S}}{2}\right)}{\sqrt{\mathcal{S}}\cdot \mathcal{Q}\left(\sqrt{\mathcal{S}}\right)}\cdot \left(3 + 2\gamma + \frac{1}{\gamma}\right)\cdot (d-k)\cdot \frac{1}{N_{T}} + \mathcal{O}\left(\frac{1}{N_T^2}\right)
\end{aligned}
\label{final approximation for eta}
\end{equation}
where we used the following property of the $\mathcal{Q}$ function:
\begin{equation*}
    \mathcal{Q}'(x) = -\frac{1}{\sqrt{2\pi}}\cdot \exp\left(-\frac{x^2}{2}\right).
\end{equation*}
It is now left to show the conclusions. For $N_T \gg 1$, we have from \eqref{final approximation for eta} that
\begin{equation}
    \eta = C\cdot f(\mathcal{S})\cdot g(\gamma)\cdot (d-k)\cdot \frac{1}{N_T}
\label{eta approx cleaned form}
\end{equation}
where $C = \frac{25}{2\sqrt{2\pi}} > 0$, and
\begin{equation}
\begin{cases}
    f(\mathcal{S}) = \frac{\exp\left(-\frac{\mathcal{S}}{2}\right)}{\sqrt{\mathcal{S}}\cdot \mathcal{Q}(\sqrt{\mathcal{S}})}
    \\
    g(\gamma) = 3+ 2\gamma + \frac{1}{\gamma}
\end{cases}
\label{f,g definition}
\end{equation}
It is now clear from \eqref{eta approx cleaned form} that as $d - k$ increases, the efficiency increases (linearly), and as $N_T$ increases, the efficiency decreases. It is easy to see that in $(0,1]$, the function $g$ defined in \eqref{f,g definition} achieves a minimum at $\gamma = \frac{1}{\sqrt{2}}$:
\begin{equation*}
    g'(\gamma) = 2 - \frac{1}{\gamma^2} = 0\Rightarrow \gamma = \pm \frac{1}{\sqrt{2}}.
\end{equation*}
It is easy to check that
\begin{equation*}
    g\left(\frac{1}{\sqrt{2}}\right) < \lim_{x \rightarrow 0^+} g(x) = \infty , \ g\left(\frac{1}{\sqrt{2}}\right) < g(1).
\end{equation*}
Thus, from \eqref{eta approx cleaned form}, in the range $\left(0, \frac{1}{\sqrt{2}}\right]$, as $\gamma$ decreases, the efficiency increases. Finally, we will prove that the function $f$, defined in \eqref{f,g definition}, decreases as $\mathcal{S}$ increases, and thus from \eqref{eta approx cleaned form}, the efficiency decreases as $\mathcal{S}$ increases: It is now sufficient to prove the following:
\begin{gather*}
    f'(\mathcal{S}) = \frac{-\frac{1}{2}\exp\left(-\frac{\mathcal{S}}{2}\right)\cdot \sqrt{\mathcal{S}}\mathcal{Q}\left(\sqrt{\mathcal{S}}\right) - \frac{d}{d\mathcal{S}}\left(\sqrt{\mathcal{S}}\mathcal{Q}\left(\sqrt{\mathcal{S}}\right)\right)\cdot \exp\left(-\frac{\mathcal{S}}{2}\right)}{\mathcal{S}\mathcal{Q}^2\left(\sqrt{\mathcal{S}}\right)} < 0
    \\
    -\frac{1}{2}\sqrt{\mathcal{S}}\cdot\mathcal{Q}\left(\sqrt{\mathcal{S}}\right)-\frac{d}{d\mathcal{S}}\left(\sqrt{\mathcal{S}}\mathcal{Q}\left(\sqrt{\mathcal{S}}\right)\right) < 0
    \\
    -\frac{1}{2}\sqrt{\mathcal{S}}\mathcal{Q}\left(\sqrt{\mathcal{S}}\right) < \frac{1}{2\sqrt{\mathcal{S}}}\mathcal{Q}\left(\sqrt{\mathcal{S}}\right) + \sqrt{\mathcal{S}}\cdot \mathcal{Q}'\left(\sqrt{\mathcal{S}}\right)\frac{1}{2\sqrt{\mathcal{S}}}
    \\
    -\sqrt{\mathcal{S}}\mathcal{Q}\left(\sqrt{\mathcal{S}}\right) < \frac{1}{\sqrt{\mathcal{S}}}\mathcal{Q}\left(\sqrt{\mathcal{S}}\right) - \frac{1}{\sqrt{2\pi}}\exp\left(-\frac{\mathcal{S}}{2}\right)
    \\
    \left(\sqrt{\mathcal{S}} + \frac{1}{\sqrt{\mathcal{S}}}\right)\mathcal{Q}\left(\sqrt{\mathcal{S}}\right) > \frac{1}{\sqrt{2\pi}}\exp\left(-\frac{\mathcal{S}}{2}\right)
    \\
    \mathcal{Q}\left(\sqrt{\mathcal{S}}\right) > \frac{1}{\sqrt{2\pi}}\cdot \frac{\sqrt{\mathcal{S}}}{\mathcal{S} + 1}\cdot \exp\left(-\frac{\mathcal{S}}{2}\right)
\end{gather*}
Where we used the identity $\mathcal{Q}'(\mathcal{S}) = - \frac{1}{\sqrt{2\pi}}\cdot \exp\left(-\frac{\mathcal{S}^2}{2}\right)$. Let us now prove the final inequality. It is equivalent to the following inequality $\left(x = \sqrt{\mathcal{S}}\right)$:
\begin{equation}
    \forall_{x \ge 0} \ \mathcal{Q}(x) > \frac{1}{\sqrt{2\pi}}\cdot \frac{x}{x^2 + 1}\cdot\exp\left(-\frac{x^2}{2}\right) = \frac{x}{x^2+1}\cdot \phi(x)
\label{inequality for Q}
\end{equation}
where we defined the following function:
\begin{equation}
    \phi(x) = \frac{1}{\sqrt{2\pi}}\cdot \exp\left(-\frac{x^2}{2}\right).
\label{phi definition}
\end{equation}
Indeed, for all $x \ge 0$:
\begin{equation}
\begin{aligned}
    \left(1 + \frac{1}{x^2}\right)\cdot \mathcal{Q}(x) &= \int_x^\infty \left(1 + \frac{1}{x^2}\right)\cdot \phi(u) \ du 
    \\
    &> \int_x^\infty \left(1 + \frac{1}{u^2}\right)\cdot \phi(u) \ du 
    \\
    &= -\left[\frac{\phi(u)}{u}\right]_x^\infty
    \\
    &= \frac{\phi(x)}{x}
\label{inequality of Q - first step}
\end{aligned}
\end{equation}
where we used the following identity:
\begin{equation}
    \frac{d}{du}\left(-\frac{\phi(u)}{u}\right) = - \frac{\phi'(u)\cdot u - \phi(u)}{u^2} = \frac{\phi(u) + u^2\cdot \phi(u)}{u^2} = \left(1 + \frac{1}{u^2}\right)\cdot \phi(u).
\label{identity of phi}
\end{equation}
And \eqref{identity of phi} follows from the identity $\phi'(u) = - u\cdot \phi(u)$ which is straightforward from the definition of $\phi$ in \eqref{phi definition}. Now, from \eqref{inequality of Q - first step}, it follows that:
\begin{equation*}
    \mathcal{Q}(x) > \frac{\phi(x)}{x}\cdot \frac{x^2}{x^2 + 1} = \frac{x}{x^2 + 1}\cdot \phi(x)
\end{equation*}
which proves exactly \eqref{inequality for Q}. It is left to show that: 
\begin{equation}
    \exists_{N_0 \in \mathbb{N}}\forall_{N \ge N_0} \ p_{\vct{z}}(\mathrm{error}) < p_{\vct{x}}(\mathrm{error}).
\label{conclusion thm 7 to prove}
\end{equation}
That is, $\eta > 0$. We proved in \eqref{l'(0) formula} that:
\begin{equation}
    \ell'(0) = -\frac{\mathcal{Q}'(\sqrt{\mathcal{S}})}{\sqrt{\mathcal{S}}\cdot \mathcal{Q}(\sqrt{\mathcal{S}})}\cdot (d-k)\cdot \left(\frac{1}{8\gamma} + \frac{1}{4}\right) > 0.
\label{ell'(0) > 0 proof}
\end{equation}
We defined $\ell(x)$ in \eqref{h(x) approximation} as:
\begin{equation}
    \ell(x) = h\left(\frac{1}{x}\right) \Rightarrow \ell'(x) = -\frac{1}{x^2}\cdot h'\left(\frac{1}{x}\right)
\label{ell in terms of h}
\end{equation}
where we used the chain rule. Finally, from \eqref{ell'(0) > 0 proof}, \eqref{ell in terms of h}, it follows that:
\begin{equation*}
    \ell'(0) = - \lim_{x \rightarrow 0^+} \frac{1}{x^2}\cdot h'\left(\frac{1}{x}\right) = -\lim_{t \rightarrow \infty} t^2\cdot h'\left(t\right) > 0
\end{equation*}
where we used the fact that if the two-sided limit exists, then each one-sided limit exists and they are equal to the limit. In total,
\begin{equation*}
    \lim_{t \rightarrow \infty} t^2\cdot h'(t) < 0.
\end{equation*}
This means that for large enough $t$, we have 
\begin{equation*}
    h'(t) < 0
\end{equation*}
which implies that $h$ is strictly decreasing. Thus, from \eqref{eta first expansion as 100*h(N)}, we have that for $N \gg 1$, $\eta$ is decreasing. In addition, from \eqref{h(x) approximation}, \eqref{ell(0)}, it follows that:
\begin{equation*}
    \lim_{N \rightarrow \infty} \eta = 100\cdot \lim_{N \rightarrow \infty} h(N) = 100\cdot \lim_{N \rightarrow \infty} \ell\left(\frac{1}{N}\right) = 100\cdot \ell(0) = 0
\end{equation*}
Finally, $\eta$ is decreasing for large enough $N$ and approaches $0$. It is now easy to see that:
\begin{equation*}
    \exists_{N_0 \in \mathbb{N}}\forall_{N \ge N_0} \ \eta > 0
\end{equation*}
which exactly proves \eqref{conclusion thm 7 to prove}.
\end{proof}

\subsection{Proof of Theorem \ref{thm:thm8}}
\label{app:ext_max_eta_thm}

Let us state an extended and more detailed version of the Theorem~\ref{thm:thm8}.

\begin{theorem*}[Analysis of the maximal efficiency]

Fix $ \gamma = 1$, and let $\mathcal{S} > 0, 1 \le k < d$.
\tomtm{Consider the efficiency $\eta=\eta(N)$ as a function of continuous $N \in \mathbb{R}_+$.
The following hold.}

\begin{itemize}
    \item The maximal efficiency $\eta_{\max} = \underset{N \ge 0}{\max} \ \eta(N)$ increases as a function of $\mathcal{S} > 0$.

    \item For fixed $r \coloneqq \frac{d}{k}$ and $k \gg \max\{1,\mathcal{S}\}$, the maximizer $N_{\max} = \underset{N \ge 0}{\operatorname{arg\,max}} \ \eta(N)$ decreases with $\mathcal{S} > 0$ in both regimes $\mathcal{S} \ll 1, \mathcal{S} \gg 1$. In addition, in the regime $\mathcal{S} \ll 1$ the following approximation holds:
    \begin{equation}
        N_{\max} \approx \frac{k}{2\mathcal{S}}\cdot \frac{r^{\frac{2}{3}}\left(r^{\frac{1}{3}} - 1\right)}{r^{\frac{2}{3}} - 1}.
    \end{equation}
    Finally, in the regime $\mathcal{S} \gg 1$, the following approximation holds:
    \begin{equation}
        N_{\max} \approx \frac{k}{2\mathcal{S}}\cdot \sqrt{r}.
    \end{equation}
\end{itemize}

\end{theorem*}

\begin{proof}
Fix $ \gamma = 1$, and take
\begin{equation*}
    \mathcal{S} > 0, 1 \le k < d.
\end{equation*}
Let us now define the following parametric function:
\begin{equation}
\begin{aligned}
    f_{q}(x, \mathcal{S}) &\coloneqq \frac{\sqrt{\mathcal{S}}}{\sqrt{\left(\frac{q}{2\mathcal{S}} + 1\right)\cdot \frac{1}{x} + \frac{q}{4\mathcal{S}}\cdot \frac{1}{x^2} + 1}}
    \\
    &= \frac{\sqrt{\mathcal{S}}}{\sqrt{\frac{q + 2\mathcal{S}}{2\mathcal{S}x} + \frac{q}{4\mathcal{S}x^2} + 1}}
    \\
    &= \frac{\sqrt{\mathcal{S}}}{\sqrt{\frac{2x(q+2\mathcal{S}) + q + 4\mathcal{S}^2x}{4\mathcal{S}x^2}}}
    \\
    &= \frac{2\mathcal{S}x}{\sqrt{2qx + q + 4\mathcal{S}x + 4\mathcal{S}x^2}} 
    \\
    &= \frac{2\mathcal{S}x}{\sqrt{(2x + 1)q + 4x(x+1)\mathcal{S}}}
    \\
    &= \frac{2\mathcal{S}x}{\sqrt{D_q(x,\mathcal{S})}}
\end{aligned}
\label{f_q definition - max eta}
\end{equation}
where we denoted
\begin{equation}
\begin{aligned}
    D_q(x,\mathcal{S}) &= (2x + 1)q + 4x(x+1)\mathcal{S}
    \\
    &= 4\mathcal{S}x^2 + 2\left(2\mathcal{S} + q\right)x + q.
\end{aligned}
\label{D(x,S) definition - max eta}
\end{equation}
From Definition \ref{eq:theoretical efficiency} and \eqref{approx prob x}, it follows that:
\begin{equation}
\begin{aligned}
    \eta &= 100\cdot \left(1 - \frac{\hat{p}_{\vct{z}}(\mathrm{error})}{\hat{p}_{\vct{x}}(\mathrm{error})}\right) = 100\cdot \left(1 - \frac{\hat{p}(\mathcal{S},{N},1,k)}{\hat{p}(\mathcal{S},{N},1,d)}\right)
    \\
    &= 100\cdot \left(1 - \frac{2\cdot \mathcal{Q}\left(f_{k}({N})\right)}{2\cdot \mathcal{Q}\left(f_{d}({N})\right)}\right)
    \\
    &= 100\cdot \left(1 - \frac{\mathcal{Q}\left(f_{k}({N})\right)}{\mathcal{Q}\left(f_{d}({N})\right)}\right)
    \\
    &= 100\cdot h({N}, {\mathcal{S}})
\end{aligned}
\label{eta first expansion as 100*h(N) - max eta}
\end{equation}
where we defined the following function:
\begin{equation}
\label{eq:h(x)_def_b - max eta}
    h(x,\mathcal{S}) \coloneqq 1 - \frac{\mathcal{Q}\left(f_{k}({x}, {\mathcal{S}})\right)}{\mathcal{Q}\left(f_{d}({x}, \mathcal{S})\right)}.
\end{equation}
Hence, our task reduces to proving that the following function is increasing:
\begin{equation}
    V(\mathcal{S}) \coloneqq \max_{x > 0} h(x, \mathcal{S}) = h\left(x^*(\mathcal{S}), \mathcal{S}\right)
\label{V(S) definition - max eta}
\end{equation}
where we denoted
\begin{equation}
    x^*(\mathcal{S}) = \mathop{\mathrm{arg\,max}}_{x > 0} \  h(x, \mathcal{S}).
\label{x* definition - max eta}
\end{equation}
We note that the proof holds for each stationary point, and in particular for a maximizer that achieves the maximum value of the function $h(x,\mathcal{S})$. In addition, using the first-order condition:
\begin{equation}
    \frac{\partial h}{\partial x}\left(x^*(\mathcal{S}), \mathcal{S}\right) = 0.
\label{first order condition x - max eta}
\end{equation}
We now aim to prove that $V'(\mathcal{S}) > 0$, where $V(\mathcal{S})$ is defined in \eqref{V(S) definition - max eta}. From the chain rule, we have:
\begin{equation}
    V'(\mathcal{S}) = \frac{\partial h}{\partial \mathcal{S}} \left(x^*(\mathcal{S}), \mathcal{S}\right) + \frac{\partial h}{\partial x}\left(x^*(\mathcal{S}),\mathcal{S}\right)\cdot \frac{\partial x^*}{\partial \mathcal{S}} = \frac{\partial h}{\partial \mathcal{S}} \left(x^*(\mathcal{S}), \mathcal{S}\right)
\label{envelope theorem - V'(S), max eta}
\end{equation}
where we used \eqref{first order condition x - max eta}. Let us compute the partial derivative of $h$ with respect to $\mathcal{S}$:
\begin{equation}
\begin{aligned}
    \frac{\partial h}{\partial \mathcal{S}} &= - \frac{\partial}{\partial \mathcal{S}} \left(\frac{\mathcal{Q}\left(f_k(x, \mathcal{S})\right)}{\mathcal{Q}\left(f_d(x, \mathcal{S})\right)}\right)
    \\
    &= -\frac{\frac{\partial}{\partial \mathcal{S}}\mathcal{Q}\left(f_k(x, \mathcal{S})\right)\cdot \mathcal{Q}\left(f_d(x,\mathcal{S})\right) - \frac{\partial}{\partial \mathcal{S}}\mathcal{Q}\left(f_d(x, \mathcal{S})\right)\cdot \mathcal{Q}\left(f_k(x,\mathcal{S})\right)}{\mathcal{Q}^2\left(f_d(x,\mathcal{S})\right)}
    \\
    &= \frac{\frac{\partial}{\partial \mathcal{S}}\mathcal{Q}\left(f_d(x, \mathcal{S})\right)\cdot \mathcal{Q}\left(f_k(x,\mathcal{S})\right) - \frac{\partial}{\partial \mathcal{S}}\mathcal{Q}\left(f_k(x, \mathcal{S})\right)\cdot \mathcal{Q}\left(f_d(x,\mathcal{S})\right)}{\mathcal{Q}^2\left(f_d(x,\mathcal{S})\right)}
    \\
    &= \frac{\mathcal{Q}'\left(f_d(x,\mathcal{S})\right)\cdot \frac{\partial f_d}{\partial \mathcal{S}} \cdot \mathcal{Q}\left(f_k(x,\mathcal{S})\right) - \mathcal{Q}'\left(f_k(x,\mathcal{S})\right)\cdot \frac{\partial f_k}{\partial \mathcal{S}} \cdot \mathcal{Q}\left(f_d(x,\mathcal{S})\right)}{\mathcal{Q}^2\left(f_d(x,\mathcal{S})\right)}
    \\
    &= \frac{1}{\sqrt{2\pi}}\cdot \frac{\exp\left(-\frac{1}{2}f_k^2(x, \mathcal{S})\right)\cdot \frac{\partial f_k}{\partial \mathcal{S}}\cdot \mathcal{Q}\left(f_d(x, \mathcal{S})\right) - \exp\left(-\frac{1}{2}f_d^2(x, \mathcal{S})\right)\cdot \frac{\partial f_d}{\partial \mathcal{S}}\cdot \mathcal{Q}\left(f_k(x, \mathcal{S})\right)}{\mathcal{Q}^2\left(f_d(x,\mathcal{S})\right)}
\end{aligned}
\label{partial h partial S - first comp, max eta}
\end{equation}
where we used the identity $\mathcal{Q}'(x) = -\frac{1}{\sqrt{2\pi}}\cdot \exp\left(-\frac{1}{2}x^2\right)$. It now follows immediately from \eqref{envelope theorem - V'(S), max eta}, \eqref{partial h partial S - first comp, max eta}, that proving $V'(\mathcal{S}) > 0$ is equivalent to proving the following inequality:
\begin{equation}
    \exp\left(-\frac{1}{2}f_k^2(x^*, \mathcal{S})\right) \frac{\partial f_k}{\partial \mathcal{S}}(x^*, \mathcal{S})  \mathcal{Q}\left(f_d(x^*, \mathcal{S})\right) > \exp\left(-\frac{1}{2}f_d^2(x^*, \mathcal{S})\right) \frac{\partial f_d}{\partial \mathcal{S}}(x^*, \mathcal{S}) \mathcal{Q}\left(f_k(x^*, \mathcal{S})\right).
\label{first inequality to prove V'(S)>0 - max eta}
\end{equation}
We now turn to the first order condition for $x^*(\mathcal{S})$ in \eqref{first order condition x - max eta}, and thus equate the partial derivative of $h$ with respect to $x$ to zero. Similarly to \eqref{partial h partial S - first comp, max eta}, one can prove that for all $x > 0$ we have:
\begin{equation}
\begin{aligned}
    \frac{\partial h}{\partial x} &= -\frac{\partial}{\partial x}\left(\frac{\mathcal{Q}\left(f_k(x)\right)}{\mathcal{Q}\left(f_d(x)\right)}\right) 
    \\
    &= \frac{1}{\sqrt{2\pi}}\cdot \frac{\exp\left(-\frac{1}{2}f_k^2(x, \mathcal{S})\right)\cdot \frac{\partial f_k}{\partial x}\cdot \mathcal{Q}\left(f_d(x, \mathcal{S})\right) - \exp\left(-\frac{1}{2}f_d^2(x, \mathcal{S})\right)\cdot \frac{\partial f_d}{\partial x}\cdot \mathcal{Q}\left(f_k(x, \mathcal{S})\right)}{\mathcal{Q}^2\left(f_d(x,\mathcal{S})\right)}.
\end{aligned}
\label{partial h partial x formula - eta max}
\end{equation}
That is, equating $\frac{\partial h}{\partial x} = 0$ yields the following equation for $x^*(\mathcal{S})$:
\begin{equation}
    \exp\left(-\frac{1}{2}f_k^2(x^*, \mathcal{S})\right) \frac{\partial f_k}{\partial x}(x^*, \mathcal{S})  \mathcal{Q}\left(f_d(x^*, \mathcal{S})\right) = \exp\left(-\frac{1}{2}f_d^2(x^*, \mathcal{S})\right) \frac{\partial f_d}{\partial x}(x^*, \mathcal{S}) \mathcal{Q}\left(f_k(x^*, \mathcal{S})\right).
\label{equation for x* h'(x*) = 0 - max eta}
\end{equation}
We now divide both sides of the inequality in \eqref{first inequality to prove V'(S)>0 - max eta} by the (positive) value we have in the latter equality, in order to get the following simplified inequality:
\begin{equation}
    \frac{\frac{\partial f_k}{\partial \mathcal{S}}(x^*, \mathcal{S})}{\frac{\partial f_k}{\partial x}(x^*, \mathcal{S})} > \frac{\frac{\partial f_d}{\partial \mathcal{S}}(x^*, \mathcal{S})}{\frac{\partial f_d}{\partial x}(x^*, \mathcal{S})}.
\label{ineq to prove - second phase, max eta}
\end{equation}
We indeed divided by a positive amount, because $\exp(\cdot) > 0, \mathcal{Q}(\cdot ) > 0$, and $\frac{\partial f_q}{\partial x} > 0$: for all $x > 0$, from \eqref{f_q definition - max eta}, \eqref{D(x,S) definition - max eta}, we have:
\begin{equation}
\begin{aligned}
    \frac{\partial f_q}{\partial x} &= \frac{\partial}{\partial x}\left(\frac{2\mathcal{S}x}{\sqrt{D_q(x,\mathcal{S})}}\right)
    \\
    &= 2\mathcal{S}\cdot \frac{\partial}{\partial x}\left(\frac{x}{\sqrt{D_q(x,\mathcal{S})}}\right)
    \\
    &= 2\mathcal{S}\cdot \left(\frac{\sqrt{D_q(x,\mathcal{S})} - x\cdot \frac{\partial}{\partial x} \left(\sqrt{D_q(x,\mathcal{S})}\right)}{D_q(x,\mathcal{S})}\right)
    \\
    &= 2\mathcal{S}\cdot \left(\frac{\sqrt{D_q(x,\mathcal{S})} - x\cdot \frac{\partial D_q(x,\mathcal{S}) / \partial x}{2\sqrt{D_q(x,\mathcal{S})}}}{D_q(x,\mathcal{S})}\right)
    \\
    &= 2\mathcal{S}\cdot \left(\frac{2\cdot D_q(x,\mathcal{S}) - x\cdot \frac{\partial D_q(x,\mathcal{S})}{\partial x}}{2\cdot \left(D_q(x,\mathcal{S})\right)^{3/2}}\right)
    \\
    &= \frac{\mathcal{S}}{\left(D_q(x,\mathcal{S})\right)^{3/2}}\cdot \left(2 \left(4\mathcal{S}x^2 + 2\left(2\mathcal{S} + q\right)x + q\right) - x \left(8\mathcal{S}x + 2\left(2\mathcal{S} + q\right)\right)\right)
    \\
    &= \frac{\mathcal{S}}{\left(D_q(x,\mathcal{S})\right)^{3/2}}\cdot \left(2\left(2\mathcal{S} + q\right)x + 2q\right)
    \\
    &= \frac{2\mathcal{S}}{\left(D_q(x,\mathcal{S})\right)^{3/2}}\cdot \left(\left(2\mathcal{S} + q\right)x + q\right) > 0
\end{aligned}
\label{partial f_q partial x general formula - max eta}
\end{equation}
where we used the definition of $D_q(x,\mathcal{S})$ in \eqref{D(x,S) definition - max eta}. We now compute the partial derivative of $f_q$ with respect to $\mathcal{S}$:
\begin{equation}
\begin{aligned}
    \frac{\partial f_q}{\partial \mathcal{S}} &= \frac{\partial}{\partial \mathcal{S}}\left(\frac{2\mathcal{S}x}{\sqrt{D_q(x,\mathcal{S})}}\right)
    \\
    &= 2x\cdot \frac{\partial}{\partial \mathcal{S}}\left(\frac{\mathcal{S}}{\sqrt{D_q(x,\mathcal{S})}}\right)
    \\
    &= 2x\cdot \left(\frac{\sqrt{D_q(x,\mathcal{S})} - \mathcal{S}\cdot \frac{\partial}{\partial \mathcal{S}} \left(\sqrt{D_q(x,\mathcal{S})}\right)}{D_q(x,\mathcal{S})}\right)
    \\
    &= 2x\cdot \left(\frac{\sqrt{D_q(x,\mathcal{S})} - \mathcal{S}\cdot \frac{\partial D_q(x,\mathcal{S}) / \partial \mathcal{S}}{2\sqrt{D_q(x,\mathcal{S})}}}{D_q(x,\mathcal{S})}\right)
    \\
    &= 2x\cdot \left(\frac{2\cdot D_q(x,\mathcal{S}) - \mathcal{S}\cdot \frac{\partial D_q(x,\mathcal{S})}{\partial \mathcal{S}}}{2\cdot \left(D_q(x,\mathcal{S})\right)^{3/2}}\right).
\end{aligned}
\label{partial f_q partial S first - max eta}
\end{equation}
We now use the definition of $D_q(x,\mathcal{S})$ from \eqref{D(x,S) definition - max eta}, and get
\begin{equation}
\begin{aligned}
    \frac{\partial f_q}{\partial \mathcal{S}} &= \frac{x}{\left(D_q(x,\mathcal{S})\right)^{3/2}}\cdot \left(2\left(4\mathcal{S}x^2 + 2\left(2\mathcal{S}+q\right)x + q\right) - \mathcal{S}\cdot 4x(x+1)\right)
    \\
    &= \frac{x}{\left(D_q(x,\mathcal{S})\right)^{3/2}}\cdot \left(4\mathcal{S}x^2 + 4\left(2\mathcal{S}+q\right)x - 4\mathcal{S}x + 2q\right)
    \\
    &= \frac{x}{\left(D_q(x,\mathcal{S})\right)^{3/2}}\cdot \left(4\mathcal{S}x^2 + 4\left(\mathcal{S}+q\right)x + 2q\right)
    \\
    &= \frac{2x}{\left(D_q(x,\mathcal{S})\right)^{3/2}}\cdot \left(2\mathcal{S}x^2 + 2\left(\mathcal{S}+q\right)x + q\right).
\end{aligned}
\label{partial f_q partial S final - max eta}
\end{equation}
Let us now define
\begin{equation}
\begin{aligned}
    R_q(x,\mathcal{S}) &\coloneqq \frac{\frac{\partial f_q}{\partial \mathcal{S}}(x, \mathcal{S})}{\frac{\partial f_q}{\partial x}(x, \mathcal{S})}
    \\
    &= \frac{\frac{2x}{\left(D_q(x,\mathcal{S})\right)^{3/2}}\cdot \left(2\mathcal{S}x^2 + 2\left(\mathcal{S}+q\right)x + q\right)}{\frac{2\mathcal{S}}{\left(D_q(x,\mathcal{S})\right)^{3/2}}\cdot \left(\left(2\mathcal{S} + q\right)x + q\right)}
    \\
    &= \frac{x}{\mathcal{S}}\cdot \left(\frac{2\mathcal{S}x^2 + 2\mathcal{S}x + 2x\cdot q + q}{2\mathcal{S}x + x\cdot q + q}\right)
    \\
    &= \frac{x}{\mathcal{S}}\cdot \left(\frac{2\mathcal{S}x\left(x + 1\right) + (2x + 1)\cdot q}{2\mathcal{S}x + (x+1)\cdot q}\right)
\end{aligned}
\label{R_q(x,S) definition - max eta}
\end{equation}
where we used the partial derivatives of $f_q$, computed in \eqref{partial f_q partial x general formula - max eta}, \eqref{partial f_q partial S final - max eta}. We remember that we need to prove \eqref{ineq to prove - second phase, max eta}, which is equivalent to $R_q(x,\mathcal{S})$ being a decreasing function in the argument $q$ (this is because $1 \le k < d$). Indeed, let us compute
\begin{equation}
\begin{aligned}
    \frac{\partial R_q(x,\mathcal{S})}{\partial q} &= \frac{x}{\mathcal{S}}\cdot \frac{\partial}{\partial q}\left(\frac{2\mathcal{S}x(x+1) +(2x+1)\cdot q}{2\mathcal{S}x + (x+1)\cdot q}\right)
    \\
    &= \frac{x}{\mathcal{S}}\cdot \frac{(2x+1)\cdot \left(2\mathcal{S}x + (x+1)\cdot q\right) - (x+1)\cdot \left(2\mathcal{S}x(x+1) + (2x+1)\cdot q\right)}{\left(2\mathcal{S}x + (x+1)\cdot q\right)^2}
    \\
    &= \frac{x}{\mathcal{S}}\cdot \frac{2x(2x+1)\cdot \mathcal{S} - 2x(x+1)^2\cdot \mathcal{S}}{\left(2\mathcal{S}x + (x+1)\cdot q\right)^2}
    \\
    &= \frac{2x^2\cdot \left(2x + 1 - (x+1)^2\right)}{\left(2\mathcal{S}x+(x+1)\cdot q\right)^2}
    \\
    &= -\frac{2x^4}{\left(2\mathcal{S}x + (x+1)\cdot q\right)^2}.
\end{aligned}
\label{partial R_q partial q final - max eta}
\end{equation}

That is,
\begin{equation*}
    \frac{\partial R_q(x,\mathcal{S})}{\partial q} = -\frac{2x^4}{\left(2\mathcal{S}x+(x+1)\cdot q\right)^2} < 0
\end{equation*}
which proves \eqref{ineq to prove - second phase, max eta}. We argued that this is equivalent to \eqref{first inequality to prove V'(S)>0 - max eta}. As we proved in \eqref{envelope theorem - V'(S), max eta}, this inequality is equivalent to proving $V'(\mathcal{S}) > 0$. Finally, the result is straightforward because $\eta = 100\cdot h(N,\mathcal{S})$. This proves the first part of the theorem. 
\\
\\
We will now prove the second part of the theorem. Let us fix $d > k \gg \max\{1,\mathcal{S}\}$, and $r \coloneqq \frac{d}{k} > 1$. We prove that the maximizer $x^*(\mathcal{S})$, defined in \eqref{x* definition - max eta}, decreases as a function of $\mathcal{S}$. Let us first define the following rescaled $x$ value:
\begin{equation}
    x(t) \coloneqq \frac{k}{\mathcal{S}}\cdot t.
\label{rescaled x for dx*/dS < 0 - max eta}
\end{equation}
Thus,
\begin{equation}
    x^*(\mathcal{S}) = \frac{k}{\mathcal{S}}\cdot \underset{t > 0}{\operatorname{arg\,max}} \ h\left(x(t),\mathcal{S}\right).
\label{x* as a function of t* - first formula, max eta}
\end{equation}
From \eqref{f_q definition - max eta}, it follows that:
\begin{equation}
\begin{aligned}
    f_q\left(x(t),\mathcal{S}\right) &= \frac{2\mathcal{S}\cdot x(t)}{\sqrt{\left(2x(t) + 1\right)q + 4x(t) \left(x(t) + 1\right)\mathcal{S}}}
    \\
    &\sim \frac{2\mathcal{S}\cdot x(t)}{\sqrt{2x(t)\cdot q + 4x^2(t)\cdot \mathcal{S}}}
    \\
    &= \frac{2\mathcal{S}}{\sqrt{\frac{2q}{x(t)} + 4\mathcal{S}}}
    \\
    &= \frac{2\mathcal{S}}{\sqrt{2q\cdot \frac{\mathcal{S}}{k\cdot t} + 4\mathcal{S}}}
    \\
    &= \frac{\sqrt{2}\sqrt{\mathcal{S}}\cdot \sqrt{kt}}{\sqrt{2q + 4kt}}
    \\
    &= \sqrt{\frac{2\mathcal{S} k\cdot t}{q + 2k\cdot t}}
    \\
    &= \sqrt{\frac{2\mathcal{S}t}{\frac{q}{k} + 2t}}
\end{aligned}
\label{f_q approx for q>>1 - max eta}
\end{equation}
where we assumed $x(t) \gg 1$ which follows from $q \gg 1, \mathcal{S} \ll q, t = O(1)$. We explain the assumption $t = O(1)$ in a moment. Now, from \eqref{eq:h(x)_def_b - max eta}, we have:
\begin{equation}
\begin{aligned}
    h\left(x(t),\mathcal{S}\right) = 1-\frac{\mathcal{Q}\left(f_k\left(x(t),\mathcal{S}\right)\right)}{\mathcal{Q}\left(f_d\left(x(t),\mathcal{S}\right)\right)} &\sim 1 - \frac{\mathcal{Q}\left(\sqrt{\frac{2\mathcal{S}t}{1 + 2t}}\right)}{\mathcal{Q}\left(\sqrt{\frac{2\mathcal{S}t}{r + 2t}}\right)}
\end{aligned}
\label{h(x,S) first approx for d>k>>1 - max eta}
\end{equation}
this motivates the assumption $t = O(1)$ we used in \eqref{f_q approx for q>>1 - max eta}: the maximizer
\begin{equation}
\begin{aligned}
    t^*(r, \mathcal{S}) = \underset{t > 0}{\operatorname{arg\,max}}\; h\left(x(t) ,\mathcal{S}\right) &\sim \underset{t > 0}{\operatorname{arg\,min}}\; \left(\frac{\mathcal{Q}\left(\sqrt{\frac{2\mathcal{S}t}{1 + 2t}}\right)}{\mathcal{Q}\left(\sqrt{\frac{2\mathcal{S}t}{r + 2t}}\right)}\right)
    \\
    &= \underset{t > 0}{\operatorname{arg\,min}}\;g_{\mathcal{S},r}(t)
\end{aligned}
\label{t*(r,S) definition - max eta}
\end{equation}
is a function of $r = O(1)$ which is a constant and $\mathcal{S} \ll k$, and thus in the region of interest (close to the maximizer), $t = O(1)$ does not scale with $k$. The key insight here is that $t^*(r,\mathcal{S})$ doesn't depend on $k, d$. We also defined the following function:
\begin{equation}
    g_{\mathcal{S},r}(t) \coloneqq \frac{\mathcal{Q}\left(\sqrt{\frac{2\mathcal{S}t}{1 + 2t}}\right)}{\mathcal{Q}\left(\sqrt{\frac{2\mathcal{S}t}{r + 2t}}\right)}.
\label{g_(S,r)(t) definition - max eta}
\end{equation}
Finally, we would like to analyze the dependency of $t^*(r,\mathcal{S})$ on $\mathcal{S}$. We will prove that in both regimes $\mathcal{S} \ll 1, \mathcal{S} \gg 1$, we have that $t^*(r,\mathcal{S})$ does not depend on $\mathcal{S}$, and thus in both regimes, from \eqref{x* as a function of t* - first formula, max eta}, \eqref{t*(r,S) definition - max eta}, the maximizer
\begin{equation}
    x^*(\mathcal{S}) = \frac{k}{\mathcal{S}}\cdot t^*(r)
\label{formula x*(S)=k/St*(r) - max eta}
\end{equation}
is decreasing as a function of $\mathcal{S} > 0$.
\underline{For the regime $\mathcal{S} \ll 1$:} We use the approximation $\mathcal{Q}(x) \sim \frac{1}{2} - \frac{1}{\sqrt{2\pi}}\cdot x$ for $x \ll 1$ and \eqref{g_(S,r)(t) definition - max eta} to get
\begin{equation}
\begin{aligned}
    g_{\mathcal{S},r}(t) = \frac{\mathcal{Q}\left(\sqrt{\frac{2\mathcal{S}t}{1 + 2t}}\right)}{\mathcal{Q}\left(\sqrt{\frac{2\mathcal{S}t}{r + 2t}}\right)} &\sim  \frac{\frac{1}{2} - \frac{1}{\sqrt{2\pi}}\cdot \sqrt{\frac{2\mathcal{S}t}{1 + 2t}}}{\frac{1}{2} - \frac{1}{\sqrt{2\pi}}\cdot \sqrt{\frac{2\mathcal{S}t}{r + 2t}}}
    \\
    & = \frac{1 - \frac{2}{\sqrt{2\pi}}\cdot \sqrt{\frac{2\mathcal{S}t}{1 + 2t}}}{1 - \frac{2}{\sqrt{2\pi}}\cdot \sqrt{\frac{2\mathcal{S}t}{r + 2t}}}
    \\
    &\sim \left(1 - \frac{2}{\sqrt{2\pi}}\sqrt{\frac{2\mathcal{S}t}{1 + 2t}}\right)\left(1 + \frac{2}{\sqrt{2\pi}}\sqrt{\frac{2\mathcal{S}t}{r + 2t}}\right)
    \\
    &\sim \left(1 - \frac{2}{\sqrt{2\pi}}\sqrt{\frac{2\mathcal{S}t}{1 + 2t}} + \frac{2}{\sqrt{2\pi}}\sqrt{\frac{2\mathcal{S}t}{r + 2t}}\right)
    \\
    &= 1 - \sqrt{\frac{2}{\pi}}\cdot \left(\sqrt{\frac{2\mathcal{S}t}{1 + 2t}} - \sqrt{\frac{2\mathcal{S}t}{r + 2t}}\right)
    \\
    &= 1 - 2\sqrt{\frac{\mathcal{S}}{\pi}}\cdot \left(\sqrt{\frac{t}{1+2t}} - \sqrt{\frac{t}{r + 2t}}\right).
\end{aligned}
\label{g_(S,r)(t) approx for S<<1 - max eta}
\end{equation}
This approximation is motivated from the fact that the argument of the $\mathcal{Q}$ function in both the numerator and the denominator is at most $\sqrt{\mathcal{S}}\ll 1$. Thus, from \eqref{t*(r,S) definition - max eta}, we have:
\begin{equation}
\begin{aligned}
    t^*(r,\mathcal{S}) = \underset{t > 0}{\operatorname{arg\,min}}\;\left(\sqrt{\frac{t}{1+2t}} - \sqrt{\frac{t}{r + 2t}} \right)
\end{aligned}
\label{t*(r,S) first formula for S<<1 - max eta}
\end{equation}
is independent of $\mathcal{S}$. We will also calculate the minimizer. Let us define 
\begin{equation}
    \Psi_1(t) = \sqrt{\frac{t}{1+2t}} - \sqrt{\frac{t}{r + 2t}}.
\end{equation}
We now equate the derivative of $\Psi_1(t)$ to 0:
\begin{equation}
\begin{aligned}
    \frac{d \Psi_1}{dt} &= \frac{d}{dt}\left(\sqrt{\frac{t}{1+2t}}\right) - \frac{d}{dt}\left(\sqrt{\frac{t}{r+2t}}\right)
    \\
    &= \frac{1}{2}\left(\frac{1}{\sqrt{t}\cdot \left(1 + 2t\right)^{\frac{3}{2}}} - \frac{r}{\sqrt{t}\cdot (r + 2t)^{\frac{3}{2}}}\right)
    \\
    &= \frac{1}{2\sqrt t}\cdot \left(\frac{1}{\left(1 + 2t\right)^{\frac{3}{2}}} - \frac{r}{\left(r + 2t\right)^{\frac{3}{2}}} \right)
    \\
    &= \frac{1}{2\sqrt t}\cdot \frac{\left(r + 2t\right)^{\frac{3}{2}} - r\left(1 + 2t\right)^{\frac{3}{2}}}{\left(1+2t\right)^{\frac{3}{2}}\cdot \left(r + 2t\right)^{\frac{3}{2}}}
\end{aligned}
\label{d Psi_1/dt formula}
\end{equation}
where we used the following formula:
\begin{align*}
    \frac{d}{dt}\left(\sqrt{\frac{t}{a + 2t}}\right) &= \frac{\frac{d}{dt}\left(\frac{t}{a + 2t}\right)}{2\cdot \sqrt{\frac{t}{a + 2t}}}
    \\
    &= \frac{\frac{a}{(a + 2t)^2}}{2\cdot \sqrt{\frac{t}{a + 2t}}}
    \\
    &= \frac{a}{(a+2t)^2}\cdot \frac{\sqrt{a+2t}}{2\cdot \sqrt{t}}
    \\
    &= \frac{a}{2}\cdot \frac{1}{\sqrt{t}\cdot (a+2t)^{\frac{3}{2}}}
\end{align*}
Finally, from \eqref{d Psi_1/dt formula}, we have:
\begin{gather*}
    \left(r + 2t^*\right)^{\frac{3}{2}} = r\left(1 + 2t^*\right)^{\frac{3}{2}}
    \\
    r + 2t^* = r^{\frac{2}{3}}\cdot \left(1 + 2t^*\right)
    \\
    2t^*\cdot \left(1 - r^{\frac{2}{3}}\right) = r^{\frac{2}{3}}\left(1 - r^{\frac{1}{3}}\right)
    \\
    t^* = \frac{r^{\frac{2}{3}}\left(1 - r^{\frac{1}{3}}\right)}{2\left(1 - r^{\frac{2}{3}}\right)}
\end{gather*}
That is, the maximizer is unique, and from \eqref{formula x*(S)=k/St*(r) - max eta}, in the regime $\mathcal{S} \ll 1$ we have:
\begin{equation}
    x^*(\mathcal{S}) \sim \frac{1}{2}k\cdot \frac{r^{\frac{2}{3}}\left(r^{\frac{1}{3}} - 1\right)}{r^{\frac{2}{3}} - 1}\cdot \frac{1}{\mathcal{S}}
\end{equation}
which is strictly decreasing as a function of $\mathcal{S}$.

\underline{For the regime $\mathcal{S} \gg 1$:} We use the approximation $\mathcal{Q}(x) \sim \frac{1}{\sqrt{2\pi}x}\cdot \exp\left(-\frac{x^2}{2}\right)$ and \eqref{g_(S,r)(t) definition - max eta} to get
\begin{equation}
\begin{aligned}
    g_{\mathcal{S},r}(t) = \frac{\mathcal{Q}\left(\sqrt{\frac{2\mathcal{S}t}{1 + 2t}}\right)}{\mathcal{Q}\left(\sqrt{\frac{2\mathcal{S}t}{r + 2t}}\right)} &\sim  \frac{\frac{1}{\sqrt{2\pi}}\cdot \sqrt{\frac{1 + 2t}{2\mathcal{S}t}}\cdot \exp\left(-\frac{1}{2}\cdot \frac{2\mathcal{S}t}{1 + 2t}\right)}{\frac{1}{\sqrt{2\pi}}\cdot \sqrt{\frac{r + 2t}{2\mathcal{S}t}}\cdot \exp\left(-\frac{1}{2}\cdot \frac{2\mathcal{S}t}{r + 2t}\right)}
    \\
    &= \sqrt{\frac{1 + 2t}{r + 2t}}\cdot \frac{\exp\left(- \frac{\mathcal{S} t}{1 + 2t}\right)}{\exp\left(- \frac{\mathcal{S} t}{r + 2t}\right)}
    \\
    &= \sqrt{\frac{1 + 2t}{r + 2t}}\cdot \exp\left(-\mathcal{S}t\cdot \left(\frac{1}{1 + 2t} - \frac{1}{r + 2t}\right)\right)
\end{aligned}
\label{g_(S,r)(t) approx for S>>1 - max eta}
\end{equation}
minimizing $g_{\mathcal{S},r}(t)$ is equivalent to minimizing $\ln\left(g_{\mathcal{S},r}(t)\right)$:
\begin{equation}
    M_{\mathcal{S},r}(t)\coloneqq \ln\left(g_{\mathcal{S},r}(t)\right) = \frac{1}{2}\ln\left(1 + 2t\right) - \frac{1}{2}\ln\left(r + 2t\right) - \mathcal{S}t\cdot \left(\frac{1}{1 + 2t} - \frac{1}{r + 2t}\right)
\label{M_(S,r)(t) = ln(g_(S,r))(t) formula for S>>1 - max eta}
\end{equation}
Let us equate the derivative of $M_{\mathcal{S},r}(t)$ to 0:
\begin{gather*}
    \frac{1}{1 + 2t} - \frac{1}{r + 2t} - \mathcal{S}\cdot \left(\frac{d}{dt}\left(\frac{t}{1 + 2t}\right) - \frac{d}{dt}\left(\frac{t}{r + 2t}\right)\right) = 0
    \\
    \frac{1}{1 + 2t} - \frac{1}{r + 2t} - \mathcal{S}\cdot \left(\frac{1}{(1 + 2t)^2} - \frac{r}{(r+2t)^2}\right) = 0
    \\
    \frac{1}{1 + 2t} - \frac{1}{r + 2t} = \frac{\mathcal{S}}{(1+2t)^2} - \frac{\mathcal{S}r}{(r+2t)^2}
    \\
    (1+2t)(r+2t)^2 - (1+2t)^2(r+2t) = \mathcal{S}(r + 2t)^2 -\mathcal{S}r(1+2t)^2
    \\
    (1 + 2t)(r + 2t)\cdot (r-1) = \mathcal{S}\cdot \left(r^2 + 4rt + 4t^2\right) - \mathcal{S}r\cdot \left(1 + 4t+4t^2\right)
    \\
    (r-1)\cdot (r + 2\left(r + 1\right)t + 4t^2) = \mathcal{S}\cdot \left(r^2 + 4rt + 4t^2\right) - \mathcal{S}r\cdot \left(1 + 4t+4t^2\right)
    \\
    4(r-1)\cdot t^2 + 2(r^2-1)t + r(r-1) = \left(4\mathcal{S} - 4\mathcal{S}r\right)t^2 + \mathcal{S}r^2 - \mathcal{S}r
    \\
    4(r-1)\cdot t^2 + 2(r-1)(r+1)t + r(r-1) = 4\mathcal{S}\left(1 - r\right)t^2 + \mathcal{S}r\left(r - 1\right)
    \\
    4t^2 + 2(r+1)t + r = -4\mathcal{S}t^2 + \mathcal{S}r
    \\
    4\left(1 + \mathcal{S}\right)t^2 + 2(r+1)\cdot t + r\left(1 - \mathcal{S}\right) = 0
\end{gather*}
Finally, we take the positive root (because $t^*(r,\mathcal{S}) > 0$) and get:
\begin{equation}
\begin{aligned}
    t^*(r,\mathcal{S}) &= \frac{-2(r+1) + \sqrt{4(r+1)^2 - 16r(1 + \mathcal{S})(1-\mathcal{S})}}{8(1+\mathcal{S})}
    \\
    &= \frac{-(r+1) + \sqrt{(r+1)^2 + 4r\left(\mathcal{S}^2 - 1\right)}}{4(\mathcal{S} + 1)}
\end{aligned}
\label{t*(r,S) formula for S>>1 - max eta}
\end{equation}
we note that we got a single solution and that the maximizer is unique, and $\mathcal{S} \gg 1$ and thus the formula is well-defined. We note that
\begin{equation*}
    \lim_{\mathcal{S} \rightarrow \infty} t^*(r,\mathcal{S}) = \frac{\sqrt{r}}{2}
\end{equation*}
and thus, from \eqref{formula x*(S)=k/St*(r) - max eta}, in the $\mathcal{S} \gg 1$ regime, we have
\begin{equation}
    x^*(\mathcal{S}) \sim \frac{k}{\mathcal{S}}\cdot \frac{-(r+1) + \sqrt{(r+1)^2 + 4r\left(\mathcal{S}^2 - 1\right)}}{4(\mathcal{S} + 1)}
\label{x*(S) final formula for S>>1 - max eta}
\end{equation}
and as $\mathcal{S} \rightarrow \infty$ we have $x^*(\mathcal{S}) \sim \frac{k}{\mathcal{S}}\cdot \frac{\sqrt{r}}{2}$, and thus in the $\mathcal{S} \gg 1$ regime we have $x^*(\mathcal{S}) \sim \frac{1}{\mathcal{S}}$ is a decreasing function of $\mathcal{S}$.

\end{proof}

\newpage

{\color{black}
\section{Additional theoretical results}
\label{app:approximation_delta}

We have established an approximation of the efficiency of the processing $\eta$ for $N \gg 1$. We now do the same for the difference
\begin{equation*}
    \Delta \coloneqq \hat{p}_{\vct{x}}(\mathrm{error}) - \hat{p}_{\vct{z}}(\mathrm{error}).
\end{equation*}
This allows us to gain insight into the different factors that affect the difference $\Delta$ between the probability of error that is caused by the processing.

\begin{theorem}[Analysis of the asymptotic difference]\label{thm:thm10} Let $\mathcal{S} > 0, \ 1 \le k < d, \ 0 < \gamma \le 1$.
Denote by $N_T = \left(1 + \gamma\right)N$ the total number of training samples. 
With approximation accuracy $\mathcal{O}(1/N_T^2)$, we have
\begin{align}
    \Delta \approx \frac{1}{4 \sqrt{2\pi}}\cdot \frac{\exp\left(-\frac{\mathcal{S}}{2}\right)}{\sqrt{\mathcal{S}}}\cdot \left(3 + 2\gamma + \frac{1}{\gamma}\right)\cdot (d-k)\cdot \frac{1}{N_T}.
\label{difference approx}
\end{align}
In particular, for $N_T \gg 1$: The difference increases when $d-k$ increases or $\gamma$ decreases within $0 < \gamma \le 1/\sqrt{2}$; The difference decreases when $\mathcal{S}$ increases or $N_T$ increases.
\end{theorem}

\begin{proof} Let us take some $N_T \in \mathbb{N}$ and
\begin{equation*}
    \mathcal{S} > 0, 1 \le k < d, 0 < \gamma \le 1.
\end{equation*}
We have $N_T = N + \gamma N = (1 + \gamma)N$ the total number of training samples, and thus $N = \frac{N_T}{1+\gamma}$. Let us define the following parametric function:
\begin{equation}
    f_{s,a,q}(x) \coloneqq \frac{\sqrt{\mathcal{S}} +s\cdot  \frac{(1 - \gamma)\cdot q}{4\gamma\cdot \sqrt{\mathcal{S}}}\cdot \frac{1}{x}}{\sqrt{\left(\frac{(1+\gamma)\cdot q}{4\gamma\cdot \mathcal{S}} + \frac{1}{a}\right)\cdot \frac{1}{x} + \frac{(1 + \gamma^2)\cdot q}{8\gamma^2\cdot \mathcal{S}}\cdot \frac{1}{x^2} + 1}} = \frac{\sqrt{\mathcal{S}} + \frac{B}{x}}{\sqrt{\frac{C}{x} + \frac{D}{x^2} + 1}}
\label{f_s,a,q thm9}
\end{equation}
where the parameters $B,C,D$ are:
\begin{equation}
\begin{cases}
    B = s\cdot \frac{(1-\gamma)\cdot q}{4\gamma\cdot \sqrt{\mathcal{S}}}
    \\
    C = \frac{(1+\gamma)\cdot q}{4\gamma\cdot \mathcal{S}} + \frac{1}{a}
    \\
    D = \frac{(1+\gamma^2)\cdot q}{8\gamma^2\cdot \mathcal{S}}
\end{cases}
\label{B,C,D parameters thm9}
\end{equation}
From the Definition $\Delta \coloneqq \hat{p}_{\vct{x}}(\textrm{error}) - \hat{p}_{\vct{z}}(\textrm{error})$ and \eqref{approx prob x}, it follows that:
\begin{equation}
\begin{aligned}
    \Delta &= \hat{p}_{\vct{x}}(\textrm{error}) - \hat{p}_{\vct{z}}(\textrm{error}) = \hat{p}(\mathcal{S},{N},\gamma,d) - \hat{p}(\mathcal{S},{N},\gamma,k)
    \\
    &= \left[\mathcal{Q}\left(f_{1,\gamma,d}({N})\right) + \mathcal{Q}\left(f_{-1,1,d}({N})\right)\right] - \left[\mathcal{Q}\left(f_{1,\gamma,k}({N})\right) + \mathcal{Q}\left(f_{-1,1,k}({N})\right)\right]
    \\
    &= h({N})
\end{aligned}
\label{delta first expansion as h(N) thm9}
\end{equation}
where we defined the following function:
\begin{equation}
\label{eq:h(x)_def_b thm9}
    h(x) \coloneqq \left[\mathcal{Q}\left(f_{1,\gamma,d}({x})\right) + \mathcal{Q}\left(f_{-1,1,d}({x})\right)\right] - \left[\mathcal{Q}\left(f_{1,\gamma,k}({x})\right) + \mathcal{Q}\left(f_{-1,1,k}({x})\right)\right].
\end{equation}
Now, let us define the following parametric function: 
\begin{equation}
    g_{s,a,q}(x) \coloneqq f_{s,a,q}\left(\frac{1}{x}\right) = \frac{\sqrt{\mathcal{S}} + B\cdot x}{\sqrt{D\cdot x^2+ C\cdot x+1}}
\label{g_s,a,q formula thm9}
\end{equation}
where we used \eqref{f_s,a,q thm9}, \eqref{B,C,D parameters thm9}. Let us also define:
\begin{equation}
    \ell(x) \coloneqq h\left(\frac{1}{x}\right) = \left[\mathcal{Q}\left(g_{1,\gamma,d}(x)\right) + \mathcal{Q}\left(g_{-1,1,d}(x)\right)\right] - \left[\mathcal{Q}\left(g_{1,\gamma,k}(x)\right) + \mathcal{Q}\left(g_{-1,1,k}(x)\right)\right]
\label{l(x) def thm9}
\end{equation}
where we used \eqref{g_s,a,q formula thm9}, \eqref{eq:h(x)_def_b thm9}. Thus, the first-order Taylor expansion of $\ell$:
\begin{equation*}
    \ell(x) = \ell(0) + \ell'(0)\cdot x + \mathcal{O}\left(x^2\right)
\end{equation*}
Where the approximation is exact for $x \ll 1$. Thus, the following is exact for $x \gg 1$:
\begin{equation}
    x \gg 1 \Rightarrow h(x) = \ell\left(\frac{1}{x}\right) =  \ell(0) + \frac{\ell'(0)}{x} + \mathcal{O}\left(\frac{1}{x^2}\right).
\label{h(x) approximation thm9}
\end{equation}
We have
\begin{equation}
    {N} = \frac{N_{T}}{1+\gamma} \gg 1\Leftrightarrow N_{T} \gg 1 + \gamma
\label{condition on N thm9}
\end{equation}
Thus,
\begin{equation}
    h\left({N}\right) = \ell(0) + \frac{\ell'(0)}{{N}} + \mathcal{O}\left(\frac{1}{N^2}\right).
\label{h(N tilde) approximation thm9}
\end{equation}
\underline{Let us first compute $\ell(0)$:}
\begin{equation}
    \ell(0) = \left[\mathcal{Q}\left(g_{1,\gamma,d}(0)\right)+ \mathcal{Q}\left(g_{-1,1,d}(0)\right)\right] - \left[\mathcal{Q}\left(g_{1,\gamma,k}(0)\right)+ \mathcal{Q}\left(g_{-1,1,k}(0)\right)\right]  = 2\cdot \mathcal{Q}(\sqrt{\mathcal{S}}) - 2\cdot \mathcal{Q}(\sqrt{\mathcal{S}}) = 0
\label{ell(0) thm9}
\end{equation}
\underline{Finally, we will compute $\ell'(0)$:} We first compute $g_{s,a,q}'(0)$. From \eqref{g_s,a,q formula thm9} it follows that:
\begin{align*}
    g'_{s,a,q}(x) &= \frac{B\cdot \sqrt{D\cdot x^2 + C\cdot x + 1} - \frac{2D\cdot x + C}{2\cdot \sqrt{D\cdot x^2+C\cdot x + 1}}\cdot \left(\sqrt{\mathcal{S}} + B\cdot x\right)}{D\cdot x^2+C\cdot x+1}
    \\
    &= \frac{2B\cdot \left(D\cdot x^2+C\cdot x+ 1\right) - \left(2D\cdot x + C\right)\cdot \left(\sqrt{\mathcal{S}} + B\cdot x\right)}{2\cdot \left(D\cdot x^2+C\cdot x + 1\right)^{1.5}}
    \\
    &= \frac{\left(BC - 2SD\right)\cdot x + (2B - C\sqrt{\mathcal{S}})}{2\cdot \left(D\cdot x^2+C\cdot x + 1\right)^{1.5}}.
\end{align*}
Thus, the derivative at 0 is:
\begin{equation}
\begin{aligned}
    g_{s,a,q}'(0) &= \frac{2B - C\sqrt{\mathcal{S}}}{2} = B - \frac{1}{2}\cdot C\sqrt{\mathcal{S}}
    \\
    &= \frac{s\cdot (1 - \gamma)\cdot q}{4\gamma\cdot S} - \frac{1}{2}\cdot \left(\frac{(1+ \gamma)\cdot q}{4\gamma\cdot \mathcal{S}} + \frac{1}{a}\right)\cdot \sqrt{\mathcal{S}}
    \\
    &= \frac{s\cdot (1-\gamma)\cdot q}{4\gamma\cdot S} - \frac{(1+\gamma)\cdot q}{8\gamma\cdot \sqrt{\mathcal{S}}} - \frac{\sqrt{\mathcal{S}}}{2a} 
    \\
    &= \frac{2s\cdot (1 - \gamma)\cdot q -  (1+\gamma)\cdot q}{8\gamma\cdot S} - \frac{\sqrt{\mathcal{S}}}{2a}
    \\
    &= \frac{\left(\left(2s-1\right) -  \left(2s + 1\right)\cdot \gamma\right)\cdot q}{8\gamma\cdot \sqrt{\mathcal{S}}} - \frac{\sqrt{\mathcal{S}}}{2a}.
\end{aligned}
\label{g_s,a,q'(0) thm9}
\end{equation}

Now, from \eqref{l(x) def thm9}, the derivative $\ell'(x)$ reads
\begin{equation}
\begin{aligned}
    \ell'(x) &= \frac{d}{dx}\left(v(x) - u(x)\right)
    \\
    &= v'(x) - u'(x)
\end{aligned}
\label{ell'(x) first formula thm9}
\end{equation}
where we defined the following auxiliary functions:
\begin{equation}
\begin{cases}
    u(x) = \mathcal{Q}\left(g_{1,\gamma,k}(x)\right) + \mathcal{Q}\left(g_{-1,1,k}(x)\right) 
    \\
    v(x) = \mathcal{Q}\left(g_{1,\gamma,d}(x)\right) + \mathcal{Q}\left(g_{-1,1,d}(x)\right) 
\end{cases}
\label{u(x),v(x) def to l thm9}
\end{equation}
From the chain rule, their derivatives are:
\begin{equation}
\begin{cases}
    u'(x)=g'_{1,\gamma,k}(x)\cdot \mathcal{Q}'\left(g_{1,\gamma,k}(x)\right) + g'_{-1,1,k}(x)\cdot \mathcal{Q}'\left(g_{-1,1,k}(x)\right)
    \\
    v'(x)=g'_{1,\gamma,d}(x)\cdot \mathcal{Q}'\left(g_{1,\gamma,d}(x)\right) + g'_{-1,1,d}(x)\cdot \mathcal{Q}'\left(g_{-1,1,d}(x)\right)
\end{cases}
\label{u'(x), v'(x) formulas for ell thm9}
\end{equation}
Now, from \eqref{ell'(x) first formula thm9}, \eqref{u'(x), v'(x) formulas for ell thm9} it follows that:
\begin{equation}
    \ell'(0) = v'(0) - u'(0).
\label{l'(0) formula thm9}
\end{equation}
It is easy to verify from \eqref{g_s,a,q formula thm9} that $g_{s,a,q}(0) = \sqrt{\mathcal{S}}$. Thus, from \eqref{u(x),v(x) def to l thm9}, \eqref{u'(x), v'(x) formulas for ell thm9} and \eqref{g_s,a,q'(0) thm9}, we have the following formulas:
\begin{equation}
\begin{cases}
    u'(0) = \left(\frac{(1 - 3\gamma)\cdot k}{8\gamma\cdot \sqrt{\mathcal{S}}} - \frac{\sqrt{\mathcal{S}}}{2\gamma}\right)\cdot \mathcal{Q}'(\sqrt{\mathcal{S}}) + \left(\frac{-(3 + \gamma)\cdot k}{8\gamma\cdot \sqrt{\mathcal{S}}} - \frac{\sqrt{\mathcal{S}}}{2}\right)\cdot \mathcal{Q}'(\sqrt{\mathcal{S}})
    \\
    v'(0) = \left(\frac{(1 - 3\gamma)\cdot d}{8\gamma\cdot \sqrt{\mathcal{S}}} - \frac{\sqrt{\mathcal{S}}}{2\gamma}\right)\cdot \mathcal{Q}'(\sqrt{\mathcal{S}}) + \left(\frac{-(3 + \gamma)\cdot d}{8\gamma\cdot \sqrt{\mathcal{S}}} - \frac{\sqrt{\mathcal{S}}}{2}\right)\cdot \mathcal{Q}'(\sqrt{\mathcal{S}})
\end{cases}
\label{u(0),v(0),u'(0),v'(0) formulas thm9}
\end{equation}
Now, we substitute \eqref{u(0),v(0),u'(0),v'(0) formulas thm9} in \eqref{l'(0) formula thm9} to get the following formula for $\ell'(0)$:
\begin{equation}
\begin{aligned}
    \ell'(0) &= \frac{1 - 3\gamma}{8\gamma\sqrt{\mathcal{S}}}\mathcal{Q}'\left(\sqrt{\mathcal{S}}\right)\cdot (d-k) - \frac{3 + \gamma}{8\gamma\sqrt{\mathcal{S}}}\mathcal{Q}'\left(\sqrt{\mathcal{S}}\right)\cdot (d-k)
    \\
    &= \frac{-2 - 4\gamma}{8\gamma\sqrt{\mathcal{S}}}\mathcal{Q}'\left(\sqrt{\mathcal{S}}\right)\cdot (d-k)
    \\
    &= \frac{-2\left(1 + 2\gamma\right)}{8\gamma\sqrt{\mathcal{S}}}\cdot \left(-\frac{1}{\sqrt{2\pi}}\exp\left(-\frac{\mathcal{S}}{2}\right)\right)\cdot (d-k)
    \\
    &= \frac{1}{4\sqrt{2\pi}}\cdot \frac{\exp\left(-\frac{\mathcal{S}}{2}\right)}{\sqrt{\mathcal{S}}}\cdot \frac{1 + 2\gamma}{\gamma}\cdot (d-k)
    \\
    &= \frac{1}{4\sqrt{2\pi}}\cdot \frac{\exp\left(-\frac{\mathcal{S}}{2}\right)}{\sqrt{\mathcal{S}}}\cdot \left(2 + \frac{1}{\gamma}\right)\cdot (d-k)
\end{aligned}
\label{ell'(0) final formula-theorem 6 thm9}
\end{equation}
where we used the following property of the $\mathcal{Q}$ function:
\begin{equation*}
    \mathcal{Q}'(x) = -\frac{1}{\sqrt{2\pi}}\cdot \exp\left(-\frac{x^2}{2}\right).
\end{equation*}
Finally, from \eqref{delta first expansion as h(N) thm9}, \eqref{condition on N thm9}, \eqref{h(N tilde) approximation thm9}, \eqref{ell(0) thm9}, \eqref{ell'(0) final formula-theorem 6 thm9}, it follows that:
\begin{equation}
\begin{aligned}
    \Delta&=h({N})
    \\
    &= \ell(0) + \frac{\ell'(0)}{{N}} + \mathcal{O}\left(\frac{1}{N^2}\right)
    \\
    &= \frac{\ell'(0)}{{N}} + \mathcal{O}\left(\frac{1}{N^2}\right)
    \\
    &= \frac{1}{4\sqrt{2\pi}}\cdot \frac{\exp\left(-\frac{\mathcal{S}}{2}\right)}{\sqrt{\mathcal{S}}}\cdot \left(2 + \frac{1}{\gamma}\right)\cdot (d-k)\cdot \frac{1}{N} + \mathcal{O}\left(\frac{1}{N^2}\right).
\end{aligned}
\label{last - first approx for delta thm9}
\end{equation}
We proceed with \eqref{last - first approx for delta thm9} and substitute \eqref{condition on N thm9}:
\begin{equation*}
    {N} = \frac{N_{T}}{1+\gamma}
\end{equation*}
which leads to the following approximation of $\Delta$:
\begin{equation}
\begin{aligned}
    \Delta &= \frac{1}{4\sqrt{2\pi}}\cdot \frac{\exp\left(-\frac{\mathcal{S}}{2}\right)}{\sqrt{\mathcal{S}}}\cdot \left(2 + \frac{1}{\gamma}\right)\cdot (d-k)\cdot \frac{1 + \gamma}{N_T} + \mathcal{O}\left(\frac{1}{N_T^2}\right)
    \\
    &= \frac{1}{4\sqrt{2\pi}}\cdot \frac{\exp\left(-\frac{\mathcal{S}}{2}\right)}{\sqrt{\mathcal{S}}}\cdot \left(3 + 2\gamma + \frac{1}{\gamma}\right)\cdot (d-k)\cdot \frac{1}{N_T} + \mathcal{O}\left(\frac{1}{N_T^2}\right).
\end{aligned}
\label{final approximation for delta thm9}
\end{equation}
We now analyze the dependence of $N_T, d-k,\gamma, \mathcal{S}$ on $\Delta$ for $N_T \gg 1$. The results will be identical to those derived in \ref{thm:thm7}. From \eqref{final approximation for delta thm9}, the following hold for $N_T \gg 1$:
\begin{itemize}
    \item $\Delta$ decreases with $N_T$.
    \item $\Delta$ increases with $d - k$.
    \item $\Delta$ increases when $\gamma$ decreases within $\gamma \in \left(0, \frac{1}{\sqrt{2}}\right]$: this is because the function
    \begin{equation*}
        f(\gamma) = 3 + 2\gamma + \frac{1}{\gamma}
    \end{equation*}
    has a minimum at $\gamma = \frac{1}{\sqrt{2}}$ within $(0,1)$.

    \item $\Delta$ decreases with $\mathcal{S}$: this is because
    \begin{equation*}
        g(\mathcal{S}) \coloneqq \frac{\exp\left(-\frac{\mathcal{S}}{2}\right)}{\sqrt{\mathcal{S}}}
    \end{equation*}
    decreases with $\mathcal{S}$: indeed,
    \begin{align*}
        g'(\mathcal{S}) &= \frac{-\frac{1}{2}\exp\left(-\frac{\mathcal{S}}{2}\right)\sqrt{\mathcal{S}} - \frac{1}{2\sqrt{\mathcal{S}}}\exp\left(-\frac{\mathcal{S}}{2}\right)}{\mathcal{S}}
        \\
        &= -\frac{\exp\left(-\frac{\mathcal{S}}{2}\right)}{2\mathcal{S}}\cdot \left(\sqrt{\mathcal{S}} + \frac{1}{\sqrt{\mathcal{S}}}\right) < 0.
    \end{align*}
\end{itemize}
\end{proof}}

\section{Additional empirical details and results (CIFAR-10, denoising)}
\label{app:experiments}

\subsection{Experiments compute resources}
\label{Experiments Compute Resources}
We conducted our experiments using a few NVIDIA RTX 6000 Ada Generation GPUs with 48GB memory.
The training time for each data point in Figures \ref{fig:Practical_Networks_Efficiency} , \ref{fig:Practical_Networks_Efficiency_SURE} and \ref{fig:Practical_Networks_Efficiency_MSE_gamma_<_1} ranged from one hour to twelve hours, depending on the number of training samples.

\subsection{Training the classifier}

We consider the CIFAR-10 dataset \citep{krizhevsky2009learning} and the ResNet18 model \citep{he2016deep}. 
To train the model,
we use: batch size 128 and 350 epochs; cross-entropy loss; SGD optimizer; learning rate: 0.0679; learning rate decay: 0.1 at epochs 116 and 233; momentum: 0.9; weight decay: 0.0005.
This setting yields 90\% accuracy for clean data.

Per noise level $\sigma \in \{0.25,0.4\}$ of the additive Gaussian noise that has been added to the data, we use this setting to train two classifiers: one that operates directly on the noisy data and one that operates on the denoised data.

\subsection{Training the denoiser}

For the denoiser, we use the DnCNN model \citep{zhang2017beyond} and 15,000 training images while ignoring their labels. Per image, the clean version, $\x_{gt}$, is the target and its noisy version, $\x$, is the input to the model.
To train the model,
we use: batch size 64 and 1000 epochs; MSE loss; Adam optimizer; learning rate: 0.0001; learning rate decay: 0.5 at iterations 20k, 40k, 60k, 80k, 100k, and 200k. {\color{black} The results with the MSE-based denoiser with $\gamma < 1$ are presented in Figure \ref{fig:Practical_Networks_Efficiency_MSE_gamma_<_1}. Note that, in order for the division of samples among the first five classes to be valid, we require $\frac{N_{\text{train}}}{1 + \gamma} \le 17500 \Rightarrow N_{\text{train}} \le 17500\left(1 + \gamma\right)$. This shows that the point $N_{\text{train}} = 35000$ is invalid for all $\gamma < 1$. Thus, we add a sufficient amount of samples from the synthetic set CIFAR-5m to the classifier train set, both in the noisy and denoised case (where the noisy CIFAR-5m passes through the denoiser).}

\subsection{Training the denoiser without clean images}

Replacing the MSE loss with Stein's Unbiased Risk Estimate (SURE) \citep{stein1981estimation,soltanayev2018training} allows to train the denoiser using only noisy images.
Specifically, instead of $\mathrm{MSE}({\z}_{\theta}(\x),\x_{gt})=\|{\z}_{\theta}(\x) - \x_{gt}\|^2$, we use:
\begin{align*}
    \mathrm{SURE}({\z}_{\theta}(\x)) = \|\z_{\theta}(\x) - \x\|^2 -d\sigma^2 +2\sigma^2 \sum_{i=1}^d \frac{\partial}{\partial x_i} \z_{\theta}(\x),
\end{align*}
which obeys $\mathbb{E}[\mathrm{SURE}({\z}_{\theta}(\x))] = \mathbb{E}[\mathrm{MSE}({\z}_{\theta}(\x),\x_{gt})]$ for $\x|\x_{gt} \sim \mathcal{N}(\x_{gt},\sigma^2\I)$.
We use the common practice of approximating the divergence term with $\g^\top(\z_{\theta}(\x+\epsilon\g) - \z_{\theta}(\x))/\epsilon$, where $\epsilon$ is small and $\g \sim \mathcal{N}(\0,\I)$ is drawn per optimizer iteration. 
Additionally, 
we use:
batch size 64 and 1000 epochs; Adam optimizer; learning rate: 0.0001; learning rate decay: 0.5 at iterations 20k, 40k, 60k, 80k, 100k, and 200k.

The results for the setup with the SURE-based denoiser are presented in Figure \ref{fig:Practical_Networks_Efficiency_SURE}.
It can be seen that they resemble the results for the MSE-based denoiser, which are presented in Section \ref{sec:empirical}.

\subsection{Numerical accuracy results}

In the following Tables \ref{tab:noisy_denoised_errors} and \ref{tab:noisy_denoised_errors_sigma_0p4} we report accuracy results related to Figure \ref{fig:Practical_Networks_Efficiency}.

\begin{table}[htbp]
    \centering
    \begin{tabular}{|c|c|c|}
        \hline
        \textbf{$N_{\text{train}}$} & \textbf{Error without denoising (\%)} & \textbf{Error with denoising (\%)} \\
        \hline
        1000  &  71.12 $\pm 1.39$  &  63.32 $\pm 0.96$  \\
        2000  &  64.91 $\pm 1.44$  &  58.14 $\pm 1.41$ \\
        3000  &  63.19 $\pm 1.21$  &  54.62 $\pm 0.37$  \\
        5000  &  60.43 $\pm 1.42$  &  49.79 $\pm 1.80$ \\
        10000 &  46.15 $\pm 0.81$  &  42.43 $\pm 0.25$  \\
        15000 &  42.14 $\pm 0.8$  &  39.74 $\pm 0.40$  \\
        25000 &  38.48 $\pm 0.29$  &  36.02 $\pm 0.38$  \\
        35000 &  35.77 $\pm 0.30$  &  33.83 $\pm 0.25$ \\
        \hline
    \end{tabular}
    \caption{Classification error rates (\%) on noisy and denoised CIFAR-10 images for varying training set sizes $N_{\text{train}}$. The noise level is $\sigma=0.25$, $\gamma=1$, and the denoiser is trained with MSE loss.}
    \label{tab:noisy_denoised_errors}
\end{table}

\begin{table}[htbp]
    \centering
    \begin{tabular}{|c|c|c|}
        \hline
        \textbf{$N_{\text{train}}$} & \textbf{Error without denoising (\%)} & \textbf{Error with denoising (\%)} \\
        \hline
        1000  & 74.45 $\pm$ 0.81  & 64.82 $\pm$ 1.14 \\
        2000  & 71.01 $\pm$ 2.00  & 61.25 $\pm$ 0.90 \\
        3000  & 67.86 $\pm$ 1.46  & 58.30 $\pm$ 0.66 \\
        5000  & 64.98 $\pm$ 1.21  & 56.24 $\pm$ 2.03 \\
        10000 & 54.70 $\pm$ 1.14  & 49.64 $\pm$ 0.50 \\
        15000 & 50.30 $\pm$ 0.47  & 47.34 $\pm$ 0.41 \\
        25000 & 47.65 $\pm$ 0.37  & 45.14 $\pm$ 0.43 \\
        35000 & 45.40 $\pm$ 0.33  & 43.21 $\pm$ 0.20 \\
        \hline
    \end{tabular}
    \caption{Classification error rates (\%) on noisy and denoised CIFAR-10 images for varying training set sizes $N_{\text{train}}$.  The noise level is $\sigma=0.4$, $\gamma=1$, and the denoiser is trained using MSE loss.}
    \label{tab:noisy_denoised_errors_sigma_0p4}
\end{table}

Figure \ref{fig:errors_comparison} shows the classification error vs.~the training epoch in a single trial for noise level 0.25, $\gamma=1$ and 35,000 training images. It demonstrates that the classifier does not suffer from overfitting.

\begin{figure}[htbp]
    \centering
    \begin{subfigure}[b]{0.48\linewidth}
        \centering
        \includegraphics[width=\linewidth]{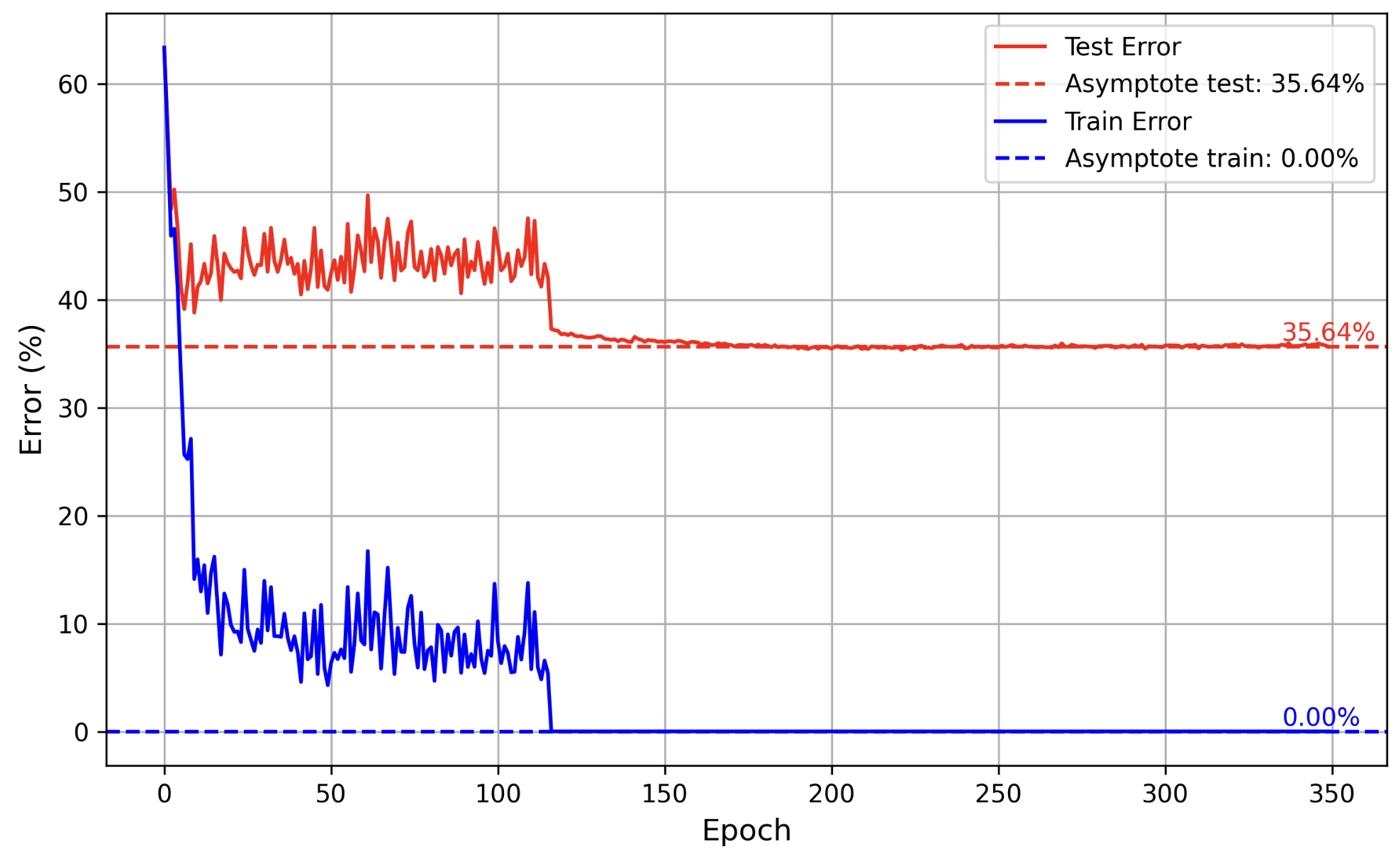}
        \caption{Noisy data}
        \label{fig:noisy}
    \end{subfigure}
    \hfill
    \begin{subfigure}[b]{0.48\linewidth}
        \centering
        \includegraphics[width=\linewidth]{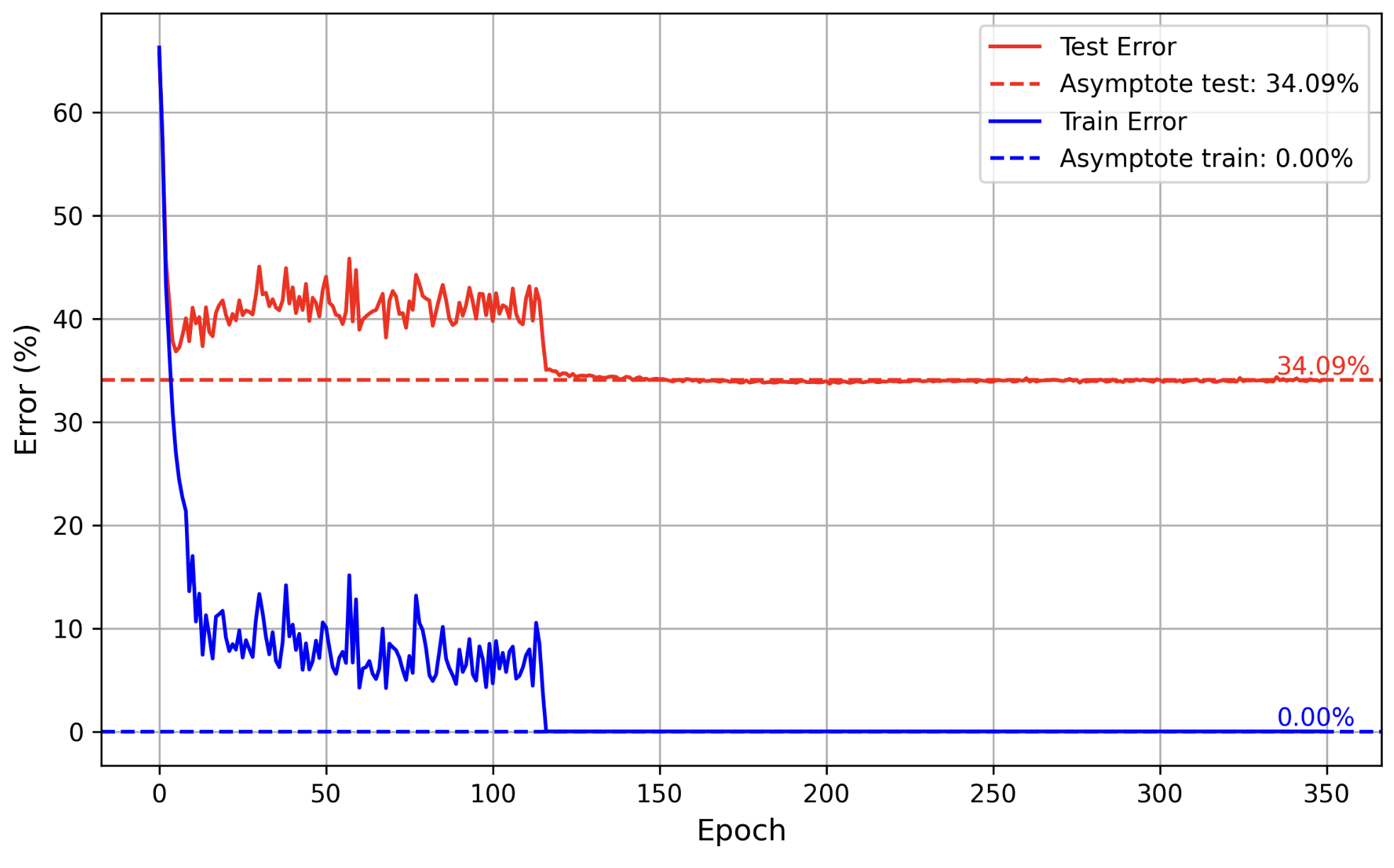}
        \caption{Denoised data}
        \label{fig:denoised}
    \end{subfigure}
    \caption{Training and testing error as a function of epochs for (a) noisy data and (b) denoised data. The noise level is $\sigma = 0.25$, $\gamma=1$, and $N_{\text{train}} = 35{,}000$. The denoiser is trained using MSE loss.}
    \label{fig:errors_comparison}
\end{figure}

\begin{figure}[t]
    \centering
    \begin{subfigure}[b]{0.48\linewidth}
        \centering
        \includegraphics[width=\linewidth]{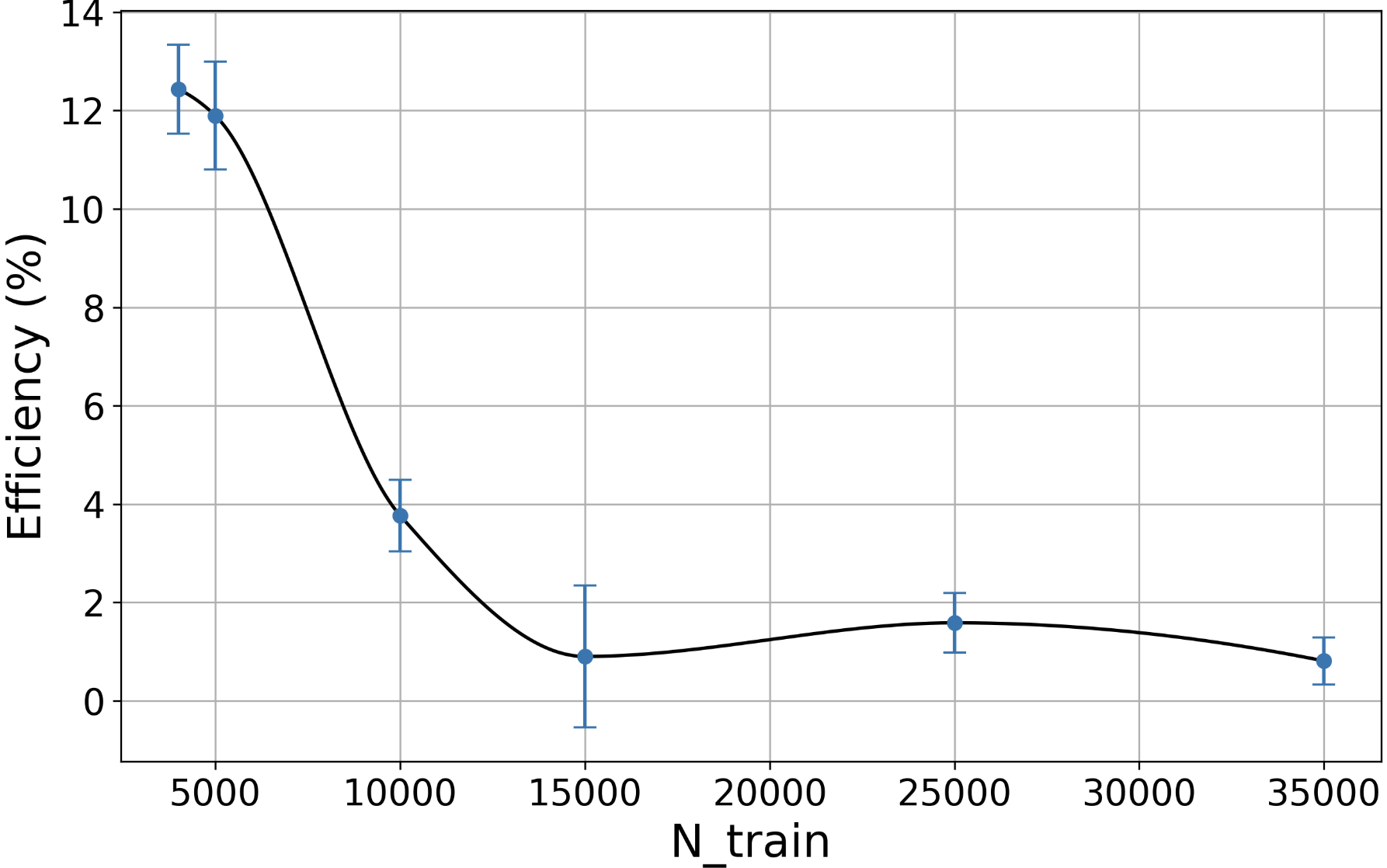}
        \caption{$\sigma = 0.25, \gamma = 1$}
        \label{fig:SURE_Sigma_0p25_gamma_1}
    \end{subfigure}
    \hfill
    \begin{subfigure}[b]{0.48\linewidth}
        \centering
        \includegraphics[width=\linewidth]{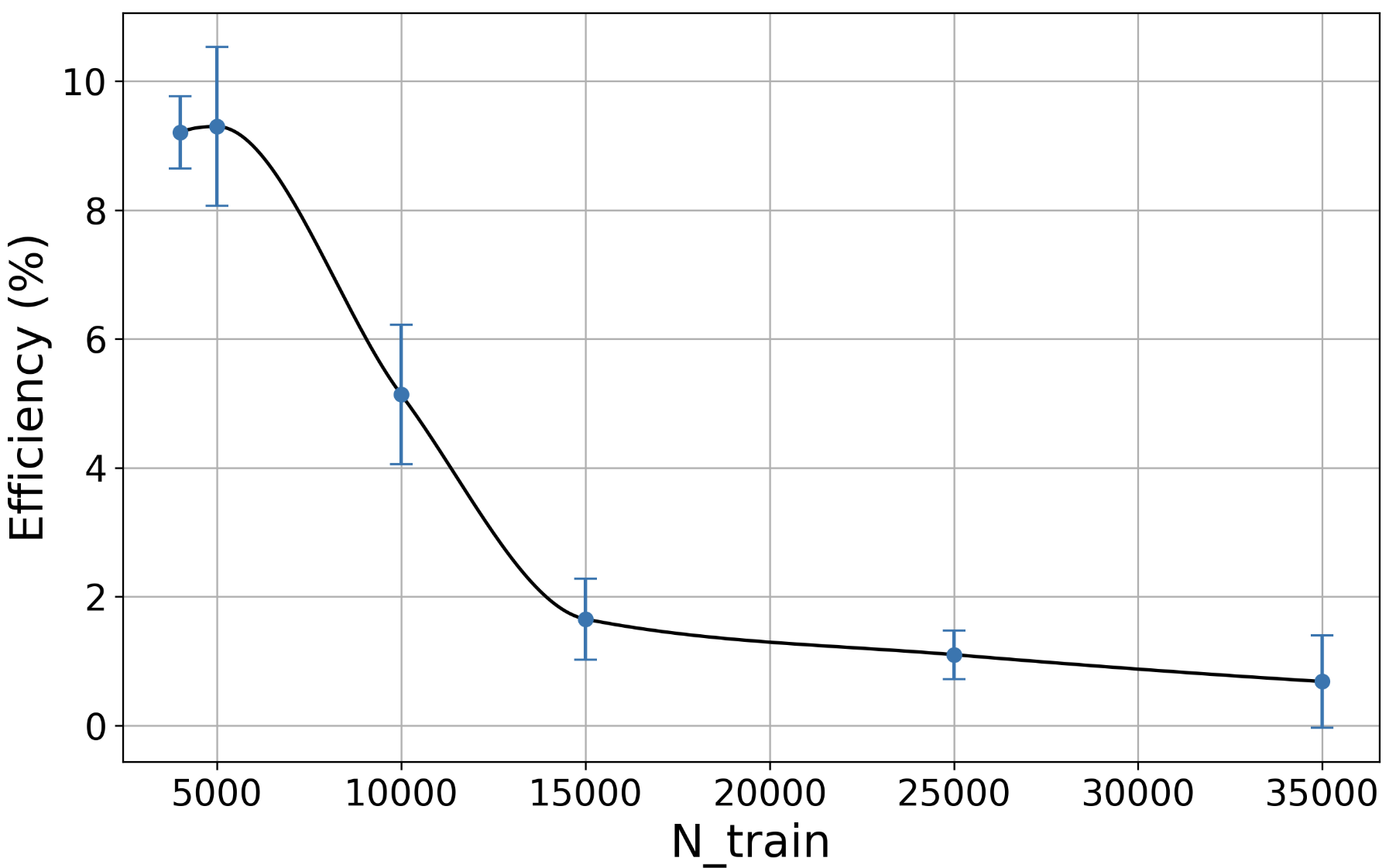}
        \caption{$\sigma = 0.4, \gamma = 1$}
        \label{fig:SURE_Sigma_0p4_gamma_1}
    \end{subfigure}

    \vspace{0.5cm}

    \begin{subfigure}[b]{0.48\linewidth}
        \centering
        \includegraphics[width=\linewidth]{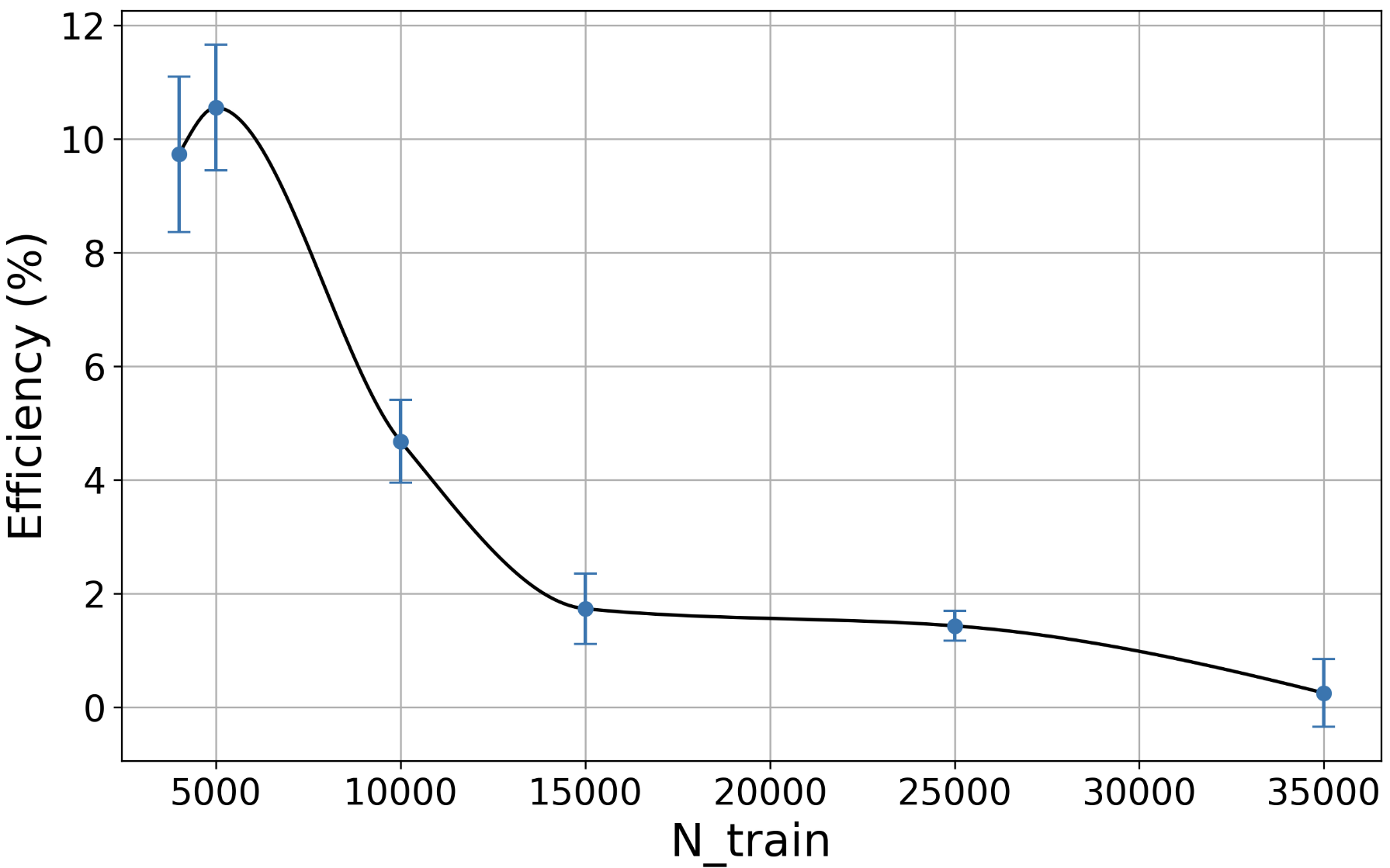}
        \caption{$\sigma = 0.25, \gamma = 0.75$}
        \label{fig:SURE_0p25_gamma_0p75}
    \end{subfigure}
    \hfill
    \begin{subfigure}[b]{0.48\linewidth}
        \centering
        \includegraphics[width=\linewidth]{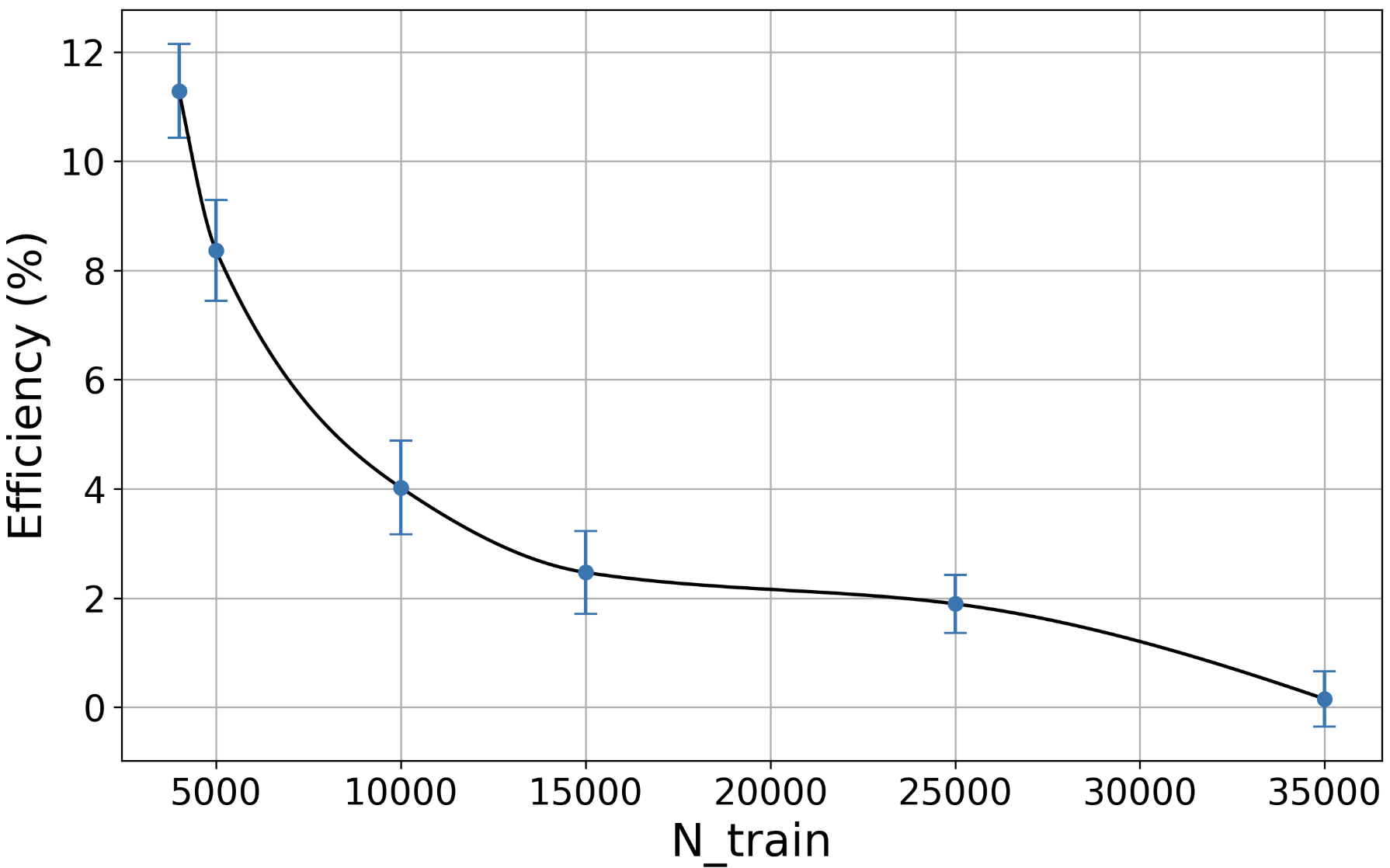}
        \caption{$\sigma = 0.4, \gamma = 0.75$}
        \label{fig:SURE_Sigma_0p4_gamma_0p75}
    \end{subfigure}

    \vspace{0.5cm}

    \begin{subfigure}[b]{0.48\linewidth}
        \centering
        \includegraphics[width=\linewidth]{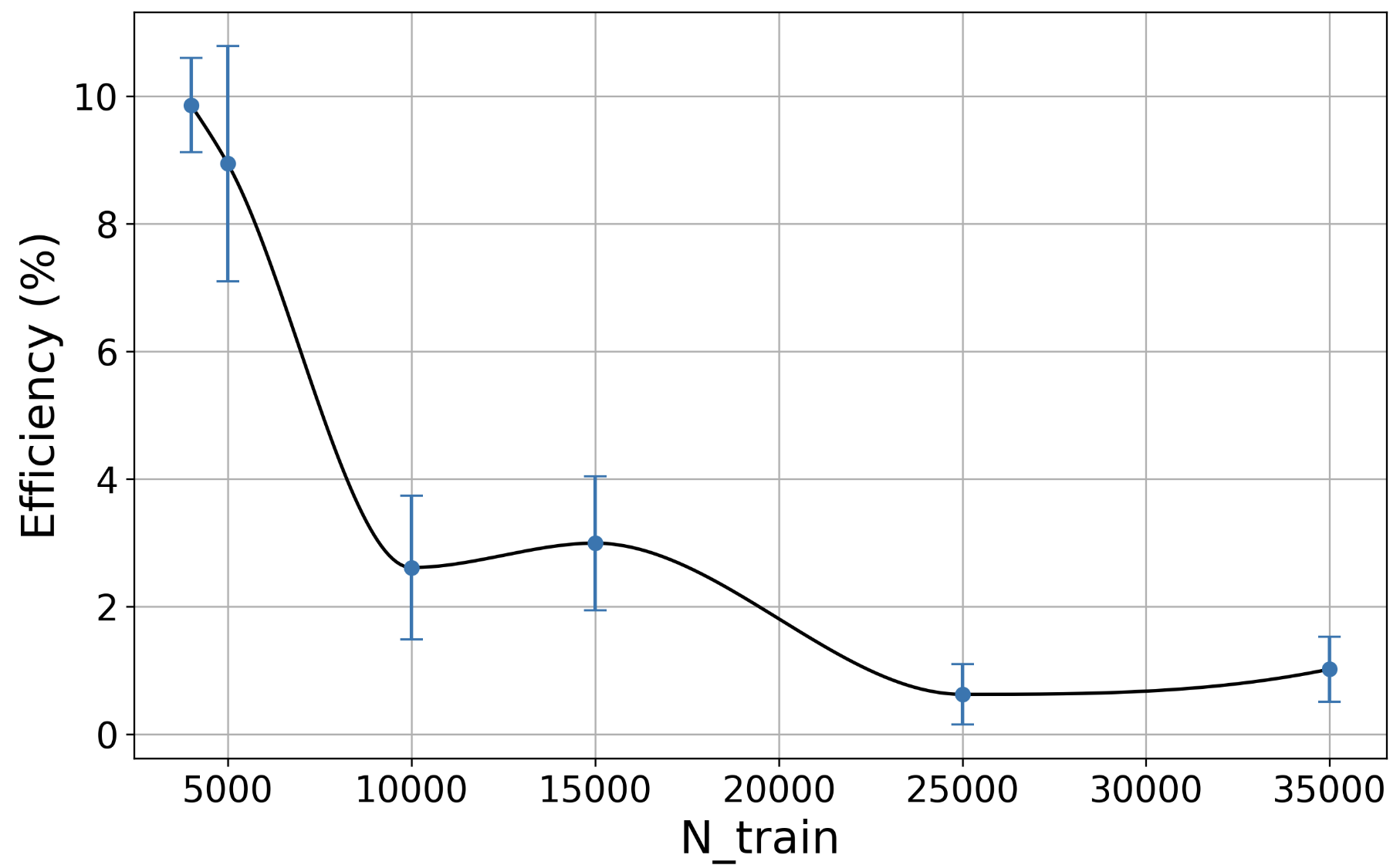}
        \caption{$\sigma = 0.25, \gamma = 0.5$}
        \label{fig:SURE_Sigma_0p25_gamma_0p5}
    \end{subfigure}
    \hfill
    \begin{subfigure}[b]{0.48\linewidth}
        \centering
        \includegraphics[width=\linewidth]{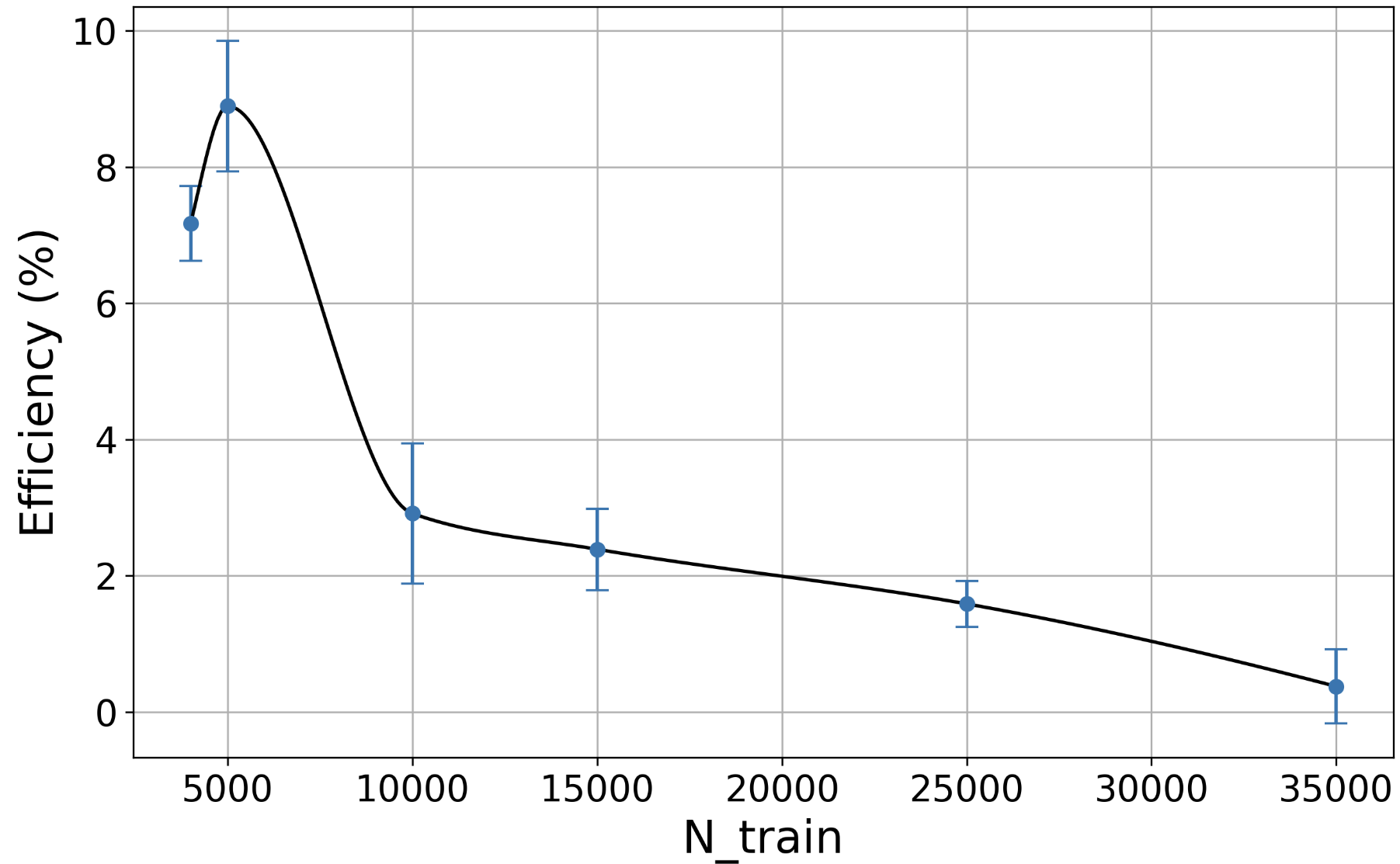}
        \caption{$\sigma = 0.4, \gamma = 0.5$}
        \label{fig:SURE_Sigma_0p4_gamma_0p5}
    \end{subfigure}

    \caption{Practical deep learning setup with noisy CIFAR-10 and SURE-based denoiser. Efficiency of the data processing procedure versus the number of training samples for various values of the training imbalance factor, $\gamma$, and the standard deviation of the noise, $\sigma$.}
    \label{fig:Practical_Networks_Efficiency_SURE}
\end{figure}

\begin{figure}[t]
    \centering
    \begin{subfigure}[b]{0.48\linewidth}
        \centering
        \includegraphics[width=\linewidth]{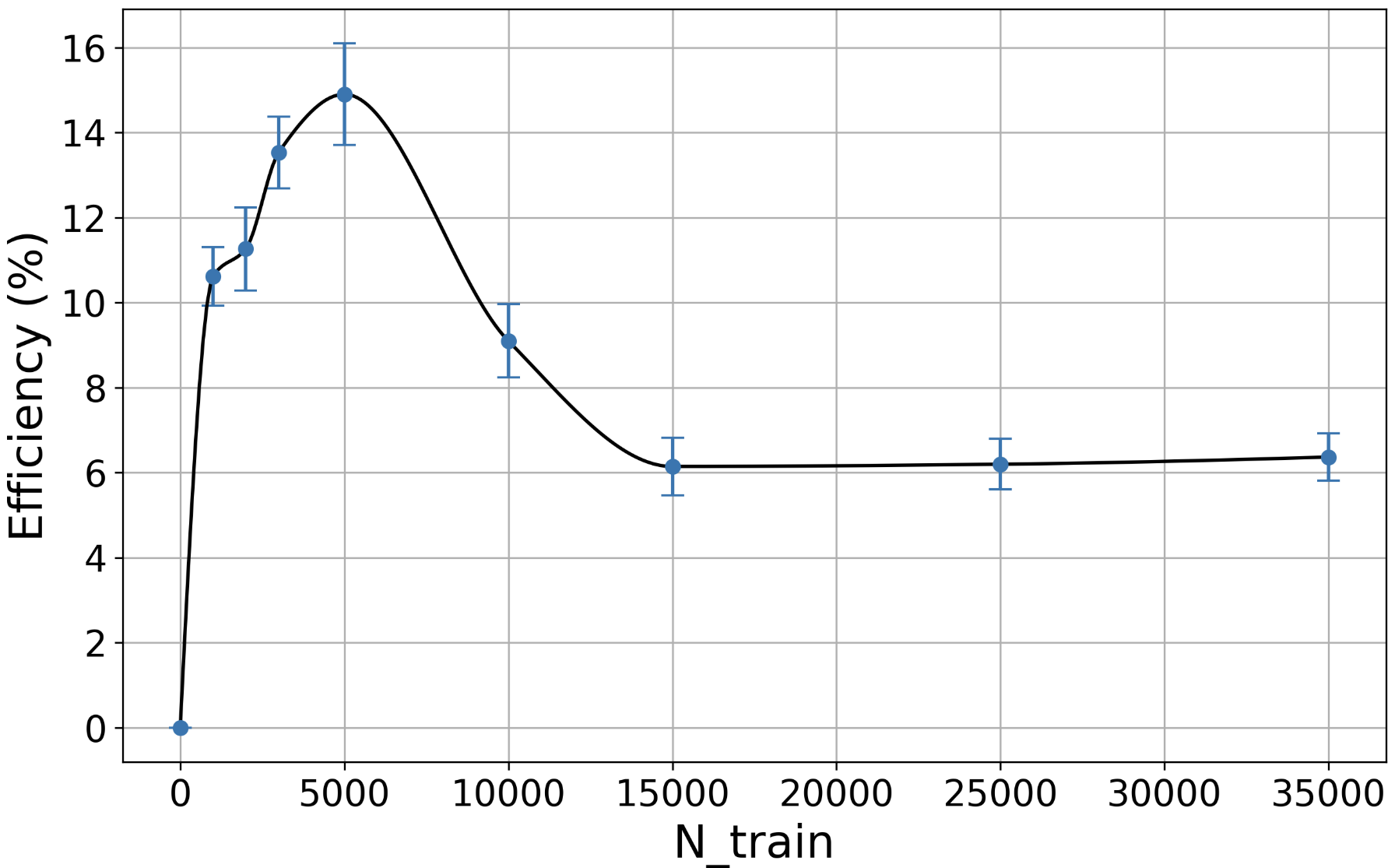}
        \caption{$\sigma = 0.25, \gamma = 0.75$}
        \label{fig:MSE_Sigma_0p25_gamma_1}
    \end{subfigure}
    \hfill
    \begin{subfigure}[b]{0.48\linewidth}
        \centering
        \includegraphics[width=\linewidth]{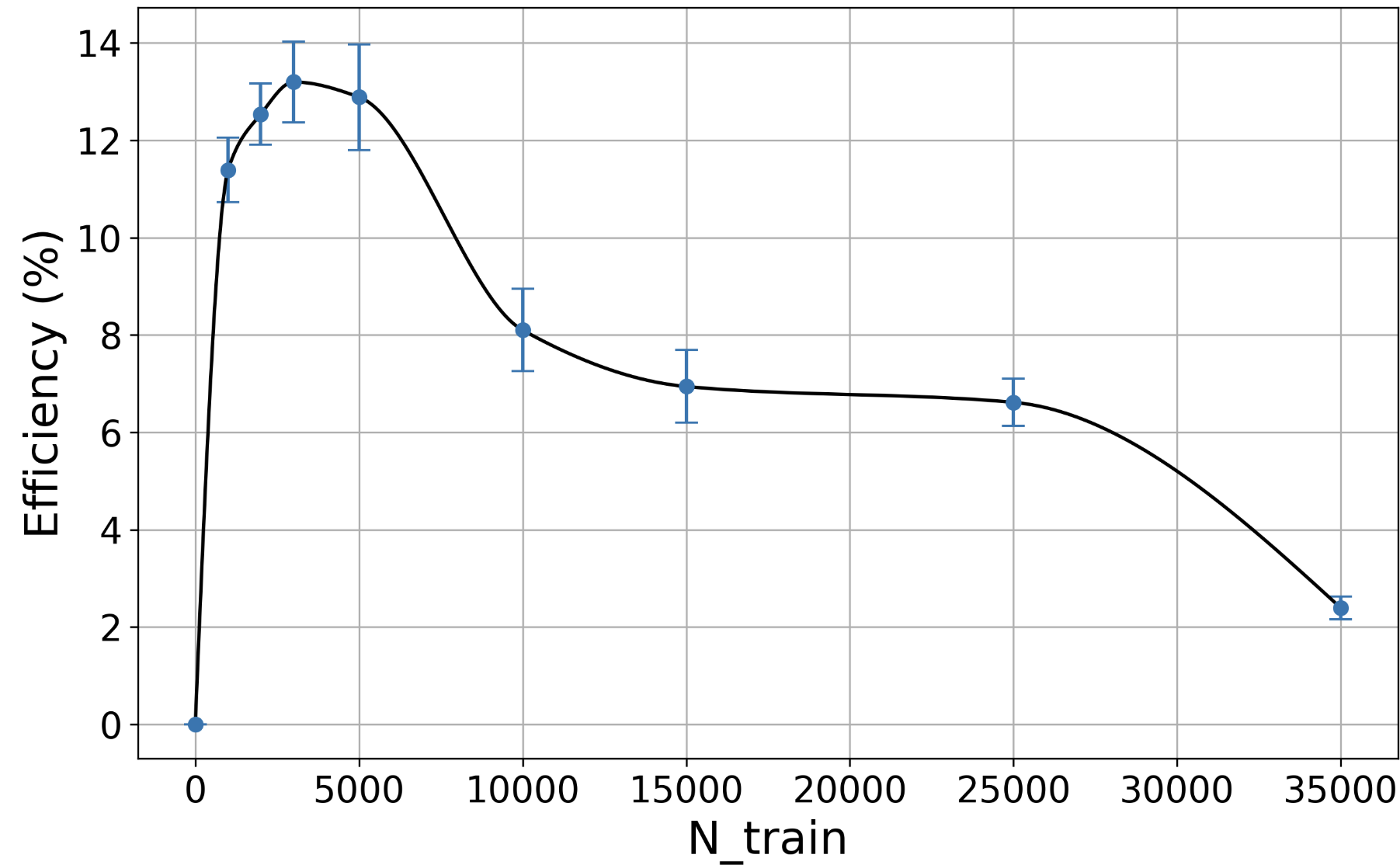}
        \caption{$\sigma = 0.4, \gamma = 0.75$}
        \label{fig:MSE_Sigma_0p4_gamma_1}
    \end{subfigure}

    \vspace{0.5cm}

    \begin{subfigure}[b]{0.48\linewidth}
        \centering
        \includegraphics[width=\linewidth]{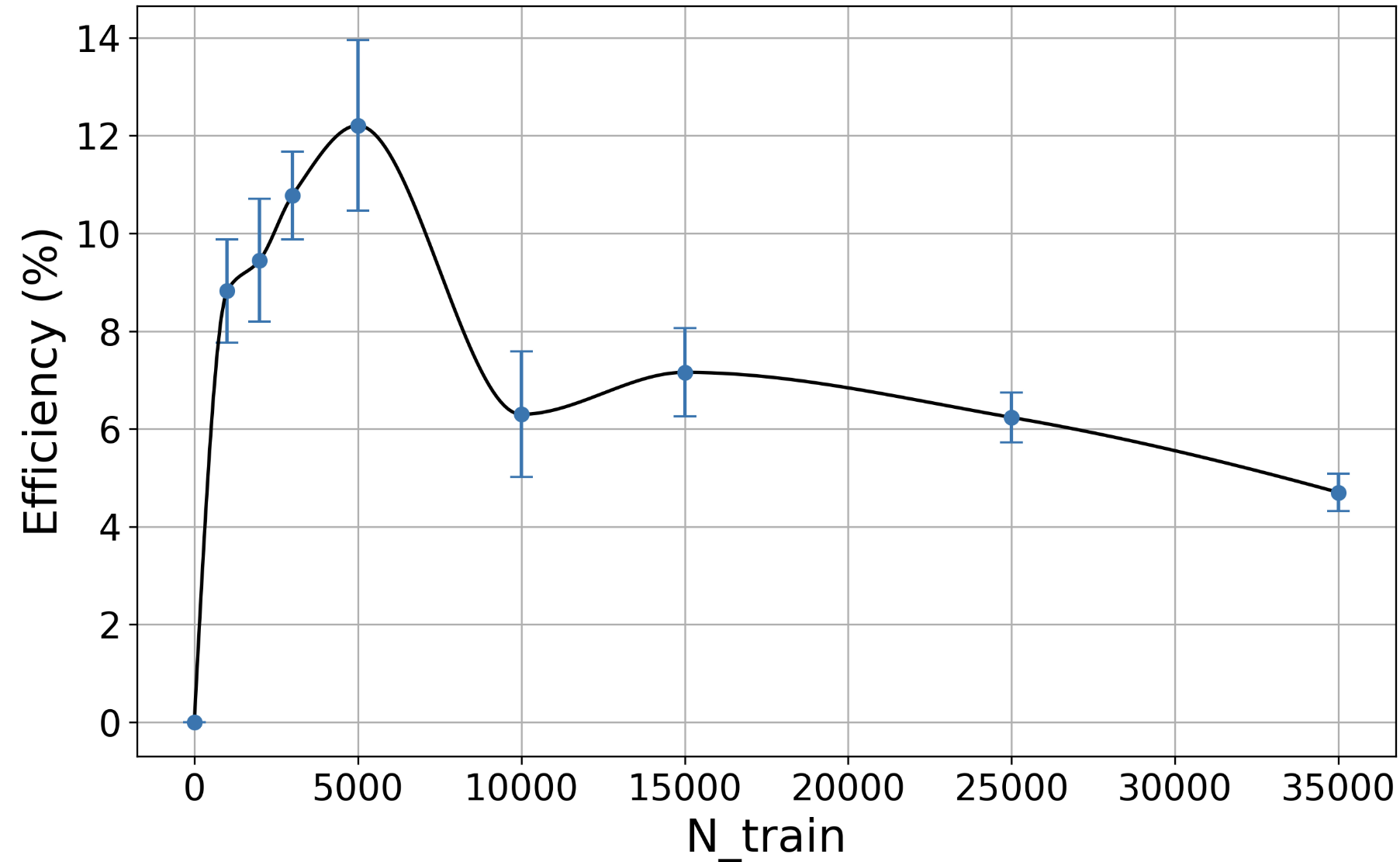}
        \caption{$\sigma = 0.25, \gamma = 0.5$}
        \label{fig:MSE_0p25_gamma_0p5}
    \end{subfigure}
    \hfill
    \begin{subfigure}[b]{0.48\linewidth}
        \centering
        \includegraphics[width=\linewidth]{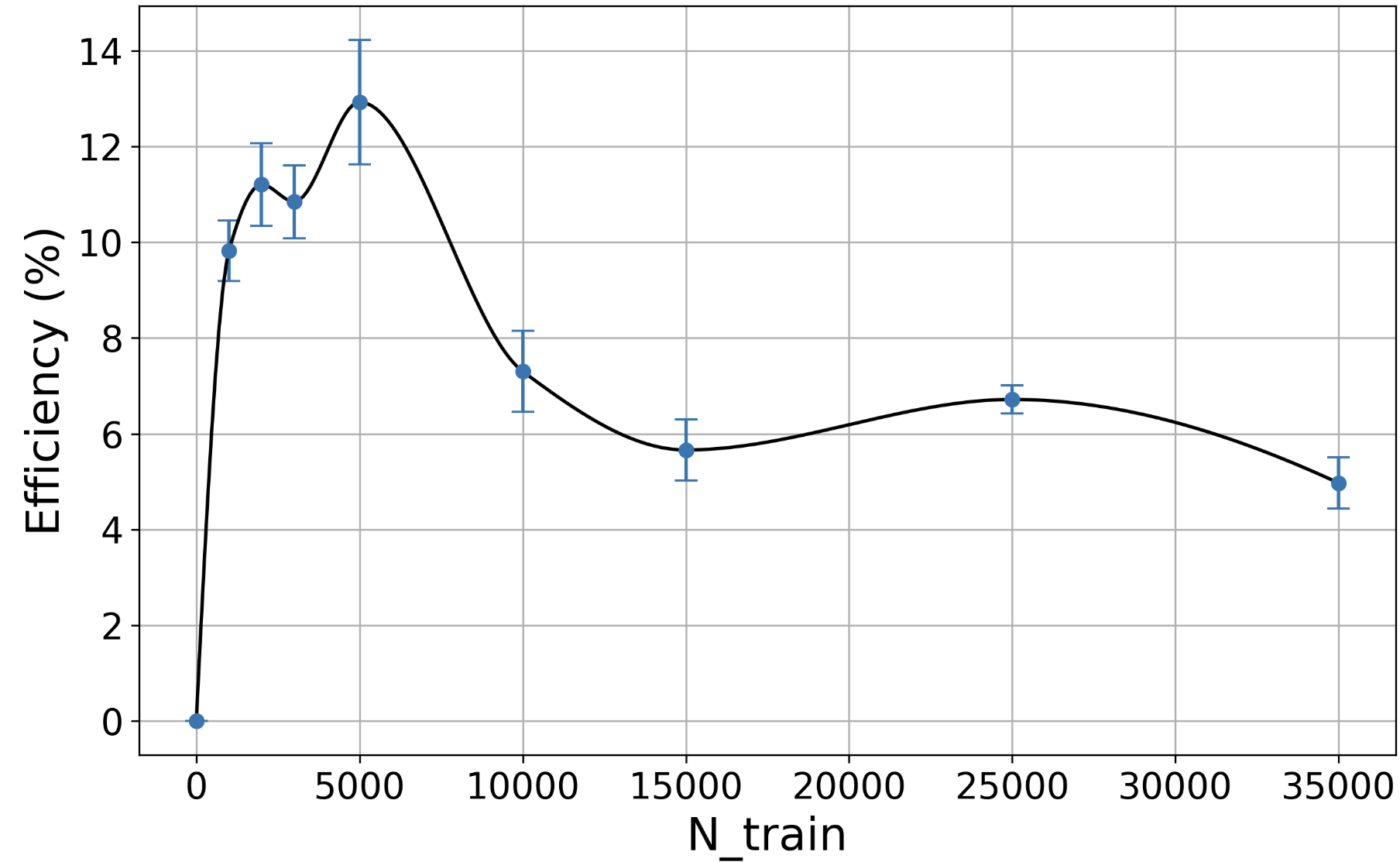}
        \caption{$\sigma = 0.4, \gamma = 0.5$}
        \label{fig:MSE_Sigma_0p4_gamma_0p75}
    \end{subfigure}

    \vspace{0.5cm}

    \caption{Practical deep learning setup with noisy CIFAR-10 and MSE-based denoiser. Efficiency of the data processing procedure versus the number of training samples for various values of the training imbalance factor, $\gamma \in \{0.5, 0.75\}$, and the standard deviation of the noise, $\sigma$.}
    \label{fig:Practical_Networks_Efficiency_MSE_gamma_<_1}
\end{figure}

\newpage

\section{Additional empirical details and results (Mini-ImageNet, encoding)}
\label{app:experiments_encoding}

\subsection{Experiments compute resources}
\label{Experiments Compute Resources - encoding}
We conducted our experiments using 16 NVIDIA Tesla V100-SXM2 GPUs with 32GB memory, 12 NVIDIA RTX 6000 Ada Generation GPUs with 48GB memory, and 2 NVIDIA A100 PCIe GPUs with 80GB memory. The training time for each data point in Figures \ref{fig:Practical_Networks_Efficiency_encoding} and \ref{fig:Practical_Networks_Efficiency_encoding_gamma<1} ranged from 10 hours to 30 hours, depending on the number of training samples.

\subsection{Training the classifier}
\label{Training the classifier - encoding}
We consider the Mini-ImageNet dataset and the ResNet50 model. To train the model, we use: batch size 128 and 225 epochs; cross-entropy loss; SGD optimizer; learning rate: 0.0679; learning rate decay: 0.1 at epochs 75 and 150; momentum: 0.9; weight decay: 0.0005. This setting yields 73\% accuracy for clean data.

Per noise level $\sigma \in \{\frac{50}{255}, \frac{100}{255}\}$ of the additive Gaussian noise that has been added to the data, we use this setting to train one classifier that operates directly on the noisy data.

\subsection{Training the encoder}
\label{training the encoder - encoding}

For self-supervised learning, we adopt the DINOv2 framework \citep{lu2025ditch}. The student encoder is a Vision Transformer (ViT-S/16), which splits each input image of size 224×224 into 16×16 patches and produces a 384-dimensional [CLS] token representation. This is passed through a 3-layer MLP projection head to produce the final 256-dimensional embedding ($\mathbf{z} \in \mathbb{R}^{256}$), which is used for self-supervised training. The teacher network has the same architecture and is updated as an exponential moving average of the student, providing stable target embeddings. Training is performed on the Mini-ImageNet dataset for 200 epochs with a per-GPU batch size of 40. We apply the AdamW optimizer with a base learning rate of 0.004 (scaled with the square root of the effective batch size), $\beta = (0.9, 0.999)$, weight decay scheduled from 0.04 to 0.4, and gradient clipping at 3.0. The teacher momentum is linearly increased from 0.992 to 1.0 over training. Multi-crop augmentation is employed with 2 global crops of size 224×224 and 8 local crops of size 96×96. Model evaluation is conducted every 6,250 iterations.

\subsection{Training an MLP on top of the embeddings}
\label{training an MLP on top of the embeddings - encoder}

Per noise level $\sigma \in \{\frac{50}{255}, \frac{100}{255}\}$, after training the DINOv2 encoder, we pass the noisy Mini-ImageNet images through the encoder to obtain 256-dimensional embeddings. On top of these embeddings, we train a multi-layer perceptron (MLP) classifier to perform image classification. The MLP consists of three hidden layers with dimensions 4096, 2048, and 1024, each followed by LayerNorm and GELU activation, and a final linear layer mapping to the number of classes (i.e. 100). Hidden layers are initialized with Xavier uniform, and the final layer with a small normal distribution. 

To train the model, we use: per-GPU batch size 128 and 20 epochs, with 1250 iterations per epoch; cross-entropy loss; SGD optimizer with a cosine annealing learning rate schedule; momentum: 0.9; no weight decay. Linear evaluation is performed with periodic check-pointing and evaluation on the validation set. After training, the classifier is evaluated on the test set to report final accuracy.

The results for the setup with $\gamma < 1$ are presented in Figure \ref{fig:Practical_Networks_Efficiency_encoding_gamma<1}. It can be seen that they resemble the results in \ref{fig:Practical_Networks_Efficiency_encoding}: 1) similar non-monotonicity of the curve while remaining positive, and 2) the maximal efficiency increases with the SNR, for fixed $\gamma$.

\begin{figure}[t]
    \centering
    \begin{subfigure}[b]{0.48\linewidth}
        \centering
        \includegraphics[width=\linewidth]{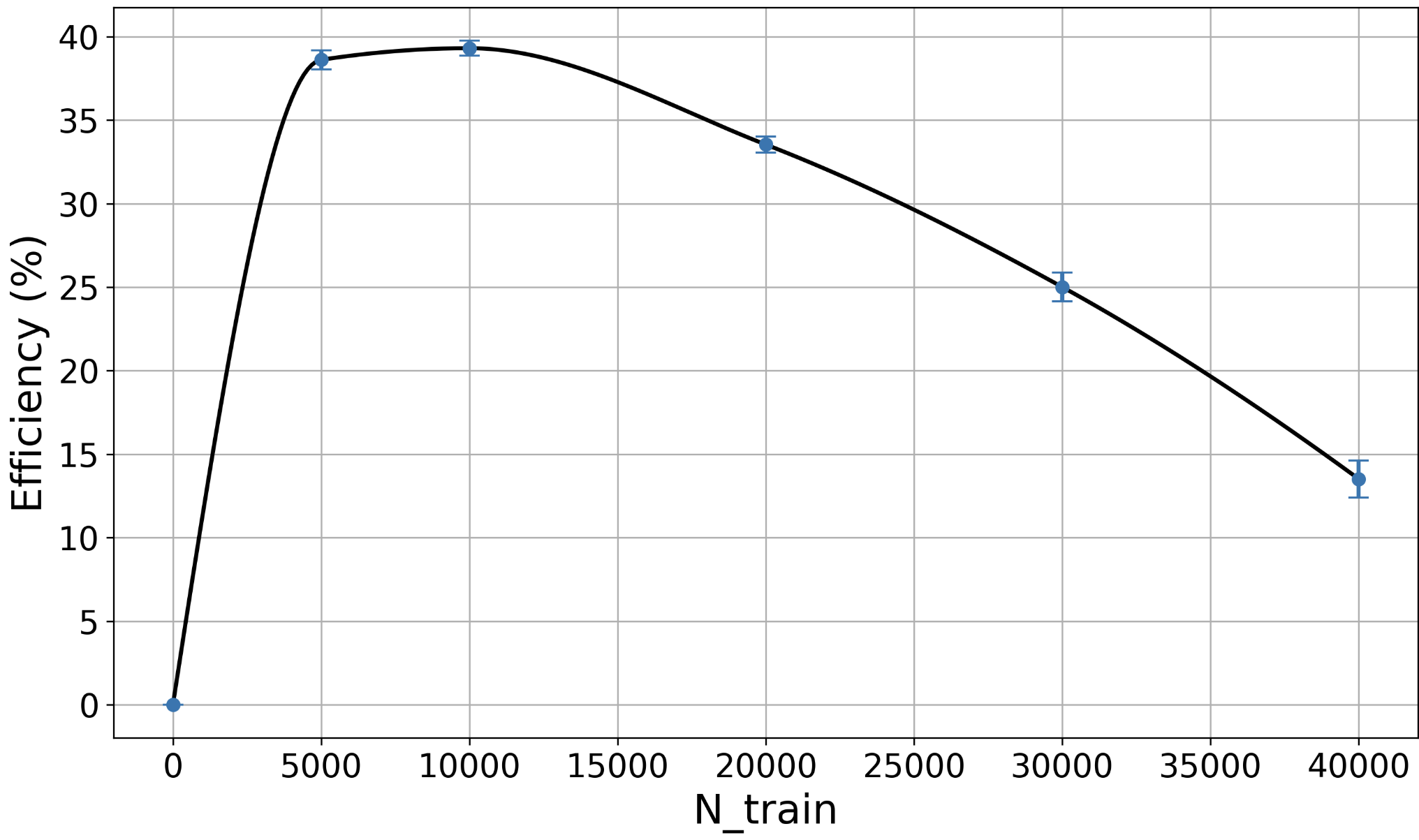}
        \caption{$\sigma = \tfrac{50}{255}, \gamma = 0.75$}
        \label{fig:Embedding_Sigma50_gamma0p75}
    \end{subfigure}
    \hfill
    \begin{subfigure}[b]{0.48\linewidth}
        \centering
        \includegraphics[width=\linewidth]{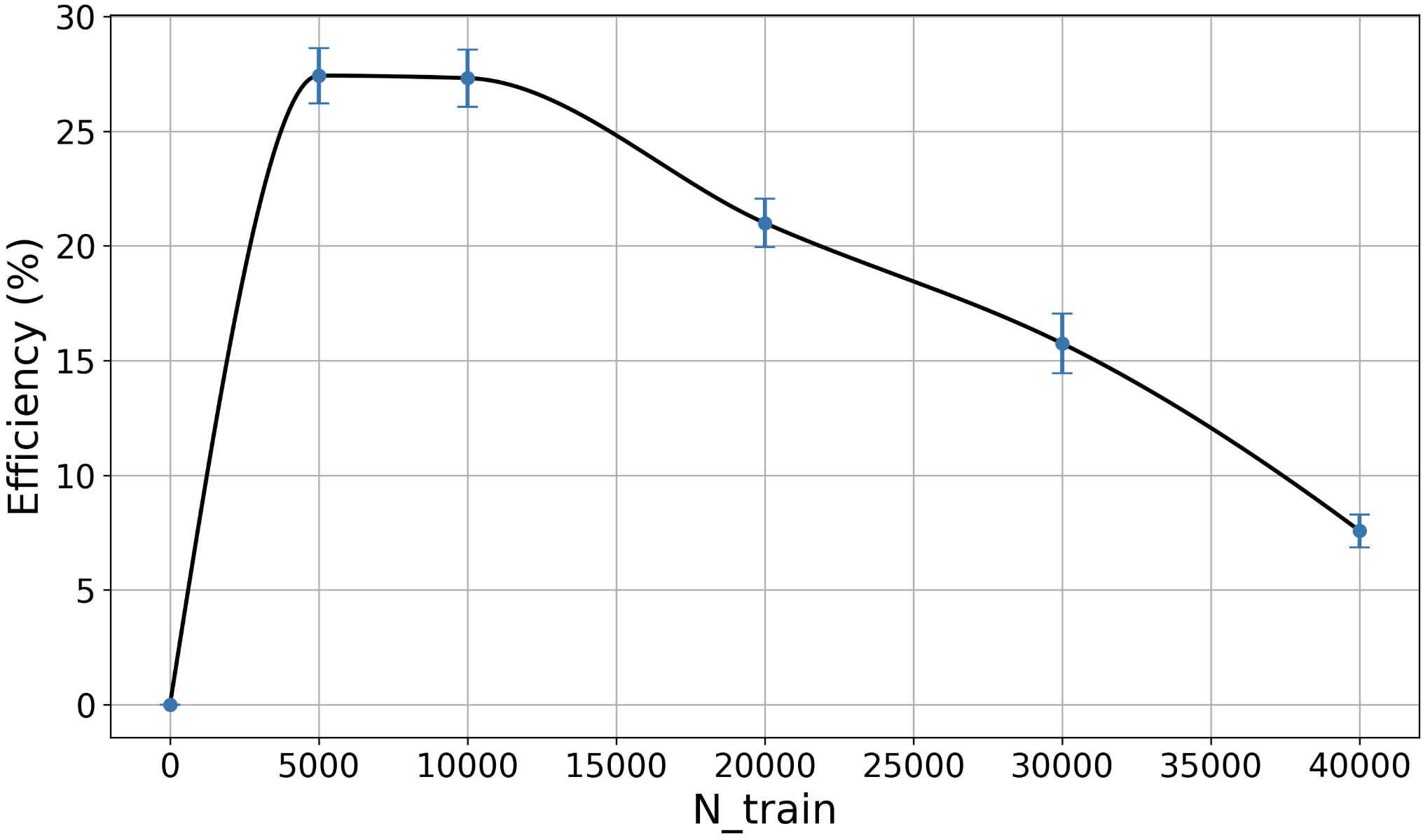}
        \caption{$\sigma = \tfrac{100}{255}, \gamma = 0.75$}
        \label{fig:Embedding_Sigma100_gamma0p75}
    \end{subfigure}

    \vspace{0.5cm}

        \begin{subfigure}[b]{0.48\linewidth}
        \centering
        \includegraphics[width=\linewidth]{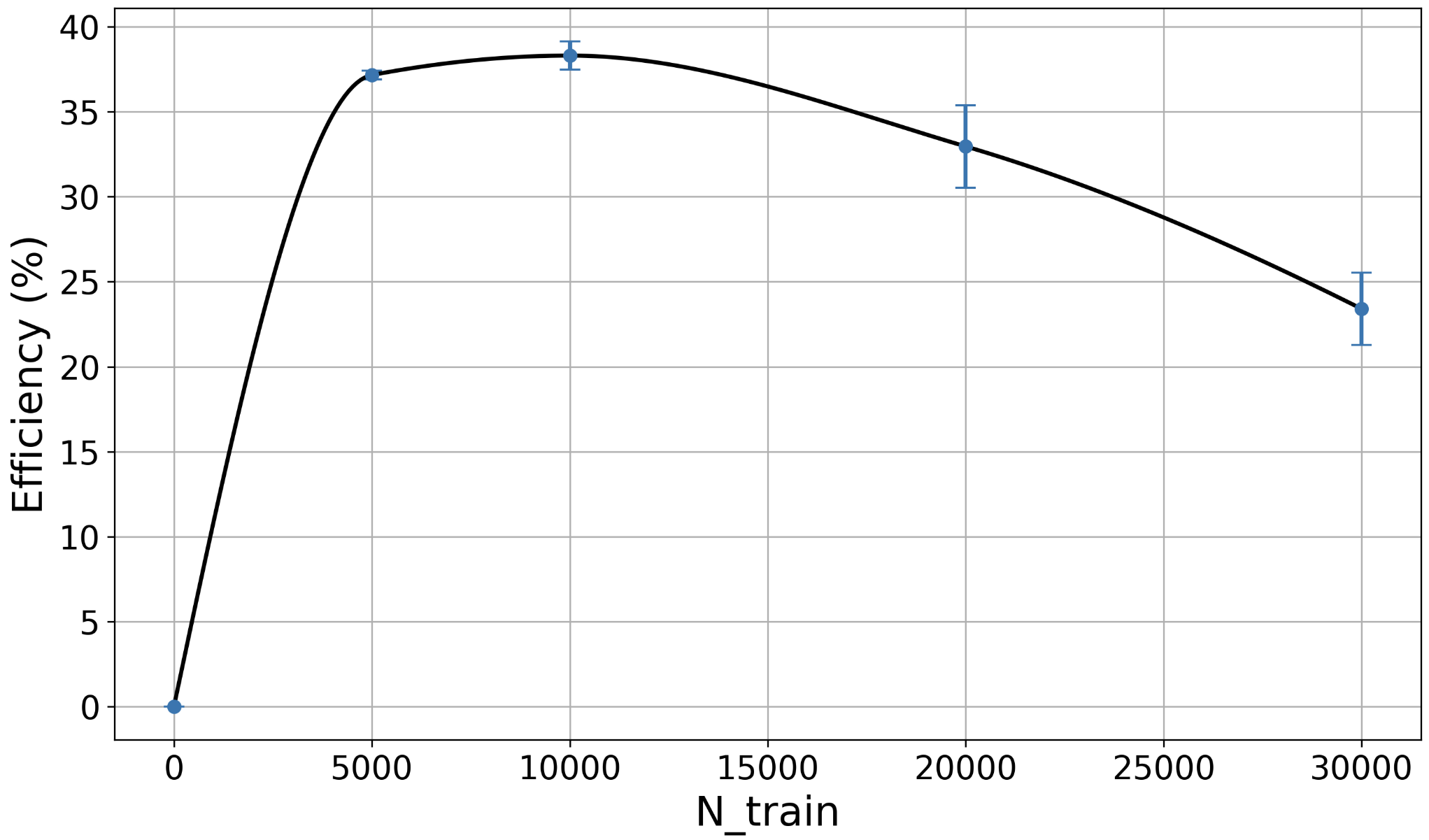}
        \caption{$\sigma = \tfrac{50}{255}, \gamma = 0.5$}
        \label{fig:Embedding_Sigma50_gamma0p5}
    \end{subfigure}
    \hfill
    \begin{subfigure}[b]{0.48\linewidth}
        \centering
        \includegraphics[width=\linewidth]{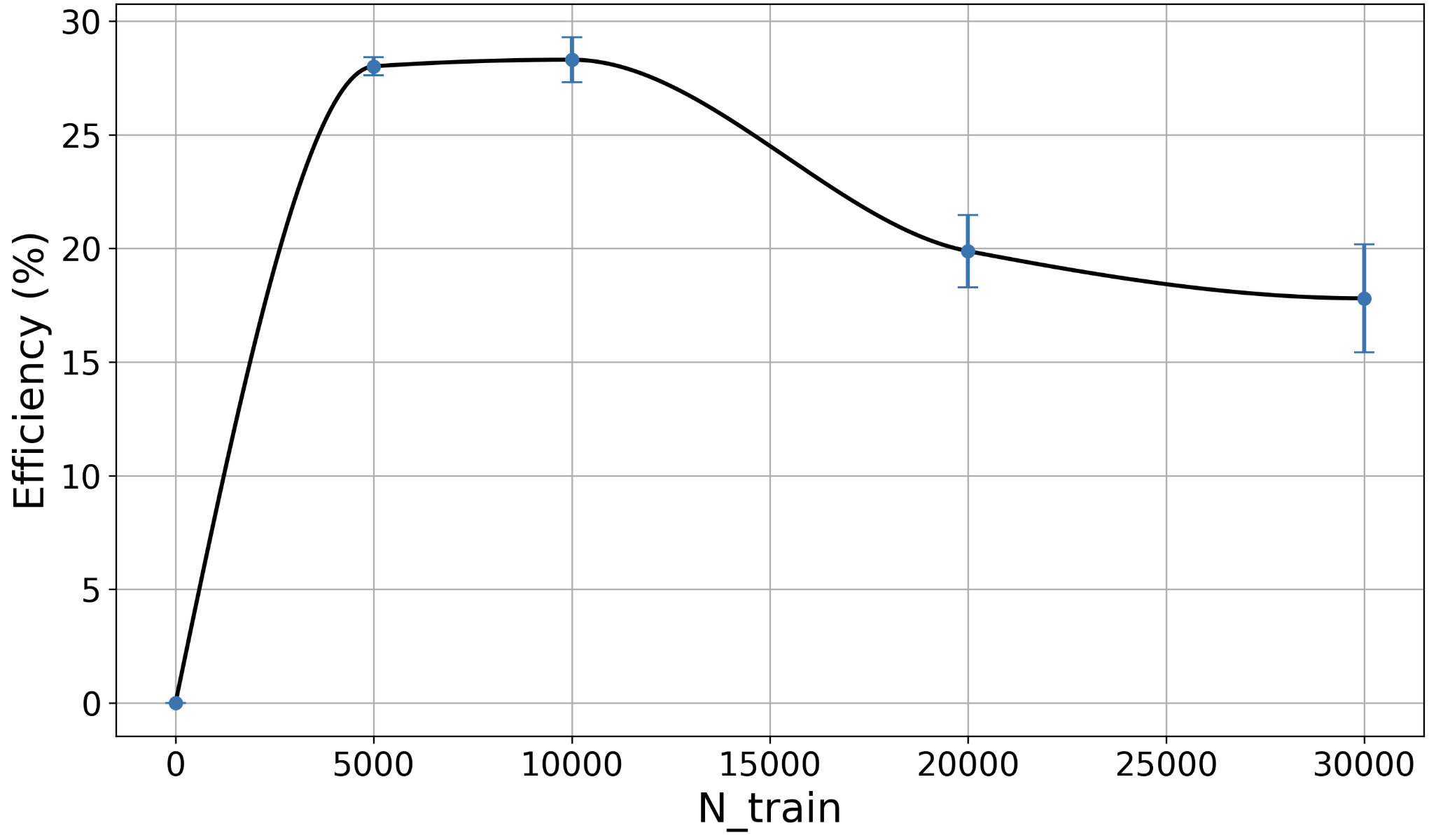}
        \caption{$\sigma = \tfrac{100}{255}, \gamma = 0.5$}
        \label{fig:Embedding_Sigma100_gamma0p5}
    \end{subfigure}

    \vspace{0.5cm}

    \caption{Practical deep learning setup with noisy Mini-ImageNet and pre-classification encoding. Efficiency of the data processing procedure versus the number of training samples for various values of the training imbalance factor, $\gamma \in \{0.5, 0.75\}$, and the standard deviation of the noise, $\sigma$.}

    \label{fig:Practical_Networks_Efficiency_encoding_gamma<1}
    
\end{figure}

\subsection{Numerical accuracy results}
In the following Tables \ref{tab:noisy_denoised_errors_sigma50_encoding} and \ref{tab:noisy_denoised_errors_sigma100_encoding} we report accuracy results related to Figure \ref{fig:Practical_Networks_Efficiency_encoding}. Lastly, in Figure \ref{fig:NoisyData_ImageNet} we present an image of noisy data for $\sigma \in \{50/255, 100/255\}$, to show that the noise is not too-severe.

\begin{table}[htbp]
    \centering
    \begin{tabular}{|c|c|c|}
        \hline
        \textbf{$N_{\text{train}}$} & \textbf{Error without encoding (\%)} & \textbf{Error with encoding (\%)} \\
        \hline
        5000  &  80.03 $\pm 1.25$  &  48.67 $\pm 0.48$  \\
        10000  &  74.56 $\pm 1.12$  &  45.03 $\pm 0.27$ \\
        20000  &  58.81 $\pm 0.45$  & 41.45 $\pm 0.29$  \\
        30000  &  49.92 $\pm 0.8$  &  39.23 $\pm 0.39$ \\
        40000 &  44.95 $\pm 1.37$  &  37.99 $\pm 0.31$  \\
        50000 &  40.07 $\pm 0.58$  &  36.62 $\pm 0.25$  \\
        \hline
    \end{tabular}
    \caption{Classification error rates (\%) on noisy and encoded Mini-ImageNet images for varying training set sizes $N_{\text{train}}$. The noise level is $\sigma=\tfrac{50}{255}$, $\gamma=1$.}
    \label{tab:noisy_denoised_errors_sigma50_encoding}
\end{table}

\begin{table}[htbp]
    \centering
    \begin{tabular}{|c|c|c|}
        \hline
        \textbf{$N_{\text{train}}$} & \textbf{Error without encoding (\%)} & \textbf{Error with encoding (\%)} \\
        \hline
        5000  &  85.5 $\pm 0.95$  &  60.94 $\pm 0.25$  \\
        10000  &  79.71 $\pm 1.73$  &  57.28 $\pm 0.23$ \\
        20000  &  68.47 $\pm 2.15$  & 53 $\pm 0.27$  \\
        30000  &  58.39 $\pm 1.77$  &  50.92 $\pm 0.31$ \\
        40000 &  54.62 $\pm 0.08$  &  49.35 $\pm 0.34$  \\
        50000 &  50.01 $\pm 0.56$  &  48.23 $\pm 0.27$  \\
        \hline
    \end{tabular}
    \caption{Classification error rates (\%) on noisy and encoded Mini-ImageNet images for varying training set sizes $N_{\text{train}}$. The noise level is $\sigma=\tfrac{100}{255}$, $\gamma=1$.}
    \label{tab:noisy_denoised_errors_sigma100_encoding}
\end{table}

\begin{figure}
    \centering

    \begin{subfigure}[b]{0.3\linewidth}
        \centering
        \includegraphics[width=\linewidth]{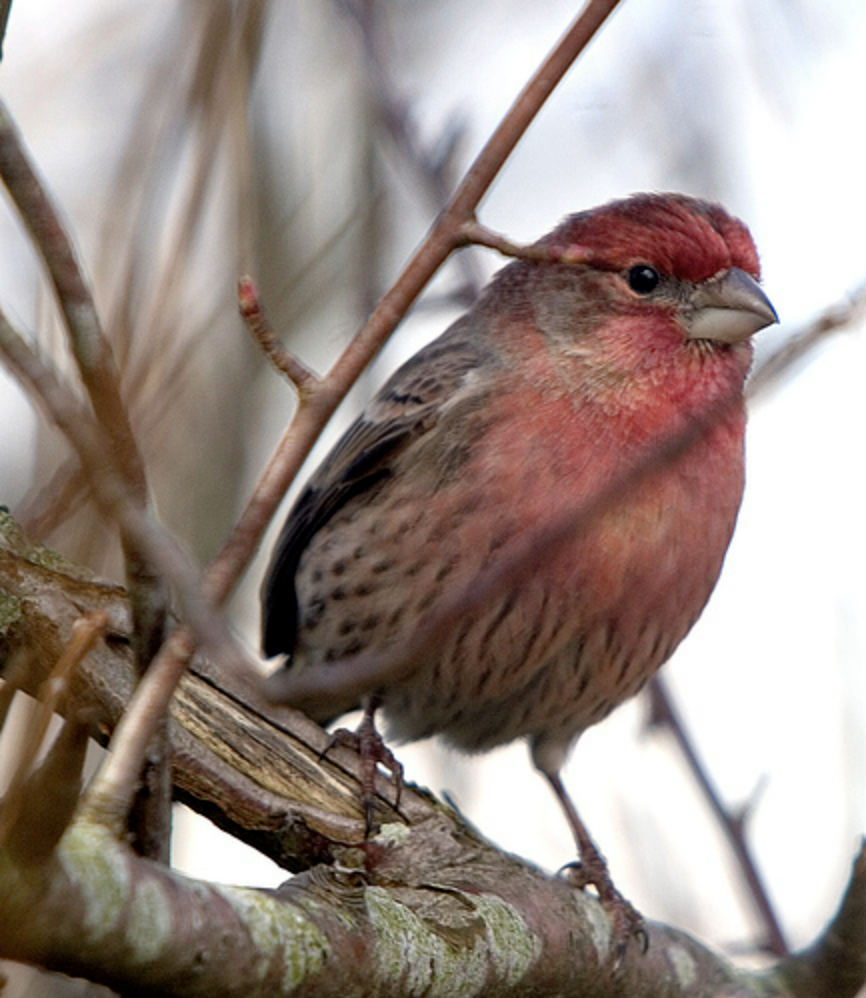}
        \caption{clean}
        \label{fig:Clean_imagenet}
    \end{subfigure}
    \vspace{3mm}

    \begin{subfigure}[b]{0.3\linewidth}
        \centering
        \includegraphics[width=\linewidth]{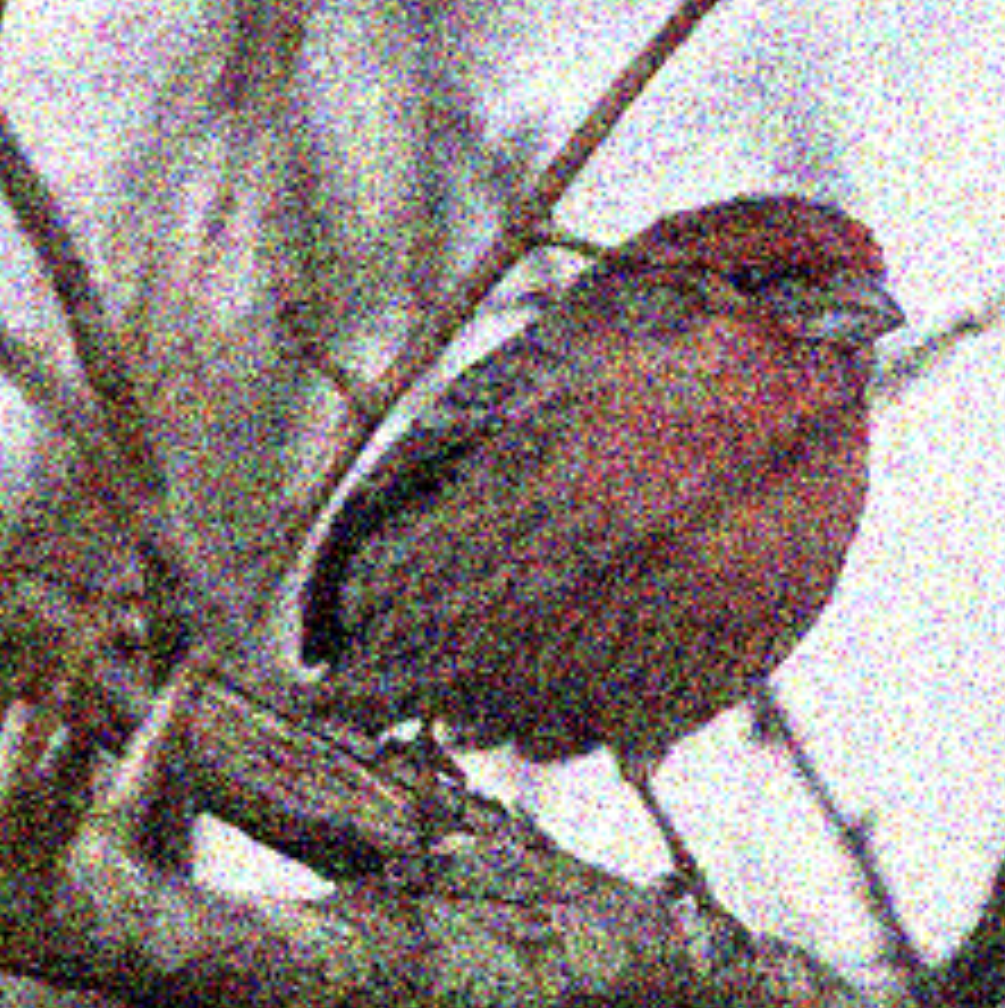}
        \caption{$\sigma = \tfrac{50}{255}$}
        \label{fig:Sigma_50_imagenet}
    \end{subfigure}
    \hspace{3mm}
    \begin{subfigure}[b]{0.3\linewidth}
        \centering
        \includegraphics[width=\linewidth]{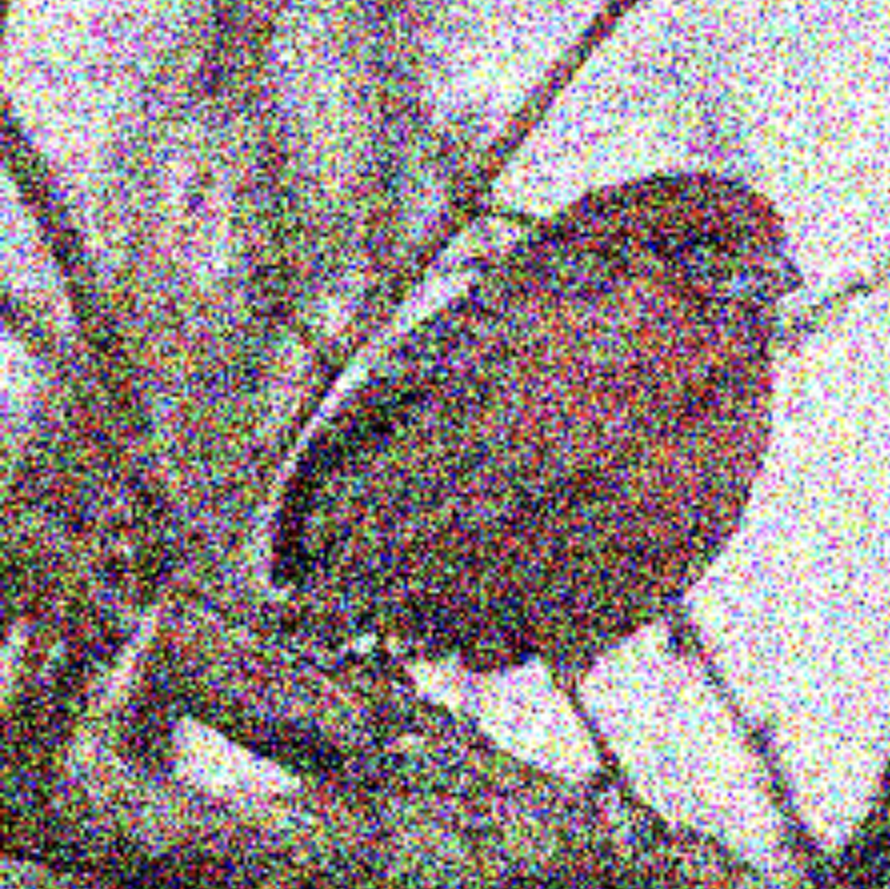}
        \caption{$\sigma = \tfrac{100}{255}$}
        \label{fig:Sigma_100_imagenet}
    \end{subfigure}
    \caption{Clean and noisy Mini-ImageNet images. (a) Clean image. (b) Image with Gaussian noise $\sigma = \tfrac{50}{255}$. (c) Image with Gaussian noise $\sigma = \tfrac{100}{255}$.}
    \label{fig:NoisyData_ImageNet}
\end{figure}

{\color{black}

\section{Additional empirical details and results (CIFAR-10, encoding)}
\label{app:experiments_encoding_cifar10}

We investigate the CIFAR-10 dataset and the ResNet18 model. Both the training and test sets are subjected to additive Gaussian noise with standard deviations $\sigma \in \{0.25, 0.4\}$. 
This time, as the data processing procedure we use an
encoding step that maps each image (rescaled from its original CIFAR-10 resolution to $224 \times 224$) into a 256-dimensional embedding. This encoder model follows \citep{lu2025ditch} and is trained from scratch with self-supervision on $45000$ noisy unlabeled images for each noise level. Then, for each combination of $(\sigma, N_{\text{train}})$, considering the balanced case of $\gamma=1$, we divide $N_{\text{train}}$ equally among all 10 classes. Then, we train a ResNet18 model on the noisy images across 6 seeds and, in parallel, a small MLP on the corresponding embeddings across 3 seeds. After we have the mean of the probability of error before and after the data processing, we compute the empirical efficiency, i.e., the relative percentage change in the probability of error induced by the encoding step. Details of the training procedures for the ResNet18 and the MLP are provided in Appendix \ref{app:experiments} and \ref{app:experiments_encoding}, respectively. % 

Figure \ref{fig:Practical_Networks_Efficiency_encoding_cifar10} presents the efficiency versus $N_{\text{train}}$, for $\gamma = 1$. We see the same trends that are aligned with our theory as before: 1) similar non-monotonicity of the curve (increase to a maximal efficiency value and then decrease) while remaining positive, and 2) the maximal efficiency increases with the SNR.}

\begin{figure}[t]
    \centering
    \begin{subfigure}[b]{0.48\linewidth}
        \centering
        \includegraphics[width=\linewidth]{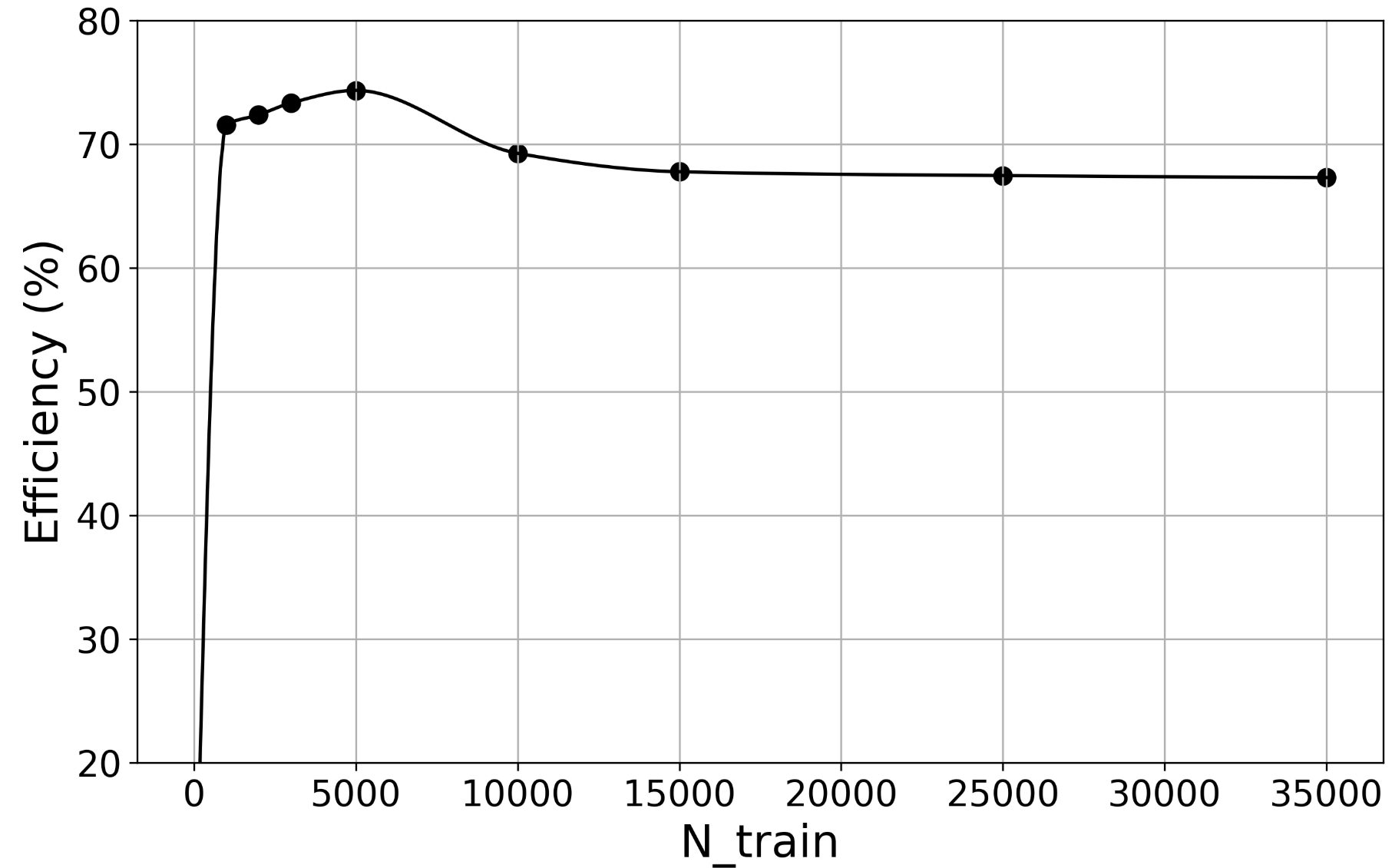}
        \caption{$\sigma = 0.25, \gamma = 1$}
        \label{fig:Sigma_0p25_gamma_1_embedding_cifar10}
    \end{subfigure}
    \hspace{3mm}
    \begin{subfigure}[b]{0.48\linewidth}
        \centering
        \includegraphics[width=\linewidth]{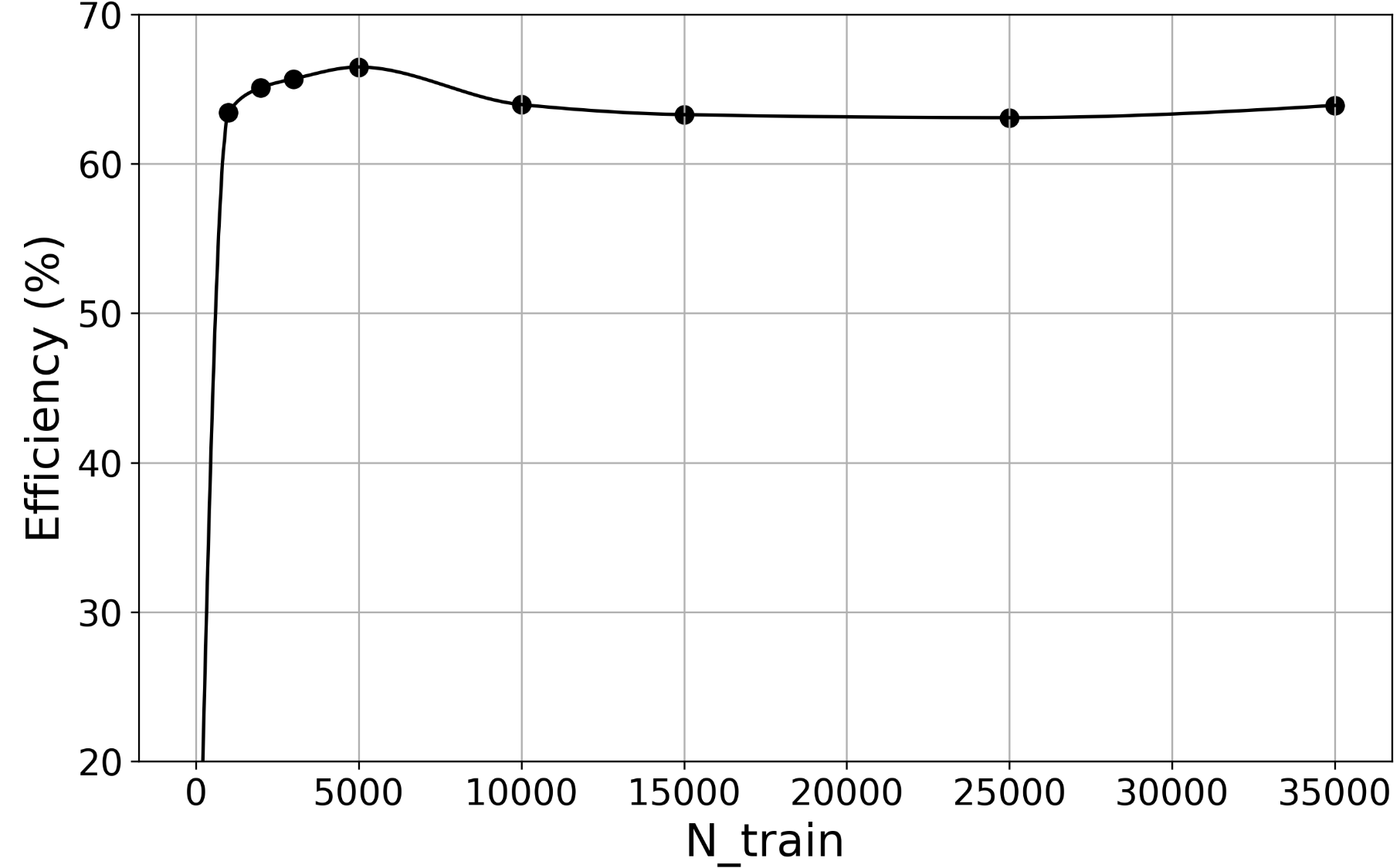}
        \caption{$\sigma = 0.4, \gamma = 1$}
        \label{fig:Sigma_0p4_gamma_1_embedding_cifar10}
    \end{subfigure}
    \caption{ 
    {\color{black} Noisy CIFAR-10 and pre-classification encoding. Efficiency versus $N_{\text{train}}$.}
    }
    \label{fig:Practical_Networks_Efficiency_encoding_cifar10}
\end{figure}

\section{Extended empirical verification}
\label{app:experiments_verification_extended}

In this section, we extend our empirical verification. In Figure \ref{fig:VerFor_S=1}, we simulate the theoretical setup, as in \ref{sec:empirical_verification} (we use $d = 2000, k=1000$ and $\sigma = 1$), but with $\mathcal{S}=1$. We see that the empirical efficiency coincides with the theoretical efficiency. 

\begin{figure}[t]
    \centering
    \includegraphics[width=0.7\linewidth]{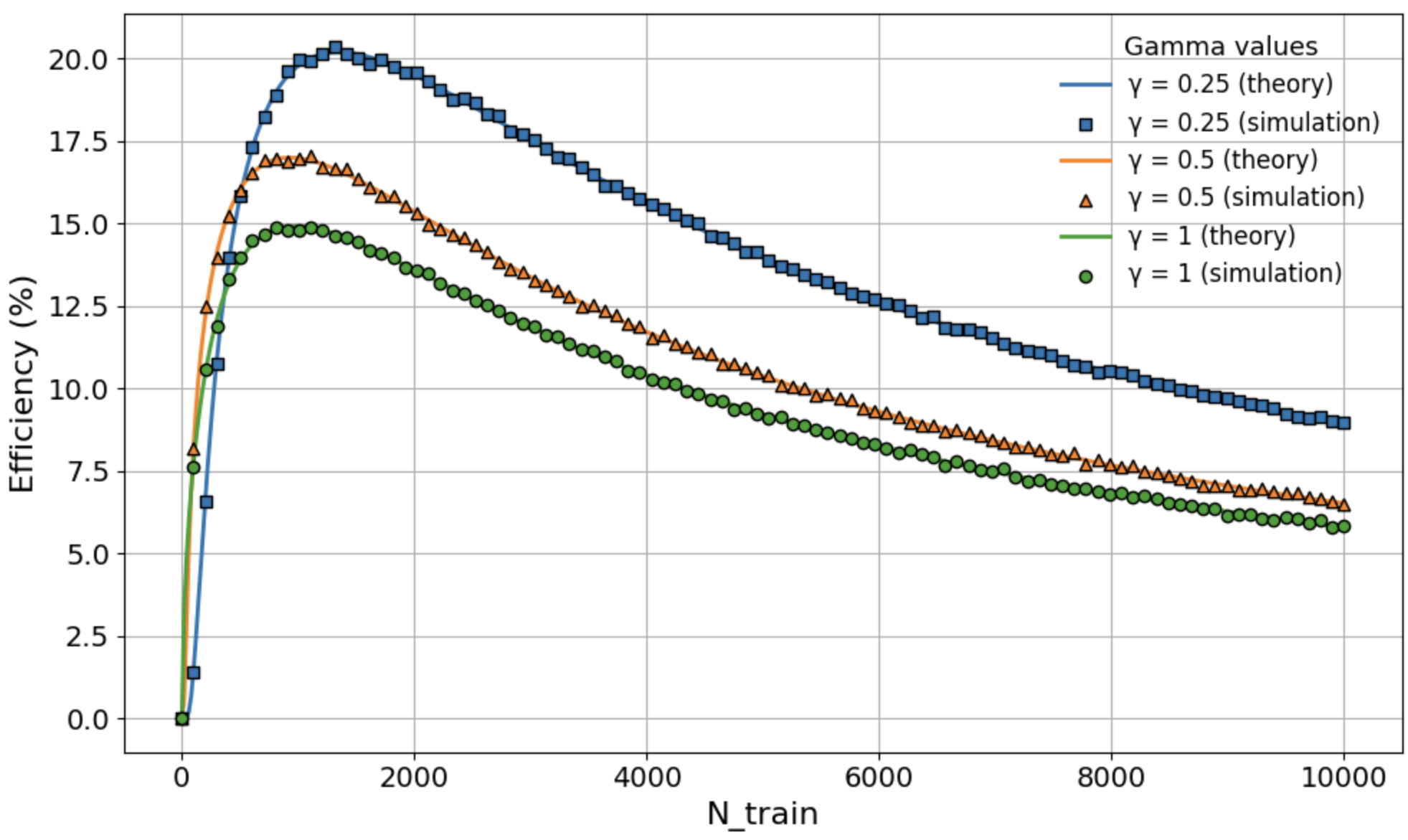}
    \caption{The theoretical setup. Efficiency of the data processing 
    procedure versus the number of training samples $N_{\text{train}}$, 
    for various values of the training imbalance factor $\gamma$, 
    and SNR of $\mathcal{S} = 1$.}
    \label{fig:VerFor_S=1}
\end{figure}

We now examine the effect of $\mathcal{S}$ on efficiency. We fix $\gamma = 1, d = 2000, k =1000$ and vary $\mathcal{S} \in \{0.5^2, 1, 1.5^2\}$. The results are presented in Figure \ref{fig:VerSweepS}. Let us discuss the results. We see that for $N_{\text{train}} \gg 1$ (in the right Figure), for larger SNR (lower noise level), the efficiency decreases. However, as \ref{thm:thm8} suggests, when the number of samples is limited, and for larger SNR (lower noise level), the efficiency increases. We notice this phenomenon in the left Figure, presenting low $N_{\text{train}}$, compared to the right Figure. First, the efficiency increases with the SNRs, and then as $N_{\text{train}}$ gets larger, the dependency flips. We also see a different behavior that sheds more light on this conclusion: the difference between different SNRs is larger in the low-samples region than the inverse relation in the high-samples region. This concludes the non-monotonic and non-intuitive dependency of the efficiency on the SNR.

\begin{figure}[t]
    \centering
    \begin{subfigure}[b]{0.4\linewidth}
        \centering
        \includegraphics[width=\linewidth]{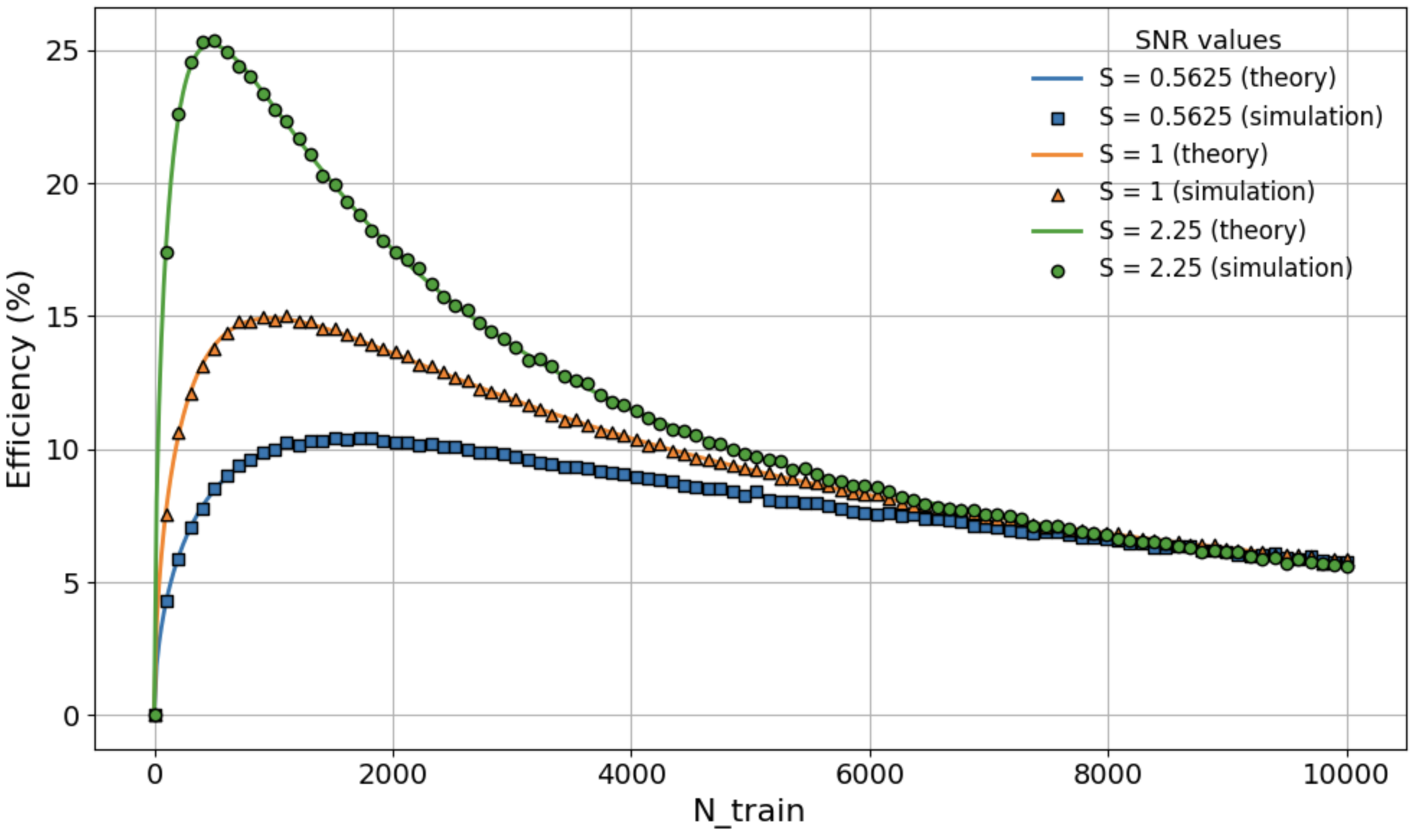}
        \caption{$N_{\text{train}} \le 10,000$}
        \label{fig:SweepS_lowN}
    \end{subfigure}
    \hspace{3mm}
    \begin{subfigure}[b]{0.4\linewidth}
        \centering
        \includegraphics[width=\linewidth]{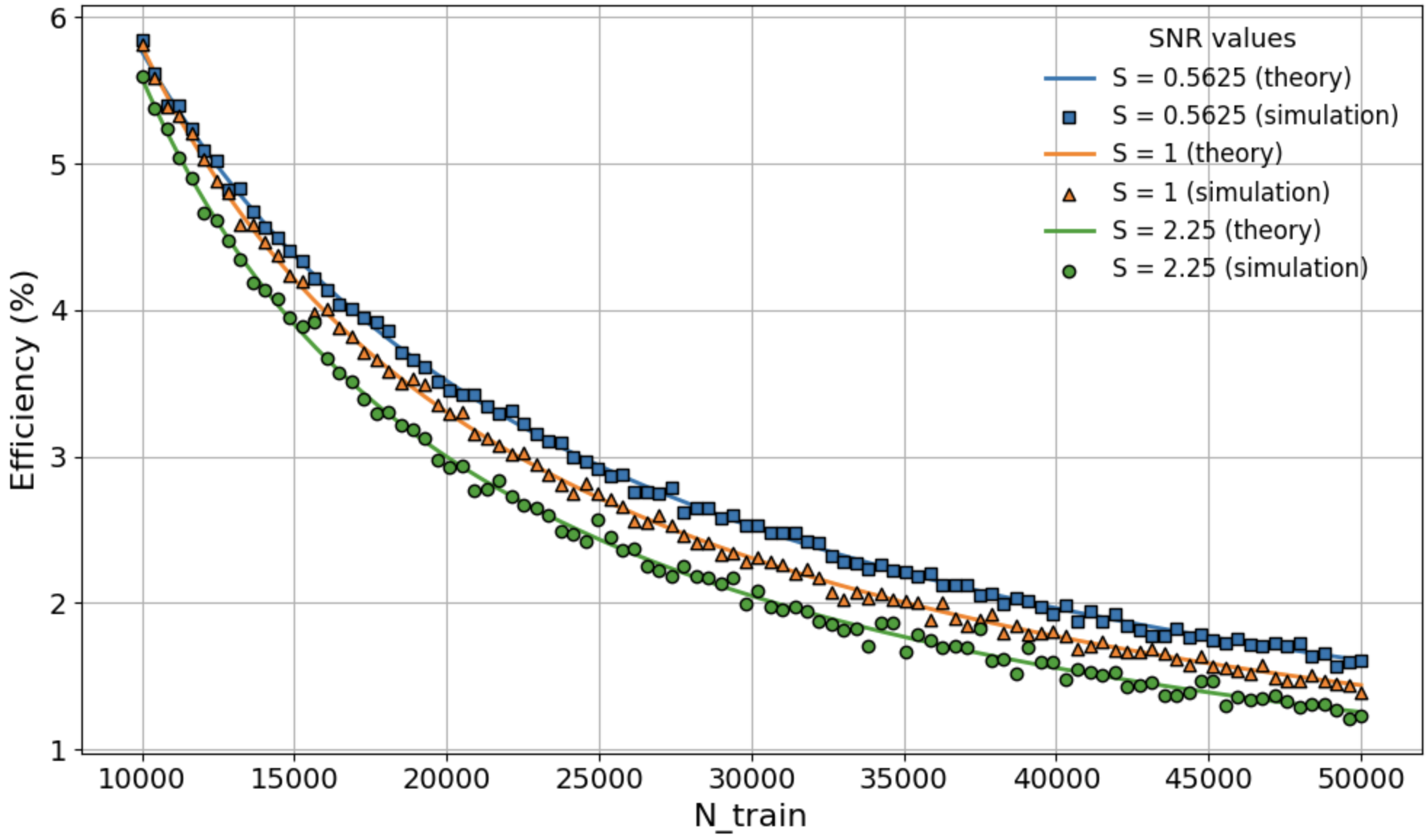}
        \caption{$10,000\le N_{\text{train}} \le 50,000$}
        \label{fig:SweepS_highN}
    \end{subfigure}
\vspace{-2mm}
    \caption{Extended simulation of the theoretical setup. Efficiency of the data processing versus the number of training samples $N_{\text{train}}$, for $\gamma = 1$, and various values of the SNR, $\mathcal{S}$, for (a) low samples regime, and (b) high samples regime.}
    \label{fig:VerSweepS}
\end{figure}

In addition, we examine the effect of $d - k$ on efficiency. We fix $\gamma = 1, d = 2000, \mathcal{S} = 1$ and vary $k \in \{500, 1000, 1500\}$. The results are presented in Figure \ref{fig:VerSweepK}. We see that larger $d - k$ corresponds to greater accuracy. Theorem \ref{thm:thm7} proves this for $N_{\text{train}} \gg 1$, but we see that this is true even for small $N_{\text{train}}$. Indeed, intuitively, reducing more dimensions is advantageous in terms of efficiency. This suggests that there is a direct monotonic relationship between the efficiency and $d - k$.

\begin{figure}
    \centering
    \includegraphics[width=0.5\linewidth]{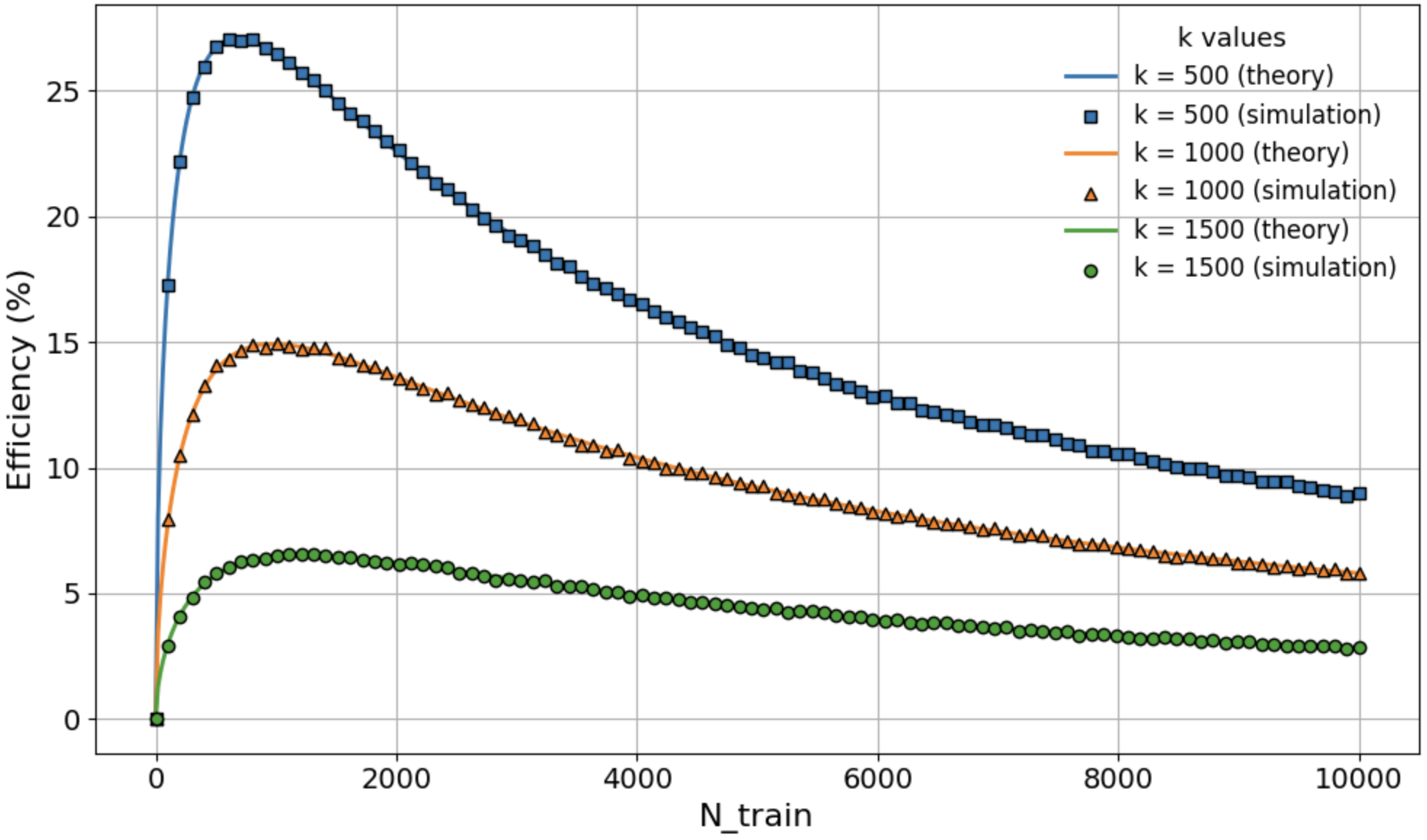}
    \vspace{-3mm}
    \caption{The theoretical setup. Efficiency of the data processing procedure versus the number of training samples $N_{\text{train}}$, for various values of the training imbalance factor, $\gamma$, and SNR of $\mathcal{S} = 1$.}
    \label{fig:VerSweepK}
\end{figure}

We now consider the same setting as in the empirical verification 
($d = 2000$, $k = 1000$, $\sigma = 1$, $\gamma \in \{0.25, 0.5, 1\}$, 
$\mathcal{S} \in \{0.75^2, 1, 1.5^2\}$), but per $N_{\text{train}}$, the data processing matrix $\boldsymbol{A}$ is learned from $50{,}000$ unlabeled samples using the algorithm described in the proof of Theorem~\ref{thm:thm3}. 
The corresponding results are shown in Figure~\ref{fig:LearnedA}, 
demonstrating the same trends as the theoretical efficiency. 
Moreover, as the number of unlabeled samples tends to infinity, the 
two curves coincide. To illustrate this, we also present results for the case that per $N_{\text{train}}$, the data processing matrix $\boldsymbol{A}$ is learned from $5{,}000{,}000$ unlabeled samples in 
Figure~\ref{fig:LearnedA_infiniteUnlabeledSamples}. Notice that as the amount of unlabeled samples available grows, the gap 
between the theoretical efficiency and the empirical efficiency is reduced.

\begin{figure}[t]
    \centering
    \begin{subfigure}[b]{0.49\linewidth}
        \centering
        \includegraphics[width=\linewidth]{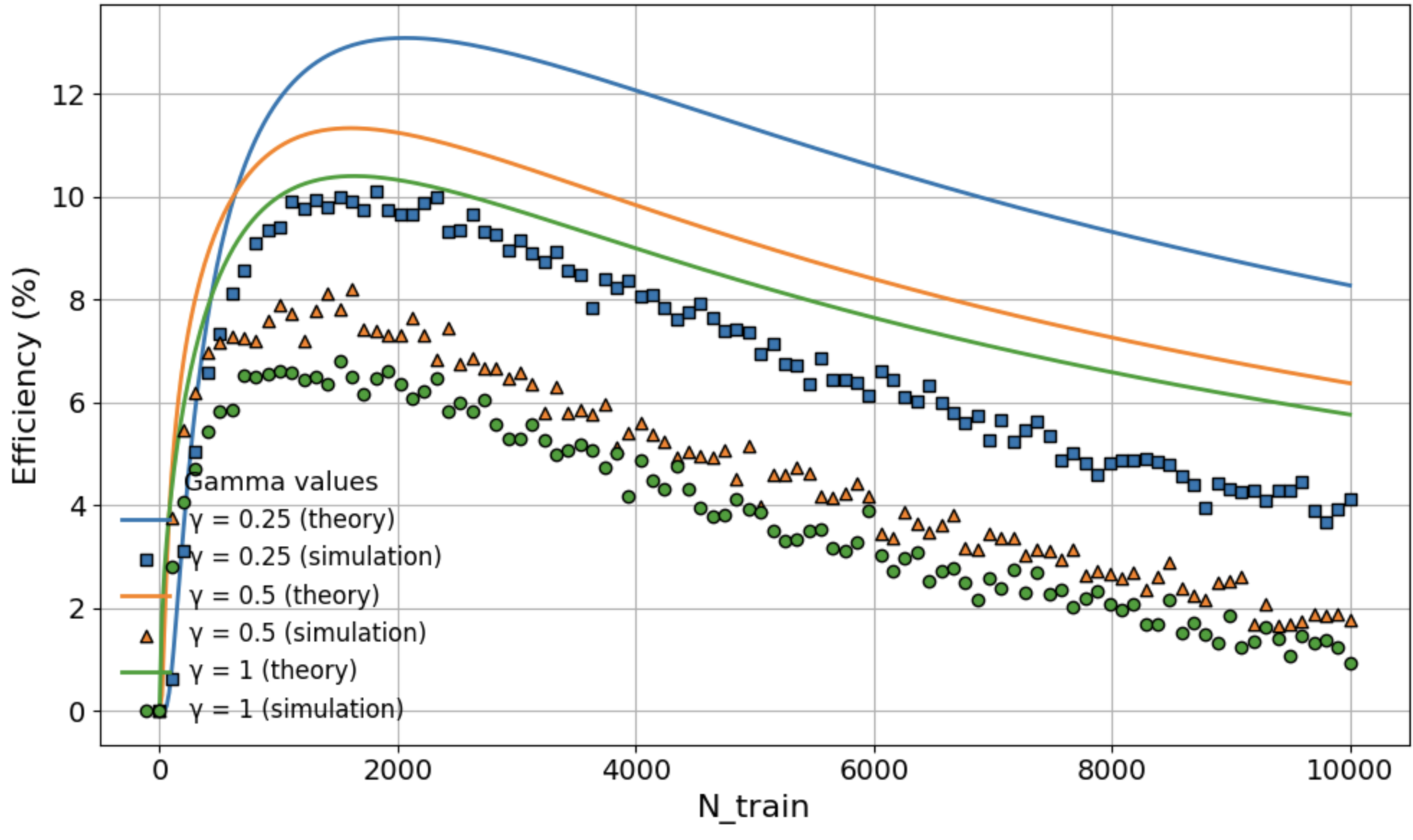}
        \caption{$\mathcal{S}=0.75^2$}
        \label{fig:S=0p5625_50000_unlabeled}
    \end{subfigure}
    \hfill
    \begin{subfigure}[b]{0.49\linewidth}
        \centering
        \includegraphics[width=\linewidth]{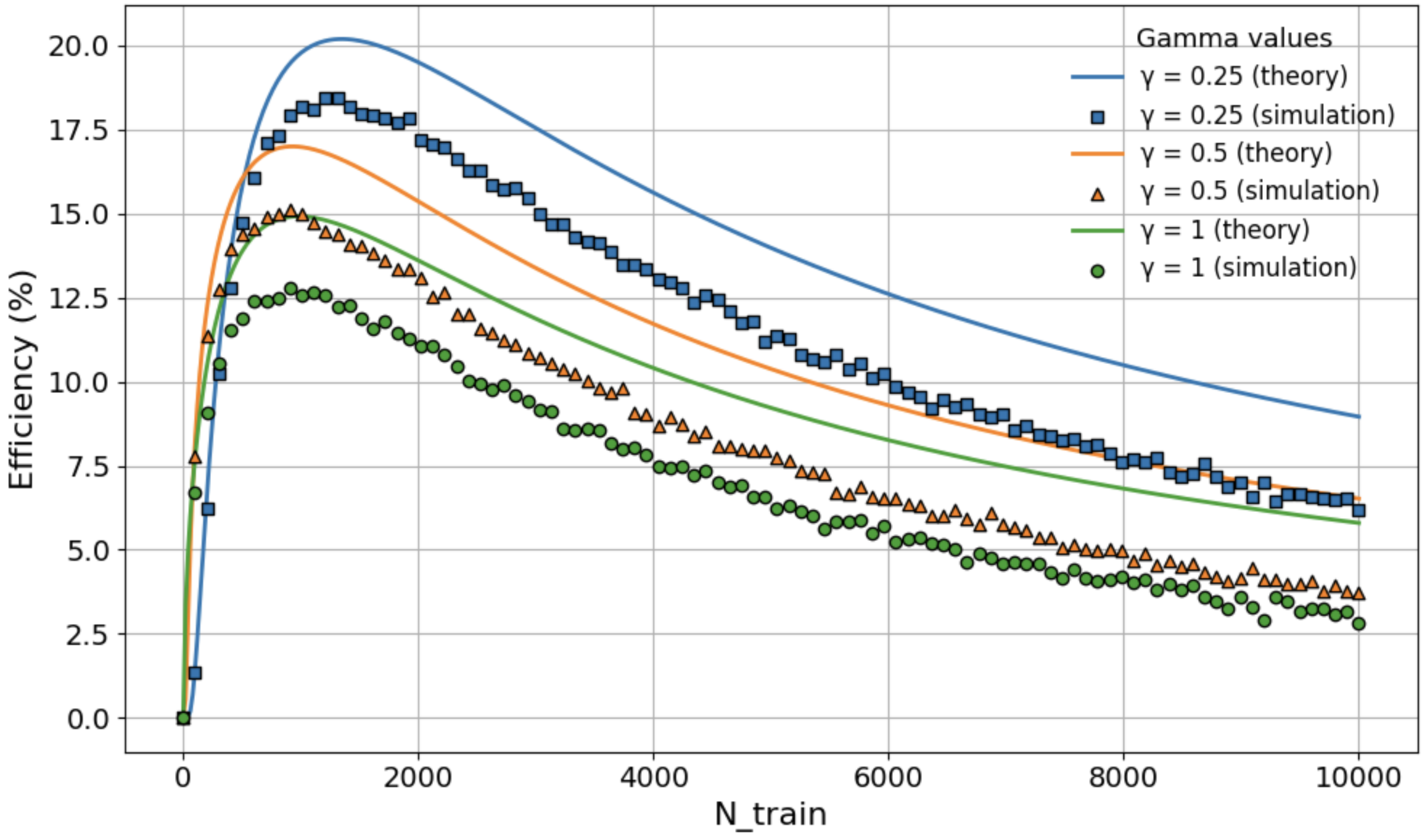}
        \caption{$\mathcal{S}=1$}
        \label{fig:S=1_50000_unlabeled}
    \end{subfigure}

    \vspace{0.5cm} % 

    \begin{subfigure}[b]{0.49\linewidth}
        \centering
        \includegraphics[width=\linewidth]{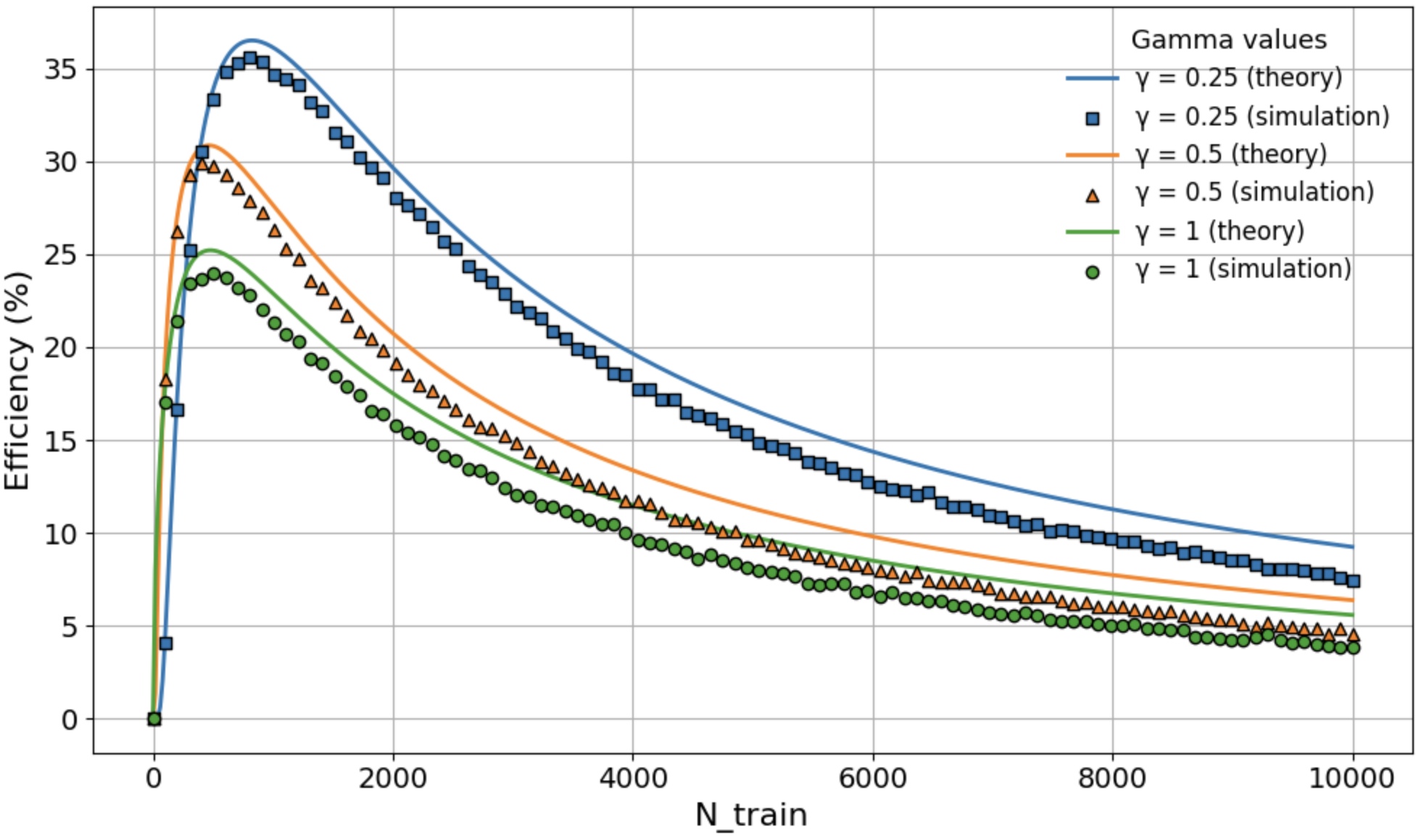}
        \caption{$\mathcal{S}=1.5^2$}
        \label{fig:S=2p25_50000_unlabeled}
    \end{subfigure}

\caption{Extended empirical verification - per $N_{\text{train}}, \boldsymbol{A}$ is learned from 50,000 unlabeled examples. Presented for (a) $\mathcal{S}=0.75^2$, (b) $\mathcal{S} =1$ and (c) $\mathcal{S}=1.5^2$.}

\label{fig:LearnedA}
\vspace{-5mm}

\end{figure}

\begin{figure}[t]
    \centering
    \begin{subfigure}[b]{0.49\linewidth}
        \centering
        \includegraphics[width=\linewidth]{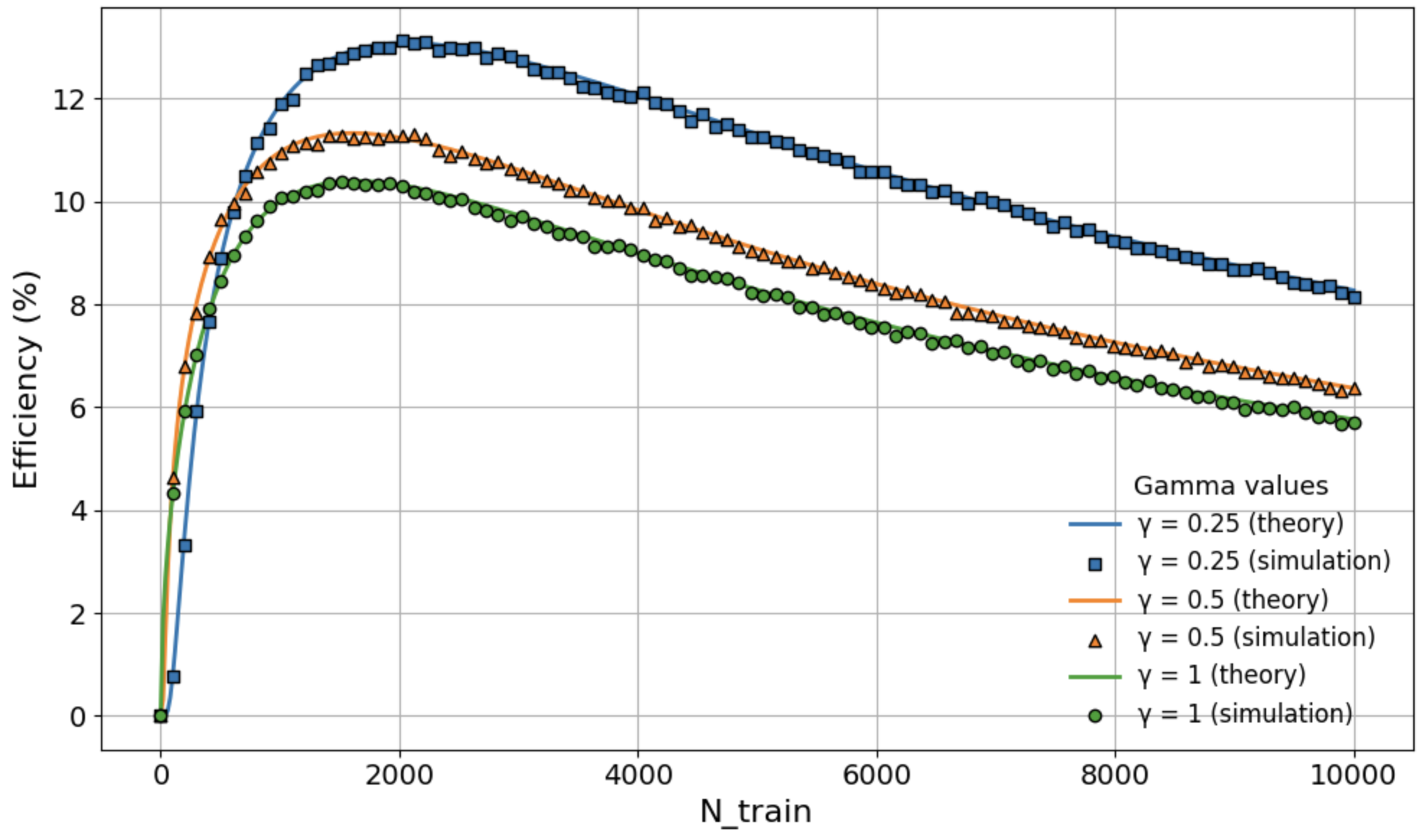}
        \caption{$\mathcal{S}=0.75^2$}
        \label{fig:S=0p5625_5000000_unlabeled}
    \end{subfigure}
    \hfill
    \begin{subfigure}[b]{0.49\linewidth}
        \centering
        \includegraphics[width=\linewidth]{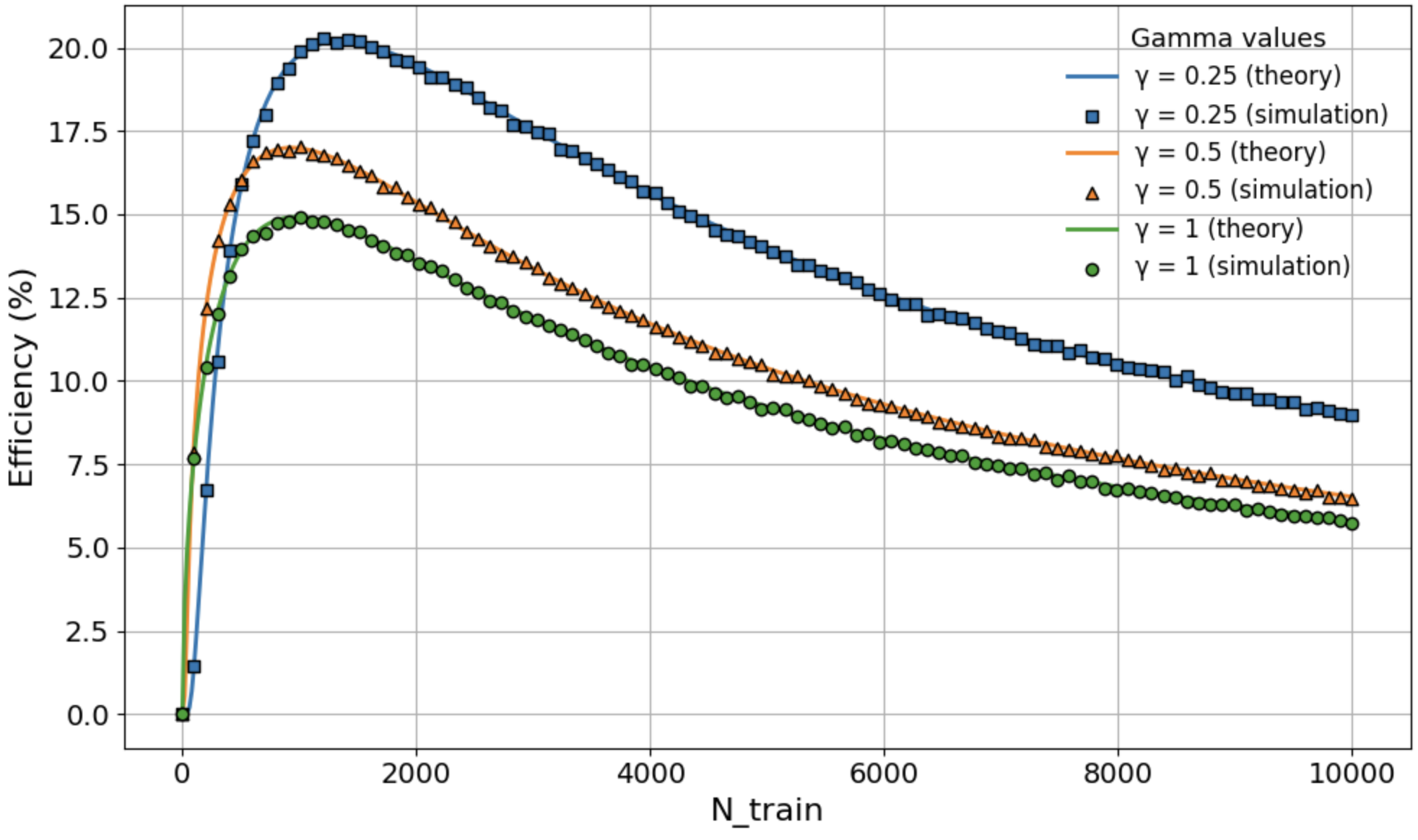}
        \caption{$\mathcal{S}=1$}
        \label{fig:S=1_5000000_unlabeled}
    \end{subfigure}

    \vspace{0.5cm} % 

    \begin{subfigure}[b]{0.49\linewidth}
        \centering
        \includegraphics[width=\linewidth]{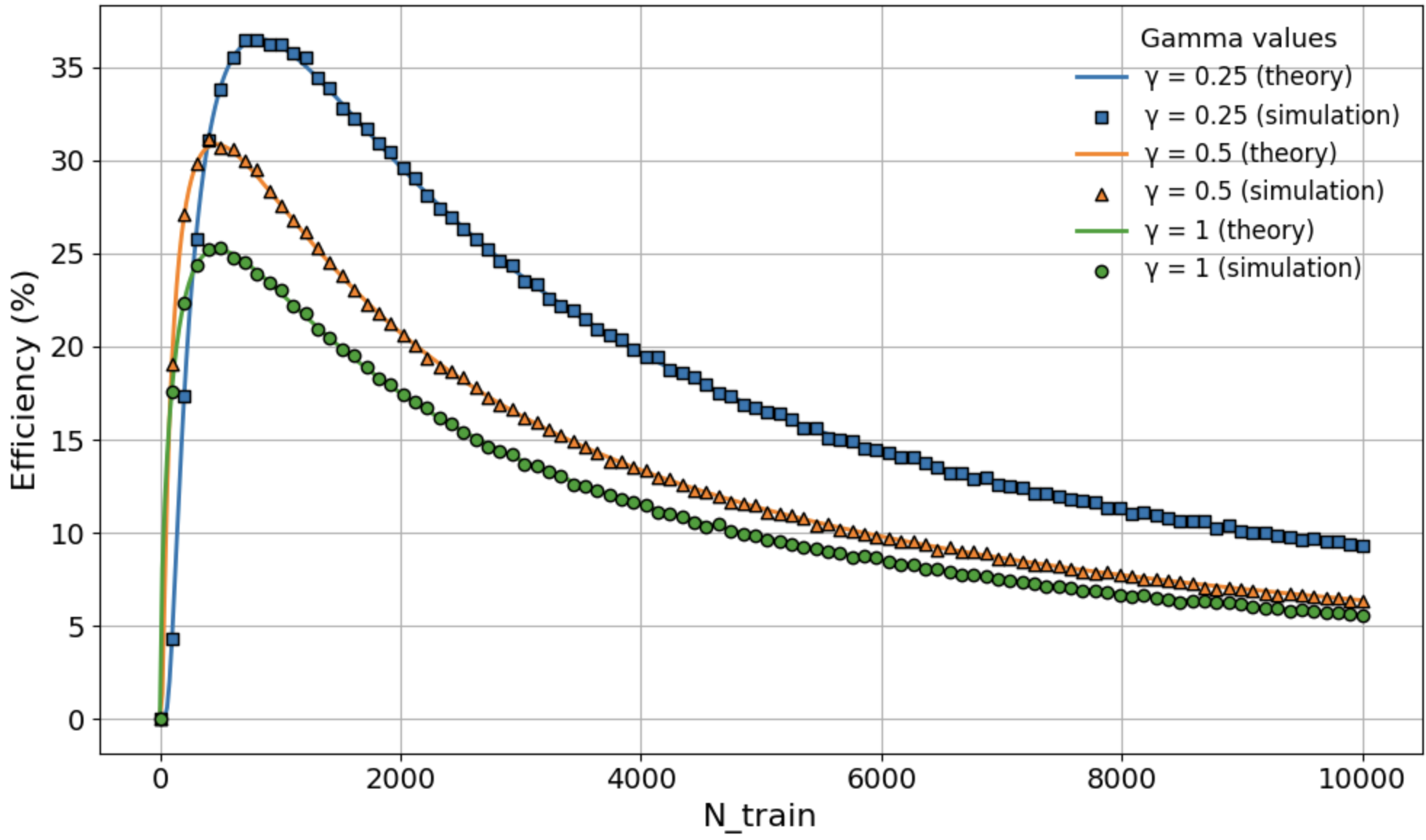}
        \caption{$\mathcal{S}=1.5^2$}
        \label{fig:S=2p25_5000000_unlabeled}
    \end{subfigure}

\caption{Extended empirical verification - per $N_{\text{train}}, \boldsymbol{A}$ is learned from 5,000,000 unlabeled examples. Presented for (a) $\mathcal{S}=0.75^2$, (b) $\mathcal{S} =1$ and (c) $\mathcal{S}=1.5^2$.}

\label{fig:LearnedA_infiniteUnlabeledSamples}

\end{figure}

{\color{black} Finally, we visualize the action of $\boldsymbol{A}: \mathbb{R}^2 \to \mathbb{R}$ on the GMM data in Figure~\ref{fig:A_visualization}. While our analysis considers the regime $d > k \gg 1$, as is common in practice, we use small values of $d=2$ and $k=1$ to enable visualization.

Recall that the plug-in classifier, before and after the data processing, depends only on the distance of a test sample from each of the empirical means.
Without processing, these empirical means are given by
$\widehat{\boldsymbol{\mu}}_j = \tfrac{1}{N_j} \sum \nolimits_{i=1}^{N_j} \vct{x}_{i,j}$ and hence distributed as $\widehat{\boldsymbol{\mu}}_j \sim \mathcal{N}\left(\boldsymbol{\mu}_j, \tfrac{\sigma^2}{N_j} \boldsymbol{I}_d\right)$.
Similarly, after processing by $\boldsymbol{A}:\mathbb{R}^d \to \mathbb{R}^k$, the empirical means obey $\boldsymbol{A}\widehat{\boldsymbol{\mu}}_j \sim  \mathcal{N}\left(\boldsymbol{A}\boldsymbol{\mu}_j, \tfrac{\sigma^2}{N_j} \boldsymbol{I}_k\right)$, where the semi-orthonormality $\boldsymbol{A} \boldsymbol{A}^\top = \boldsymbol{I}_k$ is used.
Consequently,
\begin{align*}
    \mathbb{E}\left[\norm{\widehat{\boldsymbol{\mu}}_j - \boldsymbol{\mu}_j}^2\right] &= \sum_{i=1}^d \mathbb{E}\left[\left([\widehat{\boldsymbol{\mu}}_j]_i - [\boldsymbol{\mu}_j]_i\right)^2\right]
    = \sum_{i=1}^d \frac{\sigma^2}{N_j}
    = \frac{\sigma^2}{N_j}d
\end{align*}
where we used
    $\left[\widehat{\boldsymbol{\mu}}_j\right]_i - \left[\boldsymbol{\mu}_j\right]_i \sim \mathcal{N}\left(0, \tfrac{\sigma^2}{N_j}\right)$.
Similarly, $\mathbb{E}\left[\norm{\boldsymbol{A}\widehat{\boldsymbol{\mu}}_j - \boldsymbol{A}\boldsymbol{\mu}_j}^2\right]  
    = \frac{\sigma^2}{N_j}k$.
    
The data processing lowers the dimension from $d$ to $k$, and thus improves the average squared error of the mean estimator by 
\begin{equation*}
    \frac{\sigma^2}{N}(d-k) > 0.
\end{equation*}
Since the classifier, before and after the data processing, depends only on the distance of the test sample from each of the empirical means, its accuracy increases when the accuracy of the empirical means improves while the distance between the means of the difference classes does not significantly reduce (i.e., $\|\boldsymbol{A}\widehat{\boldsymbol{\mu}}_2 - \boldsymbol{A}\widehat{\boldsymbol{\mu}}_1 \| \approx \| \widehat{\boldsymbol{\mu}}_2 - \widehat{\boldsymbol{\mu}}_1 \|$ ). 
The latter is accounted for by the property $\|\boldsymbol{A} \boldsymbol{\mu}\|=\| \boldsymbol{\mu} \|$ of the operator. 

Note that this behavior is observed in Figure~\ref{fig:A_visualization}: 
\begin{itemize}
    \item In the red class, the distance between the empirical mean and the real mean is 0.4632 before applying $\boldsymbol{A}$ and 0.39 after applying $\boldsymbol{A}$.

    \item In the blue class, the distance between the empirical mean and the real mean is 0.439 before applying $\boldsymbol{A}$ and 0.38 after applying $\boldsymbol{A}$.

    \item The distance between the empirical means of the different classes is 4.0374 before applying $\boldsymbol{A}$ and 4.01 after applying $\boldsymbol{A}$. Indeed, both are close to the distance between the real means, which is 4.
\end{itemize}

That is, $\boldsymbol{A}$ preserves the separation quality of the classes, while improving the estimation quality of $\boldsymbol{\mu}$.}

\begin{figure}[t]
    \centering
    \begin{subfigure}[b]{1\linewidth}
        \centering
        \includegraphics[width=0.5\linewidth]{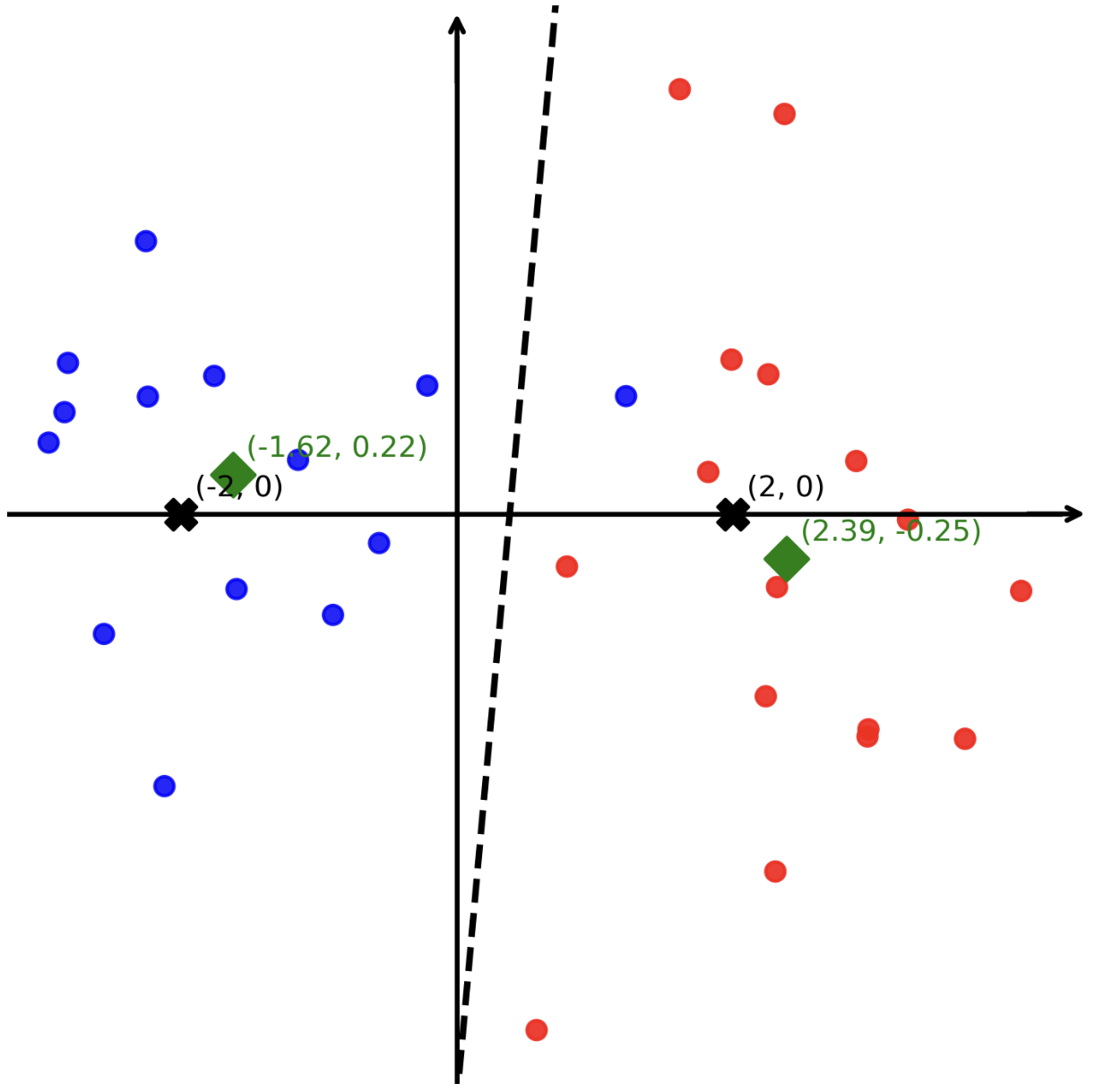}
        \caption{The original data, sampled from GMM in $\mathbb{R}^2$ with $\boldsymbol{\mu}_1 = -\boldsymbol{\mu}_2 = (2,0)^\top$, $\sigma=1$, and $N_{\text{train}} = 30$. The true means $\{ \boldsymbol{\mu}_1, \boldsymbol{\mu}_2 \}$ are marked by black `X's and the empirical means $\{ \hat{\boldsymbol{\mu}}_1, \hat{\boldsymbol{\mu}}_2 \}$ are marked by green diamonds. The learned decision boundary is marked by the dashed line (determined by the distance to the empirical means).}
        \label{fig:R2_synthetic}
    \end{subfigure}
\\
    \vspace{0.5cm}
    \begin{subfigure}[b]{1\linewidth}
        \centering
        \includegraphics[width=0.5\linewidth]{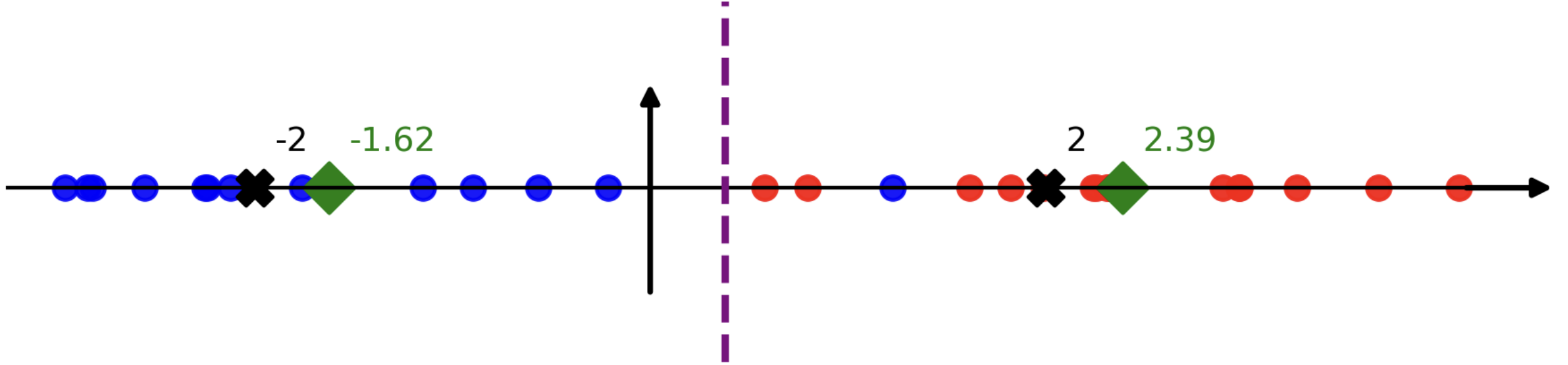}
        \caption{The data after applying $\boldsymbol{A} \in \mathbb{R}^{1\times 2}$ (in this case, projection onto the x-axis). As before, the true means $\{ \boldsymbol{A}\boldsymbol{\mu}_1, \boldsymbol{A}\boldsymbol{\mu}_2 \}$ are marked by black `X's and the empirical means $\{ \boldsymbol{A}\hat{\boldsymbol{\mu}}_1, \boldsymbol{A}\hat{\boldsymbol{\mu}}_2 \}$ are marked by green diamonds.}
        \label{fig:R1_synthetic_postA}
    \end{subfigure}

\caption{Visualization of the effect of $\boldsymbol{A}:\mathbb{R}^2 \to \mathbb{R}$. Note that the empirical means (green diamonds) are closer to the true means (black `X's) after the operation $\boldsymbol{A}$. The distance between the empirical means of the different classes remains similar.} % 

\label{fig:A_visualization}

\end{figure}

\end{document}